\documentclass[chap]{thesis}

\begin{document}

%
    
%
\thesistitle{\bf Computational Analysis of Deformable Manifolds: from Geometric Modeling \\ to Deep Learning}        
\author{Stefan C. Schonsheck}        
\degree{Doctor of Philosophy}        
\department{Mathematical Sciences} 
     
\signaturelines{4}     
\thadviser{Rongjie Lai}
\memberone{Gregor Kovacic}        
\membertwo{John E. Mitchell}        
\memberthree{Birsen Yazici}

\submitdate{[August 2020]\\ Submitted June 2020}        
\copyrightyear{2020}   

%
\titlepage     
\copyrightpage         
\tableofcontents        
\listoftables          
\listoffigures         

 
\specialhead{ACKNOWLEDGMENT}
 
There are more people who deserve credit and thanks for supporting and inspiring me than I can list. I could not have possibly done this with you. Thank you all.

I would be remiss not to mention several individuals who have been especially important in my intellectual development. Special thanks to: My parents, Drs. Janet Coy and Jonathan Schonsheck, who've always taught me the importance of education and supported all of my endeavors. To my brother, who's sibling rivalry and companionship has always inspired me. And to our family friend and mentor Varda Holland-Witter, who taught me thinking diffidently isn't always thinking wrongly. 

My high-school physics teacher Ms. Liutzi proved to me that being nerdy was cool and that math had uses beyond math classes. At Skidmore College Dr. David Vella encouraged a play-full approach to problem-solving in mathematics that I've tried to emulate ever since. There also, Dr. Rachel Roe-Dale introduced me to applied mathematics with her wonderful lectures in both differential equations and numerical algorithms. I am forever grateful. 

Every single professor in the mathematics department that I've interacted with at RPI were wonderful. This goes doubly for all of the members of my dissertation committee: Dr. Gregor Kovacic, Dr. John Mitchel, and Dr. Birsen Yazici who have been great teachers, mentors, and supervisors at different times during my time in Troy.  My frequent chats with the department securities, Ms. Dawnmarie Robens and Ms. Erin Lynch, have always been great pick-me-ups on though days. I have been extremely privileged to know you all. 

The Institute for Pure and Applied Math (IPAM) has been extremely generous in supporting my attendance at their programs and gave me a chance to interact with some truly amazing researchers. In particular, conversations with Dr. Michael Bronstein, Dr. Hongkai Zhao, and Dr. Pablo Su\`{a}rez-Serrato have guided and inspired me. Early sections of this work could not have been completed without the NSF, specifically in part by NSF DMS--1522645 and later sections with support from the IMB AI-Horizons program and in particular Dr. Jie Chen who's advice on math and writing I deeply value. 

Finally, I'd like to thank my advisor, Dr. Rongjie Lai whose been one of the best mentors I've ever had. His patience and encouragement have been invaluable and I could never have hoped to complete this thesis without him. I hope this work makes him proud.

 
\specialhead{ABSTRACT}
 
Leo Tolstoy opened his monumental novel Anna Karenina with the now famous words: 
\begin{quote}
    Happy families are all alike; every unhappy family is unhappy in its own way 
\end{quote}
    
A similar notion also applies to mathematical spaces: Every flat space is alike; every unflat space is unflat in its own way. However, rather than being a source of unhappiness, we will show that the diversity of non-flat spaces provides a rich area of study. 

The genesis of the so-called 'big data era' and the proliferation of social and scientific databases of increasing size has led to a need for algorithms that can efficiently process, analyze and, even generate high dimensional data. However, the \emph{curse of dimensionality} leads to the fact that many classical approaches do not scale well with respect to the size of these problems. One technique to avoid some of these ill-effects is to exploit the geometric structure of coherent data. In this thesis, we will explore geometric methods for shape processing and data analysis.

More specifically, we will study techniques for representing manifolds and signals supported on them through a variety of mathematical tools including, but not limited to, computational differential geometry, variational PDE modeling and deep learning. First, we will explore non-isometric shape matching through variational modeling. Next, we will use ideas from parallel transport on manifolds to generalize convolution and convolutional neural networks to deformable manifolds. Finally, we conclude by proposing a novel auto-regressive model for capturing the intrinsic geometry and topology of data. Throughout this work, we will use the idea of computing correspondences as a though-line to both motivate our work and analyze our results

One of the advantages of working in this manner is that questions which arise from very specific problems will have far reaching consequences. There are many deep connections between concise models, harmonic analysis, geometry and learning that have only started to emerge in the past few years, and the consequences will continue to shape these fields for many years to come. Our goal in this work is to explore these connections and develop some useful tools for shape analysis, signal processing and representation learning.    

\chapter{INTRODUCTION}

\section{Motivation}

Both recent and long term advances in data acquisition and storage technology have led to the genesis of the so-called 'big data era'. The proliferation of social and scientific databases of increasing size has lead to a need for algorithms that can efficiently process, analyze and, even generate this data. However, due to the large number of observations (volume) of data, and the number of variables observed (dimension), many classical approaches from traditional signal processing and statistics are not computationally feasible in this regime. The field of Geometric Data Processing has been developed as a way to exploit inherent coherence in data to design new algorithms based on motivation from both differential and discrete geometry. These techniques can be put into two broad classes: those which seek to generalize existing Euclidean methods for application on manifolds, and those which seek to incorporate geometric structure of data into standard problems. Generally speaking, methods in the first class model problems on a single (possibly deforming) manifold, while in the second class each data point is modeled as a single point drawn from a high-dimensional probability distribution which is supported only on some low dimensional structure. In this thesis we will explore several methods in each of these genres.

\section{The Curse of Dimensionality and Manifold Hypothesis}\label{sec:Curse}

The \emph{Curse of Dimensionalality}~\cite{bellman1956dynamic} is an umbrella term for a set of related phenomena in high-dimensional mathematics and data science in which there is some undesirable scaling with respect to the dimension of the data. This scaling may be in the time or memory complexity of certain algorithms or in the number of observations needed to approximate a given quantity. Generally, this is caused by the fact that the volume of an $n$-dimensional manifold increases exponentially with the dimension $n$. For a concrete example if we would like to uniformly sample a unit cube in one dimension, with a resolution of $.01$ we need $100$ points, to do so on a 5-dimensional cube we would need $1,000,000,000$ points (in general we need $10^{2d}$ points on a $d$-dimensional cube to have a resolution of .01). 

Similarly the Hughes phenomenon \cite{hughes1968mean} (or peaking paradox) observes that in pattern recognition increasing the detail at which a measurement is made often leads to poorer results. For example, increasing the resolution of the camera in a photo-recognition system may lead to the system making more errors. These types of errors can often be related to the Vapnik-Chervonenkis dimension (VC-dimension) of the problem \cite{vapnik2015uniform}; by increasing the number of pixels in the representation the space of function that the model must deal with will become more complex, even though the conceptual idea which the system has to predict (or learn) has not changed.

One way to overcome this challenge is to use the coherence of the data to reduce the complexity of the problem. A commonly held belief in data science, known as the \emph{manifold hypothesis}~\cite{demers1993non, roweis2000nonlinear,tenenbaum2000global,Belkin2003, zomorodian2005computing, fefferman2016testing}, states that real-life data often lies on, or at least near, some low-dimensional manifold embedded in a high-dimensional ambient space. This motivates us to develop algorithms that exploit this structure. Doing so allow us to reduce the dependence of our methods on the embedded dimension of the data in favor of algorithms that depend on the intrinsic dimension instead. This is useful not only because of the reduction of the size of the problem, but also because it will allow us to develop methods that are invariant to certain transformations which are common in real world applications.
    
\section{Manifold Structured Data}\label{sec:ManfioldStructuredData}

The advent of modern imaging technologies such as 3D cameras, CT, and MRI scanners as well as 3D animation and computer graphics has lead to the creation of many 'shape'-based data sets. In this context each shape is often referred to as its own \textit{data set}, with each point on the shape being called a \textit{data point}. 
Frequently, these data-sets are modeled and stored as triangulated meshes, but point clouds, level sets, and voxelized representations are also common in practice. In any case, there are several fundamental tasks that are necessary for more complex analysis. These primary tasks are:
\begin{itemize}
    \item Recognition: Determining what type of object a given data set represents
    \item De-noising: Recovery of the underlying shape from errors made in the observation or generation of the data
    \item Segmentation: Separation of points within the data into meaning classes  
    \item In-paining: Creation of new data points in areas that are not observed in the initial data acquisition
\end{itemize}
In section \eqref{sec:MotivatingProblem} we will show that each of these tasks can be accomplished by solving an even more fundamental problem: correspondence. In essence, if we can find a geometrically meaningful map between an unknown shape and a given reference shape then we only need to solve the primary problem once on the reference shape, then use the mapping to solve it on a new shape. 
 
\section{Data Manifolds}

As mentioned in section \eqref{sec:Curse} the manifold hypothesis says that most real life data lies on or near some underlying manifold which has a much lower dimension that of the full observation space. To motivate this, we use the example of natural images. For a fixed pixel resolution, there are many configuration which 'look like' images which might be captured by a camera in the real world. However, most possible combinations of pixels do not 'look like' anything more than static or noise. This suggests that the set of natural images is a low dimensional subset of the entire pixel space. Given a natural image, there many nearby examples which also look like natural images. For example, two frames from a movie of a dog walking in a park are very close when measured in pixel difference. However, not all images within this distance will look like real images. Again, most directions will look like noise. This coherence between nearby pictures and limited direction of 'valid' movement motives the manifold structure of data in the manifold hypothesis. 

These observations motivate us to exploit structures and techniques from the study of manifolds to study entire classes of data in which the exact structure of the underlying manifold is unknown, but each observation can be thought of as being a point on the unknown manifold. When working in this context, we refer to each complete observation as a \textit{data-point} (i.e. am entire image or mesh is a data-point, not a pixel or point within it) and then the entire set of data-points form a \textit{data-manifold}.

\section{Geometric Methods for Data Processing}

The study of non-Euclidean geometry has a long and rich history that dates back at least as far as Gauss, but was not applied in earnest to data science until the end of the twentieth century. Important predecessors to modern geometric processing methods include: \emph{level-set methods}, \emph{harmonic analysis}, \emph{multi-resolution analysis} and \emph{graph methods}. Level-sets methods \cite{adalsteinsson1994fast, malladi1995shape} conceptualized moving front problems as level-sets of higher dimensional geometric objects. Harmonic \cite{molchanov1995harmonic} and  multi-resolution \cite{Daubechies:1992ten, deng2008harmonic} analysis for homogeneous spaces extended classical concepts in signal analysis to apply to more abstract domains by using the underlying geometry of spaces do create bases for efficient processing and representation. Finally, graph processing methods for data \cite{ding2001min} sought to exploit sparely connected data structures to reduce computational overhead. Each of these fields contributed tremendously to the creation of the field of geometric processing and continue to provide inspiration for modern research. 

One of the first works to truly integrate differential geometric ideas and data processing was the idea of differential geometry into data science was that of Laplacian Eigenmaps \cite{Belkin2003}. In this work, the authors proposed a method for non-linear dimensional reduction based on solving a discretized version of the Laplace-Beltrami (LB) operator on a general data set. Following this inspiration the LB operator became, and continuous to be an essential tool in the field \cite{reuter2006laplace,levy2006laplace,Rustamov:2007,sun2009concise,Bronstein:2010CVPR,lai2010metric,Raviv2011CVPR,aubry2011wave,lai2011automated,shi2013cortical,shi2014metric,shtern2016fast} An important component of all of these works is separating \emph{intrinsic} and \emph{extrinsic} information. We will define these terms more rigorously in the next chapter, but essentially the underlying idea is to solve some underlying differential equation on a manifold (or data set) rather than in the ambient space in which the data set is embedded. 

Data-driven techniques, and more specifically machine learning, became popular in the field of geometric shape processing much more recently. Important early advances were made by researchers working on graph analysis problems \cite{gori2005new} as well as others working in computer graphics \cite{masci2015geodesic}. Early efforts in both of these lines of research sought to generalize convolutions, which had become very importing to the deep learning community,  to non-euclidean domains. This problem remains fundamental and will be the central focus of a chapter later in this work.  See \cite{bronstein2017geometric} for a more thorough review of geometric deep learning.

\section{Outline of This Thesis}

The rest of this thesis is structured as follows: In chapter \ref{Chap:Background}, we briefly review some concepts from differential geometry, optimization, and machine learning which will be useful in our further discussions. Next, in chapter \ref{Chap:LBBP} we develop a variational model for computing correspondences between non-isomorphic shapes. In chapter \ref{Chap:PTC}, we develop a generalization of convolution to apply to non-Euclidean spaces. This allows us to develop CNN-like neural networks on 3D shapes, which we can use for a broad class of problems, including correspondence, recognition, and data generation. Chapter \ref{Chap:CAE} deals with exploiting the manifold structure of general data sets to create more powerful models for auto-encoding and synthetic data generation. Finally, in chapter \ref{Chap:Conclusion}, we add some concluding remarks and discuss future avenues of research.


\chapter{MATHEMATICAL BACKGROUND}\label{Chap:Background}
In this chapter we briefly review some topics and standard definitions from differential geometry, optimization, and machine learning which will be useful in later chapters. After some introductory discussion each section concludes with some recommendations of texts for interested readers. In the final section, \ref{sec:MotivatingProblem}, we describe the problem of shape correspondence, which will be the central focus of chapter \ref{Chap:LBBP}, but will reappear thought, and will serve as a motivating example for the rest of this work. 


\section{Continuous Representation of Manifolds}\label{sec:ContBackground}

Colloquially speaking, the manifolds we are interested are spaces that looks flat when viewed from a close enough perspective and as a result behaves approximately like Euclidean space when operating in a sufficiently small neighborhood. Formally, around any point $x$ on the manifold $\M$ there is some neighborhood $\N(x)$ which can be homeomorphically mapped to some Euclidean space. The dimension of this Euclidean space, $d$, is reffed to as the \textit{intrinsic dimension} of the manifold, or in some contexts, just the dimension of it in which case $\M$ is called a $d$-manifold. This map $\psi: \N(x) \rightarrow \RR^d$ called a chart and is usually denoted a tuple $(U,\psi)$ where $U$ is the domain of $\psi$, an open subset of $\M$. Any manifold can be parameterized by a collection (maybe infinite) of overlapping charts parameterization called an atlas. Given two charts, $\psi_i$ and $\psi_j$, whose domains overlap (their intersection is non-empty) we define the \textit{chart transition function} $\tau_{ij}$ as $\phi_j \circ \phi_i^{-1}: \phi_i(U_i \cap U_j) \rightarrow \phi_j(U_i \cap U_j)$. Figure \ref{fig:Manifold} illustrates these concepts. Intuitively, this map allows us to change the description of a point induced by $\phi_i$ to another description induced by $\phi_j$. In Chapter \ref{Chap:CAE} we take advantage of this local description of manifolds and transition conditions to more efficiently parameters auto-regressive and generative machine learning models. 

\begin{figure}[ht]
  \centering
   \includegraphics[width=.65\linewidth]{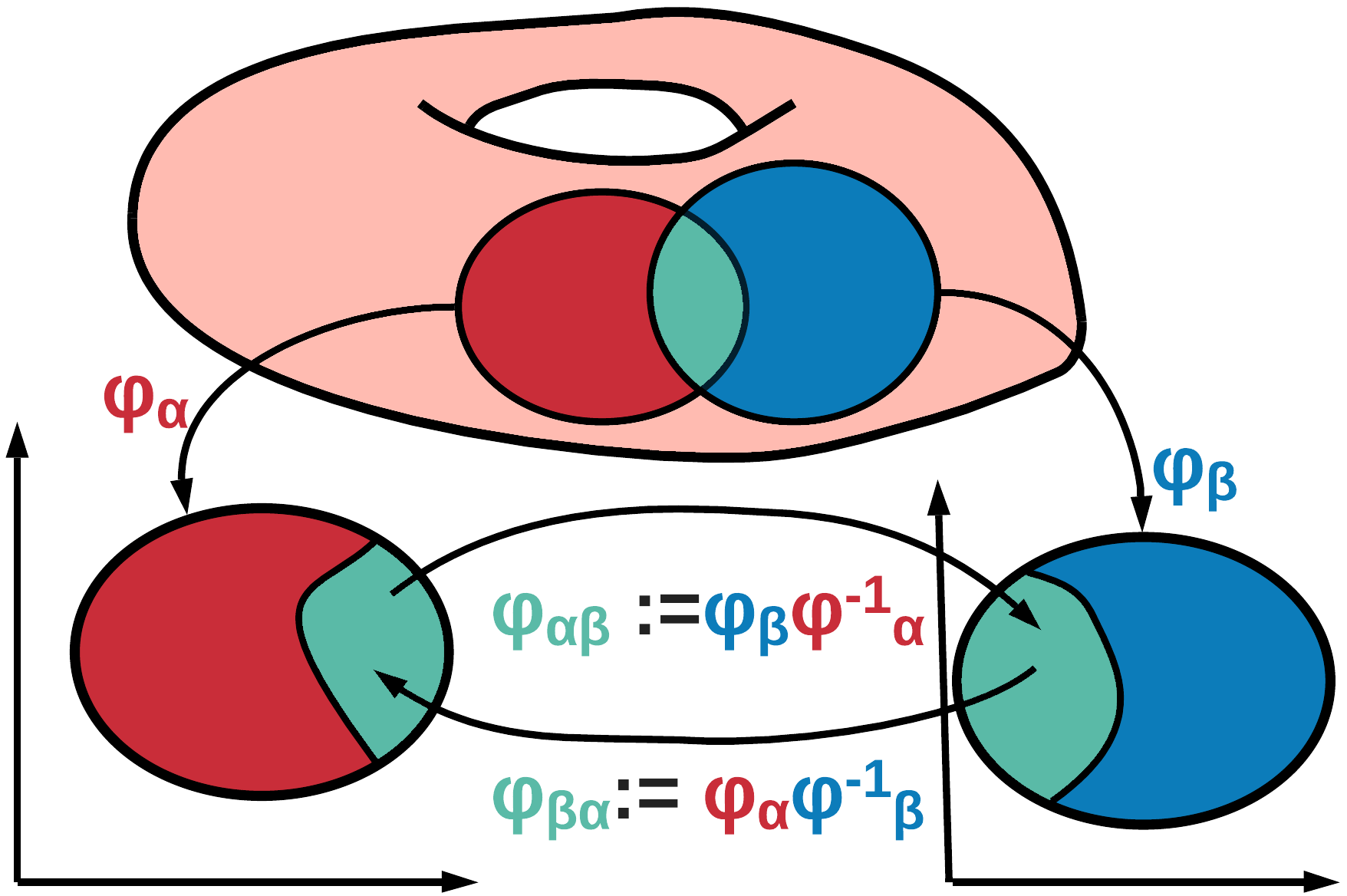}
   \lrpicaption{Illustration of a Manifold and Two Charts.}
   \label{fig:Manifold}
\end{figure}

Many important properties of the manifold can be studied by examining the properties of atlas. Importantly, a notion of smoothness for manifold can be defined as the smoothness of a set of compatible charts. For example, two charts are said to be $C^{\infty}$-compatible if the transition functions between them are $C^{\infty}$ in the usual sense. If this property holds for all charts in the atlas the manifold is said to be $C^{\infty}$ continuous. Similarly, a \textit{compact manifold} is one which can be covered with a countable number of charts, following from the usual definition of a compact topological space being one for which every open cover has a finite sub-cover. 

A \textit{metric space} is a set of points equip with a metric that measures distances between them. Unusually denoted as a tuple $(M,g)$, metric spaces satisfy the following conditions:
\begin{eqnarray} 
\left\{ 
\begin{aligned} \label{metirc}
&g(x,y) \geq 0 \quad \forall x,y \in M \\
&g(x,y) = 0 \quad \text{iff} \quad  x=y \\
&g(x,y) = g(y,x)  \\
&g(x,z) \leq g(x,y) + g(y,z) \quad \forall x,y,z \in M
\end{aligned}
\right.
\end{eqnarray}
By equipping a manifold with a metric we are able to further study properties of chart and manifolds.

Very frequently it is more convenient to work with with a manifolds embedded in some Euclidean space $\RR^D$ (with $D > d$ and often $D >> d$) rather than with the manifold alone. In this case the manifold is a subset of points taken from the space $\M \subset \RR^D$. An important result by Nash \cite{nash1954c1} shows that any smooth manifold can be be embedded in a space of dimension $2d$. With an embedding it is easy to informally describe the idea of tangent planes (although it is possible to formally define without embedding the manifold and we will do so shortly). For a $d$-dimensional manifold the tangent plane at $x$, denoted $T_x \M$, can be though of as a $d$-dimensional flat plane which kisses the manifold at $x$. That is, the tangent plane $T_x \M$ intersects $\M$ at $x$ and has the same 

Given a curve $\gamma$ on $\M$ ($\gamma:t \in [-1,1] \rightarrow  \M$) with $\gamma(0) = x$ we can define a tangent vector at $x$ to be the derivative of the compassion of $\gamma$ with a chart function. That is $d\psi_x(\gamma'(0)) = \frac{d}{dt}[(\psi \circ \gamma)(t)]\big|_{t=0}$ where $\frac{d}{dt}$ is the standard derivative.The collection of the tangent vectors for all possible curves passing through $x$ is called the tangent plane. More formally, let $\M$ be a $C^{\infty}$ manifold, $\psi$ be $C^{\infty}$ chart and $f$ be some real valued function mapping form $\M$ to $\RR^d$. A \textit{derivation} at $x$ is a linear map that satisfies the chain rule: $D(fg)(x) = (D(f) g + f D(g))(x)$ for any $f,g \in C^{\infty}:\M \rightarrow \RR$. Then by assigning linear addition and scalar multiplication operators:
\begin{equation}
(\lambda * D) (f)  = \lambda D(f)
\end{equation}
\begin{equation}
(D_1 + D_2)f = D_1(f) + D_2(f)
\end{equation}
we form a linear space. This space is defined to be the tangent plane at $x$: $T_x \M$. The collection of tangent spaces for all points on the manifold is called the \textit{tangent bundle} and denoted $T \M$.

A \textit{Riemann Manifold} is a smooth manifold additionally equipped with an inner-product $g_x$ on the tangent space $T_x\M$ at each point $x\in \M$. If the manifold is embedded in Euclidean space, then the standard inner product in $\RR^D$ can be used, and is called the \textit{induced metric}. More generally, any positive definite matrix can be used to define an inner product, and since the metric measures distance, by changing it we can change the shape of a manifold without explicitly recomputing its embedding. In chapter \ref{Chap:LBBP} we will take advantage of this.

Since Riemannian manifolds are locally Euclidean, we are able to transfer information from one point (and it's tangent space) to nearby points (and tangent spaces) though connections. Again it is easiest to motivate this concept with a manifold embedded in a real space, but we can define these concepts formally without an embedding. The tangent space $T_x \M$ can be parameters by a set of $d$ orthogonal vectors $\{d\psi_i(x)\}_{i=1}^d$ where each $d\psi_i(x)$ is the derivative of $\psi$ in the $i^{th}$ direction of the canonical Euclidean basis. Then any vector $v \in T_x \M$ can be written as a weighted sum of these basis vectors. At a nearby point $x'$ we can define a basis in the same way, but if the manifold is not flat then $d\psi_i(x) \neq d\psi_i(x')$ for some $i$. However, since these bases are coming from the same chart function it is trivial to find the rotation needed to associate the bases. By choosing the same coefficients in the weighted sum we can find a vector $v'$ which is \textit{parallel} to $v$. To do this over longer distance we will need to incorporate chart transition functions to compare points described by different points. We will pick up this discussion of connections in Chapter \ref{Chap:PTC} where we will define them more formally and use them to define convolution operators on manifolds. 

For further reading on these topics we suggest some standard texts: \cite{baker1991introduction}, \cite{do2016differential} and \cite{jost2008riemannian}.


\section{Discrete Differential Geometry}

In many applications, it is not possible to have descriptions of manifolds which are as precise as introduced in the previous section. In general, the computational representation of surfaces that we have access to is a collection of points sampled from them. Often times, especially for 2-manifolds embedded in $\RR^3$, which are of particular interest since they represent shape we encounter in the 'real world', we also have an additional simplicial structure, most frequently in the form of a triangle mesh. 

A \textit{triangle mesh} is discrete representation of a surface defined by the tuple $\T=(V,E,T)$. Where the set $V$ is an indexed set of the points (or \textit{vertices}), $E$ is a set of \textit{edges} which connect the points in $V$ into \textit{triangular elements} $T$. Some groups prefer working with quadrilateral other polygon elements, but the underlying philosophy and maths are similar. For simplicity, we stipulate that all meshes we are interested contain only triangular elements, and obey the Delaunay condition \cite{delaunay1934sphere}: no point $x\in V$ is inside the circumcircle of any triangle. Not every set of points is guaranteed to have such a triangulation, for example if all the points are sampled form a line, but these pathological cases are rare in practice. 

Given such a mesh, we can easily compute many important geometric and differential properties efficiently by looping through either $V$ or $T$. For example, for a 2-manifold in $\RR^3$ we define the normal at any vertex $x_i$ as the average of the face normals in the first ring structure. That is $n(x_i) = \sum_{T \in i} \frac{1}{Area(T)} n(T)$. This also gives us a definition of the tangent plane at any given vertex: it is simply the plane normal to $n_i$. We can compute integrals of vertex valued functions by constructing a \textit{mass matrix} of local area elements. This can be done in the standard finite element fashion: by construing pyramid functions for each element and computing their inner products or by simply assigning each vertex an area corresponding to a third of the area of its first ring structure. Similarly, we can define differential operators such as the Laplacian (which we refer to as the Laplace-Beltrami operator on surfaces) thought standard finite element methods \cite{dziuk2013finite}.

In cases where no such simplicial structure is readily available and it is too costly to explicit compute one, we work with \textit{point clouds}: list of coordinates embedded in some space. Although this data has minimal structure if the data is sampled sufficiently densely, we can approximate triangle meshes through the use of local methods as in \cite{lai2013local}. Here, we pick up small neighborhoods around each point, and compute a locally valid mesh which we can use to perform finite element operations. In chapter \ref{Chap:CAE} we propose a method for computing parameterization of data sets based on learning local approximations in cases where the both the dimension of the manifold and its embedding are very high and in chapter \ref{Chap:Approx} we show some theoretical results for approximating such data with high dimensional simplicial structures. 

For further background on these subject we suggest \cite{crane2018discrete} and \cite{bobenko2008discrete} for discussion specifically on discrete differential geometry and \cite{dziuk2013finite} for discussion on finite element theory for surfaces.


\section{Optimization}

Although not the primary focus of this work, we would be remiss not to include a brief section on optimization. Many of the problems we will encounter can be broadly put into the same frame work: "Find a object in a space which minimizes some condition". In traditional setting, we call the condition we are trying to minimize the \textit{objective function} and the space of possible objects with which to minimize the objective the \textit{search space}. The deep learning community often replaces the term objective function with \textit{loss} and splits the search space into two parts a \textit{model space} of possible configuration of operations and a \textit{parameter space} of values to use in the operations (more on this in \ref{sec:DL}). In upcoming chapters we will use the field appropriate terminology.

The objective functions we are concerned with are almost always either differentiable or at least sub-differentialbe and are almost always amenable to first order methods which rely only on gradients (or sub-gradients). Higher order methods (which require the computation of the Hessian or possibly even higher orders partial derivatives) are common in some fields, but the memory requirements make them unsuitable for our tasks. Importantly, in this work, we are always able to find (although in practice it may be very expensive) a \textit{ direction of descent}: given an objective function $L$ of variable $x$ at any point $x_k$ we can find a direction $p_k$ such that $\langle p_k , \nabla f (x_k) \rangle < 0$. Then by Taylor's theorem we know that for some sufficiently small step size $\tau$, we can update $x_{k+1} = x_k + \tau p_k$ and guarantee that $f(x_{k+1}) < f(x)$. Two question naturally arise: how do we choose find $p_k$?; how do we choose $\tau?$. 

In many settings the exact gradient can be computed, and in these cases the negative of the gradient can be used as a direction of descent. However, in many other cases, especially those coming from machine learning or other 'big data' regimes, evaluation of the entire gradient is too computationally costly to be feasible. For example when training a neural network to classify a large training set it is prohibitively expensive to compute the objective function value and gradient with respect to every single training example. Here, we have to rely on stochastic approximations of the gradient. In general, these methods rely on an direction of descent conditions which hold only in expectation. Fortunately there have been many studies, both theoretical and empirical, which show that we can quickly compute stochastic gradient estimates which are highly likely to be directions of descent and display good convergence properties. We recommend \cite{bottou2010large} for an overview of these results in big data contexts.

Computing optimal step sizes $\tau$ is a field of study in itself, and is beyond the scope of this work, but we make a few remarks on the topic here. The simplest method to choose $\tau$ is to simply choose a fixed step size, and if the algorithm fails to converge, reduce the size. Adaptive step sizes algorithms are a much more powerful and common alternative. In general these methods work by evaluating the objective function at several step lengths to adaptive choose a step size. These methods are very powerful for problems in which evaluation of the objective function is computationally simple and we will utilize one in chapter \ref{Chap:LBBP}. However, in many of the big data regimes we are concerned with, since evaluation of the objective function is very costly and stochastic approximations do not provide good enough estimations, line search methods are not feasible in practice. In these cases another alternative are momentum methods (sometimes called accelerated gradient methods). These methods choose the step size by measuring some properties of the descent direction (such as the norm of the gradient) and can also incorporate past evaluations of the gradient, or previous step sizes in the computation of $\tau_{k+1}$. These methods will be extremely valuable to us when training neural networks in chapter \ref{Chap:PTC} and chapter \ref{Chap:CAE}.

The issue of convexity, or more specifically non-convexity, will come up many times in this work. A function $f:\RR^n \rightarrow \RR$ is said to be convex if:
\begin{equation}
    f(r x+(1-r)y) \leq r f(x) + (1-r)f(y) \quad \forall x,y \in \text{domain}(f), r \in [0,1]
\end{equation}
Convex functions are nice to work with in optimization because first-order methods, with appropriate step sizes, will always converge to the global minimum. Unfortunately, nearly all of the problems we will encounter in this work are non-convex and because of this we can only guarantee convergence to local minimums. However, this is not as large of a problem as it may at first appear. One source of non-convexity, which will crop up in chapter \ref{Chap:LBBP}, will be the result of symmetries and each of the local minimums encountered being equivalent. Similarly, deep neural networks, like those we will study in chapter \ref{Chap:PTC} and chapter \ref{Chap:CAE} are famously non-convex, but in practice we find that many different initialization and stochastic seeds lead to very similar results. This is likely due to the \textit{egg crate hypothesis}, which postulates that the loss surfaces of these networks look like an egg crate mattress, with many local minimums all of which are equally low. Figure \ref{fig:egg} shows an example of such a loss surface. 

\begin{figure}
    \centering
    \includegraphics[width=.95\linewidth]{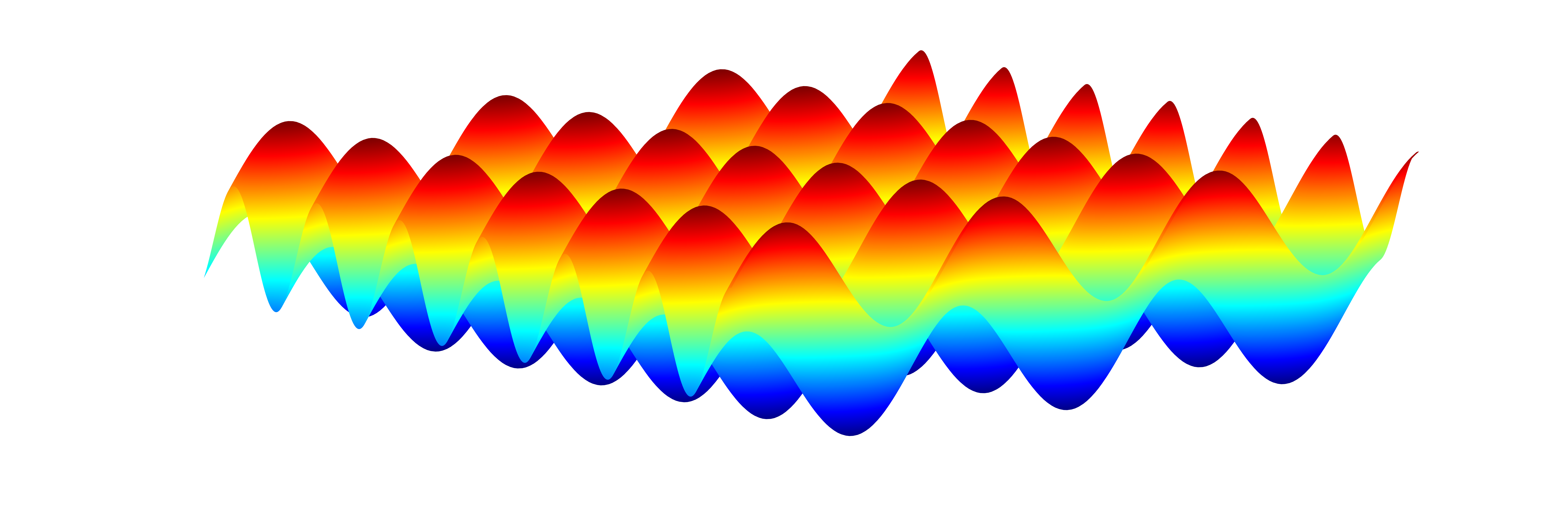}
    \lrpicaption{Egg crate-shaped loss function; all local minimums are equivalent.}
    \label{fig:egg}
\end{figure}

For more information on these topics we recommend \cite{rockafellar1970convex}, \cite{nocedal2006numerical} and \cite{bazaraa2013nonlinear} as classical references, and recommend our colleague's new book \cite{kupferschmid2017intro} for an specifically computational overview of these issues.


\section{Deep Neural Networks}\label{sec:DL}

The term \textit{deep learning} has come to encompass many related techniques in the broader field of data processing. For now we limit our discussion to feed-forward neural networks. A \textit{feed forward neural network} is a mathematical model in which cascade of linear (or more specifically affine) and non-linear operations transform some input data into some output data. Given some input data $x$ we can write a feed-forward neural networks as:
\begin{equation}
y = \sigma_k (W_k \sigma_{k-1} (...\sigma_2 (W_2 \sigma_1 (W_1 x + b_1 ) + b_2) ... )+ b_{k-1} ) 
\end{equation}
where $W_i$ are matrices referred to as the \textit{weights} and $b_i$ are vectors referred to as the \textit{biases} and collectively referred to as the \textit{parameters} of the networks. In general the internal non-linearities, $\{\sigma_i]\}_{i=1}^k$, are chosen to be simple fixed functions such as the relu functions $(\sigma(x) = max(x,0))$ or sigmoid function $(\sigma(x) = \frac{e^x}{e^x+1})$ and the last non-linearity is chosen specifically for the task at hand. These models are extremely general and have been shown to be able to approximate functions arbitrarily well in a wide variety of settings. Each pair of linear and non-linear (sometimes called activation) function form a \textit{layer} of the network and the \textit{width} of a layer is defined as the number of rows in the matrix associated with that layer. The \textit{depth} of a network refers to how many layers it contains. Generally speaking, the \textit{hyper parameters} of a network are the number of layers, their size, and the types of non-linearities employed. These are frequently chosen experimentally, although there are many ongoing research projects, all beyond the scope of this work, into more rigorous ways to choose them.  
 
The values of parameters of a network are determined through some \textit{training} procedure, most commonly optimizing some loss function over some set of \textit{training data}. This is most commonly done through a process known as \textit{back propagation} in which the derivative of the loss is propagated through the network layer by layer to update each weight. From a mathematical point of view this is simply the chain rule in action, and allows for memory-efficient gradient methods. Since networks are general approximators they are very prone to over-fitting. That is, they may perform very well on the training data, but fail to perform well on data not included in the training set. To overcome this most loss functions are augmented with regularization terms which depend on the parameters of the network. As in inverse problems when working with an underdetermined problem, a regularization terms work to make the model less sensitive to changes in the input (training) data and encourage some coherence (continuity, smoothness or flatness) of the network. The simplest and most common regulation is an $l_2$ penalty on the norm of the weight matrices. 

In recent years, convolutional neural networks have become extremely popular for processing one and two-dimensional data. In these networks the linear layer is specified to be convolutional operators (note that convolution is a linear operation and that discrete convulsion can be implemented as a sparse matrix multiplication). These layers have several advantages: they can be implemented extremely efficiently, they create a contraction of information at each layer (since the filters have compact support), and they exhibit desirable equivariance properties since the same filter is applied multiple times across different location in the domain. In chapter \ref{Chap:PTC} we will generalize these networks to apply on non-flat domains, and in chapter \ref{Chap:CAE} we will use convolutional layers in a more traditional setting to analyze data embedded in image space. 

The book \textit{Deep Learning} \cite{goodfellow2016deep} has quickly become the standard reference in the field, but we also recommend \cite{deisenroth2020mathematics} for a more mathematical treatment and \cite{abu2012learning} for a broader overview of the subject of machine learning. 


\section{A Motivating Problem: Correspondence}\label{sec:MotivatingProblem}

We finish this chapter with a motivational problem which will be fundamental in chapter \ref{Chap:LBBP}, and will reappear through this work: correspondence. Put colloquially: given two similar shapes, find a geometrically meaningful map between them. 

To give a concrete example, which we will return to again in the next chapter, suppose we want to compute a correspondence between a horse and an elephant. Each has four legs, two eyes, and a tail, but elephants have large ears and trunk much large than any horses nose. Our task, then, is to find a map between the two which preserves geometrically meaningful information: elephant feet should map to horse hooves, tail should map to tail, and head to head. Once we have this map many downstream tasks become easy. We can characterize the deformation by studying the map. We can classify new shapes based on the properties of the map. We use the map to transfer information (such as labels from segmentation data) from one surface to the other, and we can combine maps from different shapes to create networks that we can also analyze. 

For a more abstract example, instead suppose that we have two data sets: one of handwritten digits and one of photographs of house numbers. Conceptually, the information between them is similar: the number 0 through 9, however, the way that this information is represented is very different. If these data sets are very large, then computing a point-to-point mapping between them is unfeasible. However, if we take the point of view that these data sets are simply two different embeddings of samples taken from the same manifold, then another correspondence appears: compute the map between the sampled embedding, and an intrinsic representation of the data. Nearly all previously established methods for understanding this \textit{latent} representation try to do so by mapping the data to some linear, or normally disturbed, space. However, if our data displays more complex geometry (or topology) this kind of mapping will destroy the geometric information we are interested in. In chapter \ref{Chap:CAE} and chapter \ref{Chap:Approx} we overcome this by returning to fundamentals; using the language of charts and local embeddings to create latent models where computing correspondences are once again geometrically meaningful. 
 
\chapter{LAPLACE-BELTRAMI BASIS PURSUIT }\label{Chap:LBBP}


Surface registration is one of the most fundamental problems in geometry processing. Many approaches have been developed to tackle this problem in cases where the surfaces are nearly isometric. However, it is much more challenging  to compute correspondence between surfaces which are intrinsically less similar. In this paper, we propose a variational model to align the Laplace-Beltrami (LB) eigensytems of two non-isometric genus zero shapes via conformal deformations. This method enables us compute to geometricly meaningful point-to-point maps between non-isometric shapes. Our model is based on a novel basis pursuit scheme whereby we simultaneously compute a conformal deformation of a 'target shape' and its deformed LB eigensystem. We solve the model using an proximal alternating minimization algorithm hybridized with the augmented Lagrangian method which produces accurate correspondences given only a few landmark points. We also propose a re-initialization scheme to overcome some of the difficulties caused by the non-convexity of the variational problem. Intensive numerical experiments illustrate the effectiveness and robustness of the proposed method to handle non-isometric surfaces with large deformation with respect to both noise on the underlying manifolds and errors within the given landmarks or feature functions. 

\blfootnote{Portions of this chapter previously appeared as: S. C. SCHONSHECK, M. BRONSTEIN AND R. LAI, \textit{Noniometirc Surface Registration via Conformal Laplace-Beltrami Basis Pursuit},  arXiv preprint, arXiv:1809.07399, 2018. \newline
Portions of this chapter have been submitted as S. C. SCHONSHECK, M. BRONSTEIN AND R. LAI  \textit{Noniometirc surface registration via conformal Laplace-Beltrami Basis Pursuit}, J. Sci. Comput (2020).}

\section{Introduction to Nonisometric Correspondence}\label{sec:IntroLBBP}

The computation of meaningful point-to-point mappings between pairs of manifolds lies at the heart of many shape analysis tasks. It is crucial to have robust methods to compute dense correspondences between two or more shapes in different applications including shape matching, label transfer, animation and recognition~\cite{siddiqi1998area,kraevoy2004cross,reuter2006laplace,heimann2009statistical,van2011survey,ovsjanikov2012functional}. In cases where shapes are very similar (isometric or nearly isometric), there are many approaches for computing such correspondences \cite{elad2003bending,gu2004genus,bronstein2006efficient,aubry2011wave,kim2011blended,kovnatsky2013coupled,lai2017multi,ovsjanikov2012functional,shtern2014iterative,shtern2016fast}. However, it is still challenging to compute accurate correspondences when the deformation between the shapes are far away from near isometry.

One of the key challenges in largely deformed non-isometric shape matching is that the intrinsic features of the two shapes are not similar enough for standard techniques to recognize their similarity. For example, when computing the correspondence between human faces, it is not particularly difficult to geometrically characterize the structure of a `nose'. However, similar techniques can not work well to compute a map between a horse and an elephant face since these two surfaces have many largely deformed local structures including the drastic difference between the trunk of the elephant and the nose of the horse. Because of this, it is crucial to develop new methods to adaptively characterize large deformations on surfaces. 

The LB eigensystem is a ubiquitous tool for 3D shape analysis (see \cite{berard1994embedding,reuter2006laplace,levy2006laplace,Rustamov:2007,sun2009concise,Bronstein:2010CVPR,lai2010metric,Raviv2011CVPR,aubry2011wave,lai2011automated,shi2013cortical,shi2014metric,raviv2015affine,shtern2016fast} and references therein). It is invariant under isometric transformations and intrinsically characterizes the local and global geometry of manifolds through its eigensystem up to an isometry. In principle, the LB eigensystem reduces infinite-dimensional nonlinear isomorphism ambiguities between two isometric shapes to a linear transformation group between two LB eigensystems. This linear transform is necessary due to the possible sign or sub-eigenspace (geometric multiplicity) ambiguity of LB eigensystems~\cite{lai2017multi}. Additionally, similar shapes often have similar eigensystems which allows for joint analysis of similar shapes their spectral properties~\cite{ovsjanikov2012functional}. However, when the deformation between two shapes is far from an isometry, the large dissimilarity between LB eigensystems of two shapes is the major bottleneck to adapt the existing spectral geometry approach to conduct registration. 

A natural idea to extend spectral geometry methods to register non-isometric surfaces is to deform the metric of a "target surface" into the metric of a "source surface" so that two surfaces share similar LB eigensystem after deformation. However, directly computing this deformation often requires specific knowledge about corresponding regions of the shapes. In this work, we propose a method to simultaneously compute such a deformation while learning features which can be used for registration. Mathematically, one way to characterize this type of deformation is through measuring its conformal factor--the local scaling induced by a conformal deformation. It is well known that there exists a conformal mapping between any two genus-zero surfaces~\cite{jost2008riemannian}. Rather than reconstruct the conformally deformed surfaces and/or exact conformal map, we exploit a fundamental link between the conformal factor and the LB eigensystem by manipulating the conformally deformed LB eigensystem. This allows us to compute a new basis on the target surface to align the naturally defined LB eigensystem on the source surface. This leads to a variational method for non-isometric shape matching which enables us to overcome the natural ambiguities of the LB eigensystem and align the bases of non-isometric shapes while avoiding the direct computation of conformal maps. 

Numerically, we solve our model using a proximal alternating minimization (PAM) method~\cite{attouch2010proximal} hybridized with the augmented Lagrangian method~\cite{glowinski1989augmented}. The method is iteratively composed of a curvilinear search method on orthogonality constrained manifold~\cite{wen2013feasible} in one direction to compute the conformally deformed LB eigenfunctions and  the BFGS~\cite{bazaraa2013nonlinear} method for the other direction to compute the conformal factor. Theoretically, we guarantee the local convergence of the proposed algorithm since the objective function and constraints satisfy the necessary Kurdyka-Lojasiewicz (KL) condition~\cite{attouch2010proximal}. Numerical results on largely deformed test problems, including horse-to-elephant and Faust benchmark database~\cite{bogo2014faust}, validate the effectiveness and robustness of our method.

\paragraph{Related Works.}
 A large number of 3D nonrigid shape matching approaches are based on analysis of the LB eigensystem~(see \cite{reuter2005laplace,reuter2006laplace,levy2006laplace,Rustamov:2007,Bronstein:2010CVPR,ovsjanikov2012functional,kovnatsky2013coupled,Raviv2011CVPR,shtern2014iterative,lai2017multi} and reference therein). The LB eigensystem is intrinsic and invariant to isomorphism, and also characterizes the local and global geometry of a manifold. This makes it ideal for many shape processing tasks and many early works in the field involve directly comparing the LB spectrum of the shapes to determine how alike shapes are \cite{reuter2005laplace,reuter2006laplace,levy2006laplace}. More recently, the general concept of functional maps \cite{ovsjanikov2012functional} has played a central role in many new methods that have allowed for the formulation of accurate correspondence maps. This technique essentially reduces the non-linear transform between two shapes to a linear transform between their eigensystems. In general, these techniques work for well for isometric and near isometric cases, but can not produce satisfactory results when the LB eigensystems of shapes are very dissimilar. This occurs when the deformation between shapes is far from an isometry. To overcome this, the concept of coupled bases (also known as joint-diagonalization) was introduced for shape processing tasks in \cite{kovnatsky2013coupled}. In this work the authors propose a variational model to define a shared basis for a pair of shapes which is `nearly harmonic' on one shape and 'similar' to the natural LB basis on the other. This joint optimization allows for much more accurate correspondence maps, but does not characterize the underlying deformations which lie at the heart of the non-isometric shape matching problem.

Conformal maps have been widely applied to various shape processing tasks in order to characterize these deformations \cite{Hurdal:NeuroImage2000,Haker:TVCG2000,gu2004genus,springborn2008conformal}.  
In one of the first works to combine spectral and deformation based approaches, \cite{shi2011conformal} presents a scheme to find optimal conformal deformation to align two shapes in the embedded LB Space. Additionally, the authors present a general framework for computing LB eigensystems of conformally deformed surfaces as well as several other imported related quantities. Continuing on this line of work in \cite{kao2017maximization}, the authors use the LB eigenvalues as a tool to guide conformal deformations. Using derivatives of the LB eigenvalues, they compute optimal conformal metrics which approximate conformal and topological eigenvalues. In our work, we use the spectral coefficients of known features to guide the deformation, so rather than align the eigenvalues we align the eigenfunctions. This allows us to avoid the subspace ambiguity of the LB eigensystem and computational errors in calculating high-frequency eigenvalues. 

\paragraph{Major Contributions.}
We introduce a novel variational basis pursuit model for computing non-isometric shape correspondences via conformal deformation of the LB eigensystem. This model enhances spectral approaches from handling nearly isometric surface registration to tackling surfaces with large deformed metrics. It naturally combines the conformal deformation to the LB eigensystem and simultaneously computes surface deformations and LB eigenbasis which also automatically overcomes the ambiguities of LB eigensystems in surface registration. 
We also propose a numerical scheme to solve the variational model with a local convergence guarantee. Additionally, we introduce a reinitializaiton scheme to help tackle local minima and improve the quality of the computed bases. This algorithm successfully handles non-isomorphic shape correspondence problems given only a few landmarks and is shown to be robust to noise and perturbations of landmarks.

The rest of this chapter is organized as follows: In section \ref{sec:LBBPBackgroud}, we review the theoretical background of conformal deformation of LB eigensystem and functional maps. After that, we propose the variational basis pursuit model for conformal deformation of the LB eigensystem in section \ref{sec:LBBasisPursuit}. In section \ref{sec:Algs}, we discretize the model and develop an optimization scheme based on PAM to solve the variational problem. Section \ref{sec:Discussion} is further devoted to discuss a few details of the model and a reinitialization scheme to improve our numerical solver. In section \ref{sec:Results}, numerical results on several data sets are presented to show that the model accurately produces point-to-point mappings on non-isometric manifolds with large deformation given only a few landmark points. We also show that our approach is robust to both noise in the underlying manifolds and inaccuracies in the initial landmarks. Furthermore, we test the model to a benchmark data based to show its effectiveness. Lastly, we conclude our discussions of this project in Section~\ref{sec:LBBPconclusion}.

\section{Mathematical Background of LBBP}\label{sec:LBBPBackgroud}
In this section, we discuss the mathematical background of the proposed method. We first review a few key properties of the LB eigensystem of a Riemannian surface and discuss its conformal deformations with respect to deformations of the Riemannian surface metric \cite{chavel1984eigenvalues,jost2008riemannian}. After this, we review the functional maps framework in \cite{ovsjanikov2012functional} which will be closely related to our work.


%
\subsection{Conformal Deformation of LB Eeigensystem on Riemannian Surfaces}
\label{subsec:LB}

Given a closed Riemannian surface $(\M,g)$, its LB operator in a given local coordinate system, $\{x_i\}_{i=1,2}$,  is defined as \cite{chavel1984eigenvalues,jost2008riemannian}:
\begin{equation}\label{def:LB}
\Delta_{g}\phi=\frac{1}{\sqrt{G}}\sum_{i=1}^{2}\frac{\partial}{\partial
x_i}(\sqrt{G}\sum_{j=1}^{2}g^{ij}\frac{\partial \phi}{\partial x_j})
\end{equation}
where $(g^{ij})$ is the inverse of the metric matrix $g = (g_{ij})$
and $G=\det(g_{ij})$. The LB operator is self-adjoint and elliptic, therefore it has a discrete spectrum. We denote the eigenvalues of $-\Delta_{g}$ as $0=\lambda_0 < \lambda_1 \leq \lambda_2 \leq\cdots$ with the corresponding eigenfunctions $\phi_0, \phi_1,\phi_2,\cdots$ satisfying:
\begin{equation} 
-\Delta_g(x) \phi_i(x) =\lambda_i\phi_i(x), \quad \text{and} \quad 
\int_{\M}\phi_i(x) \phi_j(x) ~\mathrm{d}vol_{g}(x)= \delta_{ij}, \quad i,j = 0,1,2,\cdots 
\end{equation}
where $\mathrm{d}vol_{g}(x)$ is the area element on $\M$ with respect to $g$. It is well-known that $\Phi = \{\phi_n~|~ n=0, 1, 2, \cdots\}$ forms an orthonormal basis for the real-valued, smooth function space $\Fun (\M,\RR)$ on the manifold $(\M,g)$. This basis can be viewed as a generalization of the Fourier basis from flat space to a differentiable manifold. 
The LB eigensystem is invariant under both rigid and nonrigid isometric transformations,and it uniquely determines a manifold up to isometry \cite{berard1994embedding}.

In differential geometry, a conformal map is one which preserves angles locally. Formally, a conformal map preserves the first fundamental form up to a positive scaling factor. Given two manifolds $(\M_1,g_1)$ and $(\M_2,g_2)$, a map $F:(\M_1,g_1) \rightarrow (\M_2,g_2)$ is conformal if and only if the pullback $F^*(g_2) = w^2 g_1$ with a positive function $w^2$ (written this way to emphasize positivity). A {\it conformal deformation} of a surface is a transformation which changes the local metric by a positive scaling factor. A well-known result in conformal geometry is that there exists a conformal map between any two genus-zero surfaces \cite{jost2008riemannian}. 

Given a closed surface $(\M,g)$ with conformal deformation $w^2$, the LB eigensystem of the deformed manifold $(\M, w^2 g)$ can be viewed as a weighted LB eigensystem on the original surface $(\M,g)$. This simple fact intrinsically links the LB eigensystem of the deformed manifold to a weighed LB eigensystem on the original manifold. It allows us to compute the LB eigensystem of the conformally deformed manifold without explicitly reconstructing its embedding or coordinates. This also relates information about the local deformation and global eigensystem and later becomes the cornerstone of our approach. Formally, we have: 
\begin{proposition}
\label{prop:LBconformal}
Let $\{\phi_n^{w^2},\lambda^{w^2}_n\}_{n=1}^\infty$ be a LB eigensystem of a conformally deformed surface $(\M,w^2 g)$, then $\{\phi_n^{w^2},\lambda^{w^2}_n\}_{n=1}^\infty$ is equivalent to the following weighted LB eigensystem on $(\M,g)$:
\begin{eqnarray}\label{eqn:ConformalLBeigs}
-\Delta_{g}\phi_i(x) = \lambda  w^2(x)\phi_i(x), \qquad
\int_{\M}\phi_i(x)\phi_j(x) w^2(x)~\mathrm{d}vol_{g}(x)= \delta_{ij},
\end{eqnarray}
\end{proposition}
\begin{proof} This is because: 
\begin{align*}
\Delta_{w^2g} \phi
= \frac{1}{w^2\sqrt{G}}\sum_{i=1}^{2}\frac{\partial}{\partial
x_i}(w^2\sqrt{G}\sum_{j=1}^{2}w^{-2}g^{ij}\frac{\partial \phi}{\partial x_j}) = w^{-2}\Delta_{g} \phi
\end{align*}
Hence the eigen problem: $-\Delta_{w^2g} \phi = \lambda \phi$ is equivalent to $-\Delta_{g}\phi = \lambda  w^2\phi$. Additionally, it is clear that: $\mathrm{d}vol_{w^2 g} = w^2\ \mathrm{d}vol_{g}$, since changing the local metric is equivalent to rescaling the local area element.
\end{proof}

The problem of finding the LB eigensystem of a Riemannian manifold is equivalent to finding an orthonormal set of functions $\Phi = \{\phi_i\}$ which have minimal harmonic energy on the surface. From the above proposition, the LB eigensystem of a conformally deformed manifold $(\M,w^2 g)$ can be formulated as the following variational problem: 
\begin{equation}\label{defeig}
\arg\min_{\Phi=\{\phi_i\}} \sum_i  \int_{\M} ||\nabla_{\M} \phi_i(x) ||^2~\mathrm{d}vol_{g}(x),  \quad \text{s.t.}\quad \int_M \phi_i(x) \phi_j(x) w^2(x) ~\mathrm{d}vol_{g}(x) = \delta_{ij}
\end{equation}


\subsection{Functional Maps}
\label{funmaps}
Functional maps were introduced in~\cite{ovsjanikov2012functional} for isometric and nearly isometric shape correspondence. This method has been shown a very effective tool for various shape processing tasks \cite{ovsjanikov2012functional,kovnatsky2013coupled,rodola2015point}. Here we provide a basic overview of their framework.
Consider Riemannian surfaces $(\M_1,g_1)$ and $(\M_2,g_2)$, a smooth bijection $F:\M_1 \rightarrow \M_2$ induces a linear transformation between functional spaces of these two manifolds as:
\begin{equation}
F_T: \Fun(\M_1,\RR) \rightarrow  \Fun(\M_2,\RR),\quad 
f \mapsto f\circ F^{-1}
\end{equation}
Instead of computing surface map $F$, the crucial idea of functional map is to  compute the linear map $F_T$ between these two functional spaces. After that, the desired surface map can be encoded by considering images of indicator functions under $F_T$.

Finding a functional map, $F_T$, associated with a map $F$ is equivalent to finding the matrix representation of $F_T$ under a fixed orthonormal basis $\{\phi_i\}$ of $\Fun(\M_1,\RR)$ and a fixed orthonormal basis $\{\psi_i\}$ of $\Fun(\M_2,\RR)$, respectively.  Namely, if we write $F_T(\phi_i) = \sum_{j}c_{ji} \psi_j$, then any two given corresponding functions $ f =  \sum_i f_i \phi_i$ and $g = \sum_j g_j \psi_j$ under $F_T$ can be represented using $C = (c_{ij})$ as: 
\begin{equation} \begin{split}
F_T(f) = g \Leftrightarrow F_T\Big(\sum_i f_i \phi_i\Big) 
=\sum_i f_i F_T ( \phi_i) =
\\ \sum_i f_i \sum_j c_{ji} \psi_j = \sum_j g_j \psi_j 
\Leftrightarrow \sum_i c_{ji} f_i = g_j. 
\end{split}
\end{equation}
Each entry of the matrix $c_{ij}$ can be found by finding the $j^{th}$ coefficient of $F_T(\phi_i)$ expressed in the $\{\psi_i\}$ coordinate system, i.e. $c_{ji} = \langle F_T(\phi_i), \psi_j \rangle_{g_2}$. In practice, one can use two finite sets of orthonormal functions to approximate $\Fun(\M_1,\RR)$ and $\Fun(\M_2,\RR)$, thus the functional map can be approximated by a finite dimensional matrix. For instance, the first $N$ eigenfunctions of the LB eigensystem is one common choice of such a basis. Then, the problem of finding the transformation $F_T$ can be approximated by the problem of seeking a finite dimension matrix $C$. As long as $C$ is computed, the desired map $F$ can be computed through $C$ operating on indicator functions. 

\section{Conformal LB Basis Pursuit for Nonisometric Surface Registration}
\label{sec:LBBasisPursuit}
In this section, we propose a LB basis pursuit model for non-isometric surface registration. On the target surface $\M_2$, the model simultaneously finds a conformal deformation and a conformally deformed LB eigensystem so that the coefficients of the corresponding feature functions expressed on the deformed LB eigensystem of $\M_2$ are the same as the coefficients on the fixed source surface $\M_1$. 


\subsection{Variational PDE Model}


Given two non-isometric genus-zero closed Riemannian surfaces $(\M_1,g_1)$ and $(\M_2,g_2)$, we aim at finding a geometrically meaningful correspondence between these two surfaces. In the case that $\M_1$ and $\M_2$ are nearly isometric, there are many successful methods to constructing maps between $\M_1$ and $\M_2$ by comparing their isometric invariant features. Using spectral descriptors from solutions of the LB eigensystem on manifolds is a common way of constructing such descriptors~\cite{reuter2006laplace,levy2006laplace,vallet2008CGF,Shi:08a,Bronstein:2010CVPR}. As extensions, some other descriptors such as Heat kernel signature~\cite{sun2009concise}, wave kernel signature~\cite{aubry2011wave} and optimal spectral descriptors~\cite{litman2014learning} have also been proposed in the literature. However, most of the existing methods consider the construction of descriptors for nearly isometric manifolds. Registration methods based on the existing LB spectral descriptors can not provide satisfactory results for constructing correspondence between two non-isometric surfaces as their eigensystems are possibly quite far apart. 

We propose to overcome the limitation of the LB spectral descriptors for largely deformed non-isometric shape registration by considering a continuous deformation of the LB spectral descriptors. Intuitively, given two non-isometric shapes $(\M_1,g_1)$ and $(\M_2,g_2)$, our idea is to deform the metric of $(\M_2,g_2)$ such that the deformed surface is isometrically the same as $(\M_1,g_1)$. Then the LB spectral descriptors can be applied as in isometric shape matching. However, it is challenging to find an appropriate deformation as the accurate amount of deformation on each local region of $\M_2$ depends exactly on an accurate correspondence which is precisely the problem we would like to solve.

To handle this challenge, we propose to simultaneously find an optimal correspondence and an optimal deformation. More specifically, by fixing the LB eigensystem $\{\Phi, \Lambda\}$ of $(\M_1, g_1)$, we seek a map $T:\M_1 \rightarrow \M_2$ and a conformal factor $w^2:\M_2 \rightarrow \RR^{+}$ such that the LB eigensystem $\{\Phi, \Lambda\}$ of $(\M_1,g_1)$ can be aligned to the LB eigensystem $\{\Psi, \Theta\}$ of $(\M_2, w^2 g_2)$ via $T$. 
This problem can be written as the following variational PDE problem:
\begin{equation}\begin{split}
\label{eqn:map1}
(T^*,w^*, \Psi^*) &=
\underset{T, w,\Psi = \{\psi_i\}_{i=1}^N }{\argmin} ~ \sum_{i=1}^{N} \int_{\M_1} \|\phi_i - \psi_i\circ T\|^2 ~\mathrm{d}M_1  \\
&+ \frac{1}{2} \sum_{i=1}^N \int_{\M_2} \|\nabla_{\M_2} \psi_i \|^2 ~\mathrm{d}\M_2,  \\
 & \text{s.t.}\quad  \int_{\M_2} \psi_i \psi_j \ w^2 ~\mathrm{d}\M_2  = \delta_{ij}
\end{split}
\end{equation}
where $\mathrm{d}\M_1 =dvol_{g_1}, \mathrm{d}\M_2 = dvol_{g_2} $ and $w^2 \mathrm{d} \M_2 = dvol_{w^2 g_2}$.The first term measures the alignment of two bases as the correct correspondence should map one LB eigensystem to another one, and the second term solves the first $N$ LB eigenfunctions $\{\psi_i\}$ for the deformed manifold $(\M_2,w^2 g_2)$ due to the variational problem~\eqref{defeig}. 
Existence of a solution to this variational problem \eqref{eqn:map1} is guaranteed as any two genus-0 surfaces are conformally equivalent and the LB operator is invariant under isometric transformations.


Computationally, the numerical search for $T$ in the mapping space is usually very time-consuming. Inspired by the idea of functional maps \cite{ovsjanikov2012functional} and the coupled quasi-harmonic bases \cite{kovnatsky2013coupled}, we choose to represent $T$ in the functional space. Instead of finding $T$ directly, we look for a basis $\Psi = \psi_i\circ T = F_T(\psi_i)$ which is nearly harmonic on $(\M_2,w^2 g_\M)$ and represents the corresponding features with the same coefficients as $\Phi$ does.  
More precisely, given a set of corresponding features $F=\{f_1,\cdots, f_k\} $ on $\M_1$ and $G = \{g_1, \cdots, g_k\}$ on $\M_2$, such that $f_i(x) = g_i(y)$ if $x$ and $y$ are corresponding points on $\M_1$ and $\M_2$, we can replace the direct measurement of the basis alignment term with a coefficient matching term. 
That is, instead of measuring the alignment of $\Psi$ and $\Phi$ via $T$, we measure how closely the coefficients for $G$ in the computed basis $\Psi$ match the coefficients for $F$ in the fixed LB basis $\Phi$. Formally, we measure the coefficient alignment by constructing a matrix of the coefficients in for $F$ in $\Phi$ and for $G$ in $\Psi$ so that the $ij^{th}$ term represents the coefficient for the $i^{th}$ corresponding function in the $j^{th}$ basis and computing their difference under the Frobenius norm. With this in mind, we propose the following model:
\begin{equation}
\label{eqn:map2}\begin{split}
(w^*,\Psi^*) &= \underset{ w,\Psi}{\argmin} 
\frac{r_1}{2}\| \langle F, \Phi \rangle_{g_1} -  \langle G, \Psi \rangle_{w^2 g_2} \|_F^2 + \frac{r_2}{2} \sum_{i=1}^N \int_{\M_2} \|\nabla_{\M_2} \psi_i \|^2 \mathrm{d}\M_2, \\
&\hspace{3cm} \text{s.t.}\quad \int_{\M_2} \psi_i \psi_j \ w^2 \mathrm{d}\M_2  = \delta_{ij}
\end{split}
\end{equation}
where: 
\begin{equation} 
\langle F, \Phi \rangle_{g_1} = \Big(\int_{\M_1} f_i \phi_j \ \mathrm{d}\M_1 \Big)_{i,j = 1,2,\dots,k} \quad 
\end{equation}
and 
\begin{equation}
\langle G, \Psi \rangle_{w^2 g_2} = \Big(\int_{\M_2} g_i \psi_j\  w^2 \mathrm{d}\M_2 \Big)_{i,j = 1,2,\dots,k}.
\end{equation}
In practice we use indicator functions for $F$ and $G$, but heat signatures~\cite{sun2009concise}, wave kernel signatures~\cite{aubry2011wave}, or any other corresponding functions will also work.  Once $\Psi^* = \{\psi^*_1, \cdots, \psi^*_{M_2} \}$ is obtained, we can easily compute the functional map as 
\begin{equation}
\label{eqn:funmap}
F_T : C^\infty(\M_1) \rightarrow C^\infty(\M_2), \qquad F_T(h) = \sum_{i=1}  \Big(\int_{\M_1} h \phi_i \ \mathrm{d}vol_{g_1}\Big)~  \psi^T_i.
\end{equation} 

The main advantage of this model over previous existing methods for shape correspondence is that we are able to employ much more of the information encoded in the differential structures of $\M_1$ and $\M_2$ in our algorithm by combining the spectral descriptors and local deformations. This additional flexibility enables us to compute correspondences between largely deformed shapes. Information about the conformal deformation of the metric allows us to find a harmonic basis on the deformed shape, meanwhile information about the alignment of the functional spaces guides our calculation of the conformal deformation. Furthermore the additional constraint of the feature alignment overcomes ambiguity casued by the fact that there is no unique conformal deformation between any two genus zero surfaces.
To the best of our knowledge, the link between the conformal factor and deformed LB basis has not been exploited in such a way. Previous works have used only the conformal factor \cite{gu2004genus,kim2011blended} or only the functional space \cite{ovsjanikov2012functional,kovnatsky2013coupled} as stand alone tools rather than in concert as we present here. 


\subsection{Regularization and Area Constraint}

We add harmonic energy term to smooth the conformal deformation and regularize the problem. This can both increase the speed of the algorithm and improve the quality of the map, both in terms of the geodesic errors of the final correspondence, and the accuracy of the resulting conformal factor. This is particularly helpful to handle  deformations between the shapes which are far from isometry and to reduce the required number of features. Rather than smooth the conformal factor $w^2$ directly, we instead add the harmonic energy of $w$ to our objective function. Using $w$ instead of $w^2$ allows for easier analytic computation of the derivatives and a more efficient algorithm. In cases where the deformations are likely to be highly localized, this term may be omitted.


Lastly, we add an area preservation constraint to our model. That is, we would like the final deformed shape to be of the same size as the one we are matching it to. To enforce this, we mandate that the deformed manifold have the same surface area as the original manifold. This eliminates any scaling ambiguity. Then the final version of our model can be stated as:
\begin{equation} \label{eqn:finalmodel}
\begin{split}
(w^*,\Psi^*) =& \underset{w,\Psi=\{\psi_i\}_{i=1}^N }{\argmin}
\frac{r_1}{2}\| \langle F, \Phi \rangle_{g_1} -  \langle G, \Psi \rangle_{w^2 g_2} \|_F^2 \\
&+ \frac{r_2}{2} \sum_{i=1}^N \int_{M_2}  \|\nabla_{\M_2} \psi_i \|^2 \mathrm{d}\M_2
+\frac{r_3}{2} \int_{\M_2} ||\nabla_{\M_2} w ||^2 \mathrm{d}\M_2, \\ 
&\text{s.t.}\quad \int_{\M_2} \psi_i \psi_j \ w^2 \mathrm{d}\M_2  =\delta_{ij} \quad 
\text{and} \quad Area(\M_1)_{g_1} = Area(\M_2)_{w^2 g_2}  
\end{split}
\end{equation} 
where $Area(\M_1)_{g_1} = \int_{\M_1} 1 \mathrm{d}\M_1$ and $Area(\M_2)_{g_2} = \int_{\M_2} 1 w^2 \mathrm{d} \M_2$

\section{Discretization and Numerical Algorithms}
\label{sec:Algs}
In this section, we describe a discretization of the proposed variational model \eqref{eqn:finalmodel} using on triangular representation of surfaces. After that, we design a numerical algorithm to solve the proposed model based on proximal alternating minimization method. \footnote{We remark that this approach also works for point cloud representations of the manifolds, as the algorithm only relies on the mass and stiffness matrices which can be computed for point clouds as discussed in \cite{lai2013local}}


\subsection{Discretization of the LBBP Model}

The main method we use to discretize surfaces and differential operators is based on a finite element scheme similar to that developed in~\cite{reuter2006laplace,sun2009concise,dziuk2013finite}. Let $\{p_i\}_{i = 1}^{n}$ be a set of vertices sampled on the manifold $\M$. A surface can be discretized as a triple $\{P,E,T\}$ made of vertices ($P$), connected by edges ($E$) which form triangular faces $(T)$. We define the first ring of $p_i$, the set of all triangles which contain $p_i$ as $N(p_i)$. For each edge $E_{ij}$ connecting points $p_i$ and $p_j$, we define the angles opposite $E_{ij}$ as angles $\alpha_{ij}$ and $\beta_{ij}$.

We define a diagonal mass matrix, $\mathbf{M}$, a $n \times n$ positive definite matrix with entries given by:
\begin{equation}
    \mathbf{M}_{ii} = \frac{1}{3} \sum_{\tau \in N(p_i)}  Area(\tau)
\end{equation}
We use this simplified version, rather than the standard finite element discretization, for convince in order to avoid expensive factorizations later in our algorithm. We remark that the standard version can also be used in our algorithm at the cost of speed. The surface area can be approximated as $Area(\M) \approx \sum_{i=1}^n \textbf{M}_{ii}$. Similarly, a function $f:\M \rightarrow \RR$ with discretization $F:P \rightarrow \RR$, then we have the approximation $\int_{\M} f(x) \ d\M \approx 1^T\textbf{M}F = \sum_{i=1}^n f_i M_{ii}$. 
The stiffness matrix, $\textbf{S}$, is a $n \times n$ symmetric positive semidefinite matrix given by:
\begin{equation}
   {S}_{ij} = \sum_T \int_T \nabla_T e_i \cdot \nabla_T e_j = - \frac{1}{2}[\cot \alpha_{ij} (p_i) + \cot \beta_{ij} (p_i)] 
\end{equation}
where $e_i$ is a linear pyramid function which is $1$ at $p_i$ and zero elsewhere. These mass and stiffness matrices can be used to approximate the LB eigenvalue problem as: $\mathbf{S} f  =  \lambda \mathbf{M} f$ \cite{meyer2002discrete}. 

We remark that one can also work with point clouds representation instead of triangulated meshes. These definitions for the stiffness and mass matrices can be approximated by the point clouds method discussed in \cite{lai2013local}. The only change we would need to make is to use only the diagonal entries of the version of the mass matrix $\mathbf{M}$ proposed in their paper to populate the strictly diagonal version we employ here.

Suppose two surfaces $(\M_1,g_1),( \M_2,g_2) $ are represented by triangular meshes with the same number of points\footnote{In fact, we do not need to require that the surfaces have the same number of points, but doing so for now will allow for more convenient notation.}. We denote $\mathbf{M}^1, \mathbf{S}^1 \in \RR^{n\times n}$ as the mass and stiffness matrices of $\M_1$ and let $\Phi\in\RR^{n\times k}$ be the first $k$ LB eigenfunctions of $\M_1$, and $F \in\RR^{n\times \ell}$ be $\ell$ feature functions. 
Similarly, we write $\mathbf{M}^2, \mathbf{S}^2$ as the mass and stiffness matrices of  $\M_2$, $\Psi$ as the first $k$ LB eigenfunctions of $\M_2$ (under $g_2$), and $G$ as $\ell$ corresponding feature functions, ordered the same as in $F$. We also write $w^2$ as the discretized conformal factor on $\M_2$ and $\mathrm{diag}(w)$ as a diagonal matrix. 

Therefore, the discretized optimization model \eqref{eqn:map2} can be written as: 
\begin{equation}
\label{eqn:map2_discretization1}
\begin{split}
(w^*, \Psi^*) &= \arg\min_{w, \Psi} 
\frac{r_1}{2} \| F^T \mathbf{M}^1 \Phi - G^T  \mathrm{diag}(w) \mathbf{M}^2 \mathrm{diag}(w) \Psi \|_F^2 \\
& + \frac{r_2}{2} \tr (\Psi^T \mathbf{S}^2 \Psi) 
+ \frac{r_3}{2}w^T \textbf{S}_2 w, \quad 
\\ 
&\text{s.t.} \quad \Psi^T   \mathrm{diag}(w) \mathbf{M}^2 \mathrm{diag}(w) \Psi = \textbf{I}_k, 
\quad \text{and} \quad w^T \textbf{M}_2 w = A 
\end{split}
\end{equation}
Here $\textbf{I}_k$ is the $k \times k$ identity matrix and $A = \sum_{i=1}^n \textbf{M}_{1}(i,i) $. Since $\textbf{M}_2$ is symmetric positive definite and diagonal, we can easily calculate the matrix decomposition $\textbf{M}_2 = \textbf{L}^T \textbf{L}$. If we also substitute $\bar\Psi = \textbf{L}~ \mathrm{diag}(w) \Psi $, then \eqref{eqn:map2_discretization1} can be written as:
\begin{equation}
\label{FinalDisc}
\begin{split}
(w^*, \bar \Psi^*)& = \arg\min_{w, \bar \Psi}  \E(w, \bar \Psi ) = 
\frac{r_1}{2} \| F^T \textbf{M}^1 \Phi - G^T  \mathrm{diag}(w) \textbf{L}^T \bar\Psi \|_F^2 \\
& + \frac{r_2}{2} \tr (\bar\Psi^T \bar{\textbf{S}}^2(w) \bar\Psi)
 + \frac{r_3}{2} w^T \textbf{S}^2 w, \quad
\\
&  \text{s.t.} \quad \bar \Psi^T \bar \Psi = \textbf{I}_k 
\quad   \text{and} \quad
 w^T \textbf{M}^2 w = A 
\end{split} 
\end{equation}
where $ \bar{\textbf{S}}^2(w) = (\textbf{L}^T)^{-1}\mathrm{diag}(w)^{-1} \textbf{S}_2 \mathrm{diag}(w)^{-1} \textbf{L}^{-1}.$ Note that this parameterization of the problem moves the conformal factor $w$ out of the orthogonality constraint (and into $\bar{\textbf{S}}$). We will soon see that, for any fixed $\bar \Psi$, this will make the problem for $w$ easier to solve. 



\subsection{Numerical Optimization of LBBP Model}
\label{subsec:optimization}
The two variables $w$ and $\bar\Psi$ in \eqref{FinalDisc} make the optimization problem different from orthogonality constrained problems solved by nonconvex alternating direction method of multipliers (ADMM) methods considered in \cite{lai2014splitting,chenAugmented,wang2015global,kovnatsky2016madmm}. Rather than solve this problem directly for $\bar \Psi$ and $w$ simultaneously by directly minimizing \eqref{FinalDisc}, we employ a method based on the framework of proximal alternating minimization (PAM) method~\cite{attouch2010proximal}. 

Let $\S = \{\bar \Psi \in\RR^{n\times k} ~|~ \bar \Psi^T \bar \Psi = \textbf{I}_k \}$ and $\mathcal{W} = \{w\in\RR^n~|~w^T \textbf{M}_2 w = A \}$. We also define indicator functions:
\begin{equation}
\delta_\S(x) = \left\{\begin{array}{cc}0,& \text{if } x\in\S\\ +\infty, &\text{otherwise}\end{array}\right.,\qquad  \delta_\W(x) = \left\{\begin{array}{cc}0,& \text{if } x\in\W\\ +\infty, & \text{otherwise}\end{array}\right. 
\end{equation}
Then it is clear that $\delta_\S$ and $\delta_\W$ are semi-algebraic functions as $\S$ and $\W$ are zero sets of polynomial functions~\cite{attouch2013convergence}. Therefore, we write an equivalent form of \eqref{FinalDisc} as
\begin{equation}
\label{eqn:FinalDisc_PAM}
(w^*, \bar \Psi^*) = \arg\min_{w, \bar \Psi} \E(w, \bar \Psi ) + \delta_\S(\bar\Psi) + \delta_\W(w).
\end{equation}
Using the PAM method, we have the following iterative scheme
\begin{eqnarray} 
\label{optstep}
\left\{
\begin{aligned}
\bar \Psi^{j+1} &= \arg\min_{\bar \Psi}\E(w^j, \bar \Psi )  + \frac{1}{2\eta} ||\bar \Psi - \bar \Psi^j ||^2,  \quad \text{s.t.} \quad \bar \Psi^T \bar \Psi = \textbf{I}_k \\ 
w^{j+1} &= \arg\min_{w} \E(w, \bar \Psi^{j+1} )  + \frac{1}{2\eta} ||w - w^j ||^2,  \quad \text{s.t.} \quad w^T \textbf{M}_2 w = A 
\end{aligned}\right.
\end{eqnarray}
Here $\eta$ is a step size parameter. 
These proximal terms penalizes large step sizes in and prevents the algorithm from ``jumping" between multiple local minimums. The addition of these proximity terms allows us to analyze our proposed method in the framework of the PAM algorithm~\cite{attouch2010proximal}. It has been shown in \cite{attouch2010proximal, attouch2013convergence,bolte2014proximal} that such proximal terms can guarantee 
the solutions generated at each step converge to a critical point of the objective function. Formally, we have the following convergence theorem in accordance with Theorem 9 in~\cite{attouch2010proximal}.
\begin{theorem}\label{thm:convergence}
Let $\{w^j,\bar\Psi^j\}$ be the sequence  produced by \eqref{optstep}, then the following statements hold:
\begin{enumerate}
\item $\displaystyle \E(w^{j+1},\bar\Psi^{j+1}) + \frac{1}{2\eta} ||\bar \Psi^{j+1} - \bar \Psi^{j} ||^2 +  \frac{1}{2\eta} ||w^{j+1} - w^{j} ||^2 \leq \E(w^{j},\bar\Psi^{j}), ~\forall j\geq 0$.
\item $\displaystyle \sum_{j=1}^{\infty} (\|w^{j} - w^{j-1}\|^2 + \|\bar\Psi^{j} - \bar\Psi^{j-1}\|^2) < \infty$.
\item $\{w^j,\bar\Psi^j\}$ converges to a critical point of $\E(w,\bar\Psi)$.
\end{enumerate}
\end{theorem}

\begin{proof}
To prove this, we show that our model obeys the conditions required for local convergence of PAM in \cite{attouch2010proximal}. To do so, we need:

(1) Terms which contain only one primal variable are bounded below and lower semicontinous.

(2) Terms which contain both variables are $C^1$ and have a locally Lipschitz continuous gradients.

(3) The entire objective satisfies the Kurdyka-Lojasiewicz (KL) property. \\
It is immediately clear that that the first two properties are satisfied by our objective. Furthermore, it is known that all semi-algebraic functions have KL property \cite{attouch2010proximal,attouch2013convergence, chenAugmented}. Our objective is semi-algebraic so we can guarantee local convergence of the proposed optimization method. 
\end{proof}

We use the augmented Lagrangian method to solve the constrained sub-optimization problem for $w$ in \eqref{optstep}. For convenience, let's write
\begin{equation}\label{Lap}
\mathcal{L}(\bar \Psi,w;b) = 
\E(w, \bar \Psi )
+ \frac{r_4}{2} \Big(  w^T \textbf{M}_2 w - A   + b\Big )^2
\end{equation}
Overall, we solve \eqref{FinalDisc} in the following way by hybridizing PAM with the augmented Lagrangian method.
\begin{eqnarray} \label{optsteps}
\left\{
\begin{aligned}
 \bar \Psi^{j+1} &= \arg\min_{\bar \Psi} \mathcal{E} (w^j, \bar \Psi) + \frac{1}{2\eta} ||\bar \Psi - \bar \Psi^j ||^2  \quad \text{s.t.} \quad \bar \Psi^T \bar \Psi = \textbf{I}_k\\ 
w^{j+1} & \leftarrow  \begin{cases}
 \displaystyle w^{j+1,s+1} = \arg \min_{w}  \mathcal{L} (w, \bar \Psi^{j+1};b^{j+1,s}) + \frac{1}{2\eta} ||w - w^{j} ||^2  \\ 
 \displaystyle  b^{j+1,s+1} =  b^{j+1,s} +  (\w^{j+1,s+1})^T \textbf{M}_2 w^{j+1,s+1} -  A  .\\
\end{cases}
\end{aligned}
\right.
\end{eqnarray}

The subproblems for minimizing $\bar\Psi$ require a some special consideration. The main challenge this first sub-optimization problem is the nonconvex orthogonality constraints. Recently, several approaches have been developed to solve orthogonally constrained problems in feasible or infeasible ways~\cite{wen2013feasible,lai2014splitting,chenAugmented,wang2015global,kovnatsky2016madmm}. For our implementation, we have chosen the feasible approach developed in \cite{wen2013feasible} which uses a curvilinear method based on the Cayley transform together with Barzilai-Bowein step size line search. This method updates variables along a geodesic curve on the Stiefel manifold, a geometric description of the orthogonality. It preserves the orthogonality constraints and guarantees convergence to critical points in our scenario. More precisely, given a feasible starting point $\bar \Psi^s$  and the coordinate gradient $Y^s$ at this point, the update scheme is as follows:
\begin{eqnarray}
\label{eqn:curvilnear}
\left\{
\begin{aligned}
D^s &= Y^s (\bar \Psi^s)^T - \bar\Psi^s(Y^s)^T \\ 
Q^s &= (I+\frac{dt}{2}D^s)^{-1}(I-\frac{dt}{2} D^s) \\ 
\bar \Psi^{s+1} &= Q^s \bar \Psi^s 
\end{aligned}
\right.
\end{eqnarray}
Here $dt$ is a step size parameter chosen by the Barzilai-Bowein criteria developed in \cite{barzilai1988two}.  Although convergence to a global minimum is not guaranteed, this method has proven effective for our purposes and only requires the computation of the objective function and its coordinate gradient $Y^s$ with respect to $\bar\Psi$ at each step provided by:
%
\begin{eqnarray} \label{grad Epsi}
\begin{aligned}
\nabla_{\bar\Psi} &\left(\E(w,\bar\Psi) +\frac{1}{2\eta} ||\bar \Psi - \bar \Psi^j ||^2 \right) = 
\\& - r_1 G^T \mathrm{diag}(w) \textbf{L}^T \Big(F^T \textbf{M}_1\Phi - G^T \mathrm{diag}(w) \textbf{L}^T\bar \Psi\Big)\\
& + r_2 \bar{\textbf{S}}^2 \bar \Psi + \frac{1}{\eta}(\Psi - \bar\Psi^j)
\end{aligned}
\end{eqnarray}

The subproblem for $w$ (as written in \eqref{optsteps}), on the other hand is smooth and unconstrained.
For our implementation, we use the well known quasi-Newton BFGS algorithm \cite{bazaraa2013nonlinear}. The gradient of objective function with respect to $w$ can be written as:
\begin{eqnarray} \label{grad Ew}
\begin{aligned}
\nabla_{w}\left(\mathcal{L}(w,\bar\Psi;b) + \frac{1}{2\eta} ||w - w^{j} ||^2 \right) &= r_1 \ diag \Big(G^T(F^T\textbf{M}_1 \Phi-G w \textbf{L}^T \bar\Psi)) \bar\Psi^T \textbf{L} \Big) \\
& + r_2 \ diag \Big(\Psi \Psi^T \textbf{S} w^{-1} \Big)\odot w^{-2} \\
 & + r_3 \textbf{S}_2 w 
 + r_4 \Big( w^T \textbf{M}_2 w -  A + b \Big)  \textbf{M}_2 w + \frac{1}{\eta}(w - w^{j})
\end{aligned}
\end{eqnarray}
where $diag \big(\cdot \big)$ denotes the diagonal of the matrix, $\odot$ signifies element-wise Hadamard product and $w^{-2}$ is the inverse of diagonal matrix $w$ multiplied with itself. 


\subsection{Computation of Point-to-Point Map}

One naive way to compute a point-to-point map is to find the functional map by using the final deformed manifold and its LB eigensystem with respect to the deforamtion. However, this may not work well because of the ambiguity of LB eigensystem. Additional effort is needed to handle possible ambiguity of LB eigensystem such as the method discussed in~\cite{lai2017multi}. As an advantage of the proposed method, the resulting basis generated by the proposed algorithm (recovered as $\Psi^* = A^{-1} w \bar \Psi$) to will naturally correct ambiguities of LB eigensystem. This is similar to the method discussed in~\cite{kovnatsky2013coupled}.  Thus, we can compute the functional map as $F_T(h) = \sum_{i=1}^k  (\int_{\M_1} h \phi_i \ \mathrm{d}\M_1)  \psi^T_i = \Psi \Phi^T \textbf{M}_1h$. However, this method is still quite inefficient and may be sensitive to small errors in the resulting basis. 

Instead, after we recover the final basis from our method, we can compute the point-to-point map between the two surfaces by comparing the values of each of the basis functions. This is essentially the same scheme presented in \cite{ovsjanikov2012functional}, but applied to our new basis. We use a KNN search (with $K = 1$) to match rows of $\Phi$ and $\Psi$. This requires a search of $n$ points in $k$ dimension, but is much more efficient and accurate than using the delta function approach described in the previous paragraph. Other methods used to refine functional maps such as \cite{rodola2015point} can be applied in this setting without changes. We summarize our numerical method for nonisometric surface registration as Algorithm \ref{alg:LBBasisPursuit1}.

\begin{algorithm2e}[ht]
\lrpicaption{LB Basis Pursuit (LBBP) Algorithm.}
\label{alg:LBBasisPursuit1}
\SetKwInOut{Input}{input}\SetKwInOut{Output}{output}
\SetKwComment{Comment}{}{}
\KwIn{Triangulated surfaces $\M_1$ and $\M_2$ and list of known corresponding functions $F$ and $G$.}
\KwOut{$\Psi^*$, $w$, point-to-point map}
Compute stiffness and mass matrices for each surface: $\textbf{M}_1, \textbf{M}_2, \textbf{S}_1, \textbf{S}_2$\;
Use stiffness and mass to calculate LBO eigensystems: $\textbf{M}_1 \Phi = \lambda \textbf{S}_1 \Phi$\;
Initialize: Let $\Psi^0$ be the LB eigenfunctions of target surface: $\textbf{M}_2 \Psi = \lambda \textbf{S}_2 \Psi$\;
Compute $\bar\Psi^0 = \textbf{L} w \Psi$\;
\While{not converged}{
Update $\displaystyle \bar\Psi^{j+1}= \arg\min_{\bar\Psi}\mathcal{E} (w^j, \bar \Psi) + \frac{1}{2\eta} ||\bar \Psi - \bar \Psi^j ||^2$ using the curvilinear search algorithm \eqref{eqn:curvilnear}\;
\While{ $s \leq \ell$ }{
Update $\displaystyle w^{j+1,s} = \arg \min_{w} \mathcal{L} (w, \bar \Psi^{j+1};b^{j+1,s}) + \frac{1}{2\eta} ||w - w^{j} ||^2$ using BFGS\;
$\displaystyle b^{j+1,s+1} =  b^{j+1,s} + (w^{j+1})^T \textbf{M}_2  w^{j+1} - A$\;
}
$w^{j+1} = w^{j+1,s}$\;
 }
Recover $\Psi^* = w \textbf{L}^{-1} \bar \Psi $\;
Compute correspondence map with KNN-search of coefficient space 
\end{algorithm2e}


\section{Discussion}
\label{sec:Discussion}
In this section, we discus our choice of feature functions, as well as ways to overcome problems which may arise from the non-convexity of the proposed optimization problem. In addition, we present a novel way to jointly measure the quality of the correspondence and alignment of the bases without any prior knowledge about the ground truth of the point-to-point map.   


\subsection{Choice of Feature Functions}
The simplest, and in many applications, most natural features to choose for $F$ and $G$ are indicator functions for known landmarks. Let  $\{\chi^{1}_i\}_{i=1}^k$ be a set of points on $\M_1$ and $\{\chi^{2}_i \}_{i=1}^k$  be a corresponding set on  $\M_2$. We can view each $f_i$ and $g_i$ as a $\delta$-function on $\M_1$ and $\M_2$ respectively to indicate these landmarks.

Another option is to use heat diffusion functions. Given a corresponding pair of points we can use delta functions to define an initial condition and solve the heat diffusion problem $\frac{\partial u}{\partial t}(\x) = \Delta u(\x,t)$ using the Crank-Nicholson scheme 
$ \Big( \textbf{M}+\frac{dt}{2} \textbf{S} \Big)u^{i+1} = \Big( \textbf{M} -\frac{dt}{2}\textbf{S} \Big) u^{i}$
where $dt$ is a step size parameter.
By taking ``snap shots'' (solutions of the equation for various $t$ values) of $u$ at different time values we can generate multiple functions from a single corresponding pair. This choice allows for a multi-scale selection of features and often results in better correspondences, but is computationally more expensive. Also, since the heat diffusion is sensitive to local geometry, it is often necessary to recompute the diffusion with respect to the conformal factor. This can be included as a step in the reinitialization scheme which will be discussed in the next section. 

The wave kernel signature (WKS) has also been used for characterizing points on non-rigid three dimensional shapes \cite{aubry2011wave}. These functions are defined as the solutions to the Schrodinger equation: $\frac{\partial u}{\partial t}(\x) = i \Delta u(\x,t)$  at different points on the surface. Given two corresponding points we can solve the equation at each point and use these as our corresponding functions. However, the solutions to these equations are highly dependent on both local and global geometries of the manifold. Because of this, they are only suitable for shape correspondence when the shapes are very similar and, in general, do not work well for non-nearly-isometric problems. The same problem exists for heat diffusion features, however, in general heat diffusion tends to be much more stable with respect to local deformations.%

SHOT features \cite{tombari2010unique} are also a popular choice of feature functions for shape processing tasks. For nearly isometric shapes these descriptors work well, but since they are not intrinsically defined they do not work well with the re-initialization scheme detailed in the next section. Updating these features with respect to a conformal deformation requires computing the deformed embedding, which the rest of our method explicitly avoids.
 

\subsection{Reinitialization Schemes}\label{sec:ReInit}
Although we have shown that the proposed PAM based optimization algorithm converges to a critical point of the objective function, it is still challenging to achieve a global optimum as the problem is non-convex. In practice, we have found that the numerical results can often be improved in terms of both accuracy and speed of computation by adding a simple reinitialization scheme to our algorithm. The motivation for the scheme comes from an observation that if we know the exact conformal deformation $w^2$ and the source surface has a simple eigensystem (no repeated eigenvalues), then the LB eigensystem of $(\M_1, g_1)$ is the same as the LB eigensystem of $(\M_2,w^2 g_2)$ up to a change in sign. 
With this in mind, we propose to reinitialize the $\Psi$ problem by resetting $\Psi$ to be the solution to weighed eigenproblem $\textbf{S}_2 \Psi = \Lambda  \mathrm{diag}(w^2) \textbf{M}_2 \Psi$ . We remark that this reinitialization method to achieve an optimizer closer to the global one is empirical, although it is based on the geometric intuition. 

Computationally, to avoid introducing ambiguities of LB eigensystem by calling a standard eigen-solvers, we solve a discrete counterpart to \eqref{defeig} as
$\displaystyle \min_{\Psi} ~ \tr (\bar\Psi^T \bar{\textbf{S}}^2(w) \bar\Psi) , 
~ \text{s.t.} ~ \bar \Psi^T  \bar \Psi = \textbf{I}$
based on the curvilinear search method discussed in Section \ref{subsec:optimization} and using the current eigensystem, $\bar \Psi^{j+1}$, as an initial guess for this problem. By using $\bar \Psi^{j+1}$ as warm start for the eigenproblem we can avoid re-introducing sign or multiplicity ambiguities into the problem which our algorithm has already resolved. 

When using heat diffusion, wavelet kernel signatures, or any other functions which are defined based on local geometry as the  input feature functions, then we also need to recalculate these functions with respect to the conformally deformed metric. For example, if we are using heat diffusions, we can recompute the heat diffusion functions on the deformed manifold $(\M_2,w^2 g_{2})$ by multiplying the mass matrix by $w^2$ in the Crank-Nicholson scheme:
$ \Big(\textbf{M}_2 \mathrm{diag}(w^2)+\frac{dt}{2}\textbf{S}_2 \Big)u^{i+1}
 = \Big(\textbf{M}_2 \mathrm{diag}(w^2)-\frac{dt}{2}\textbf{S}_2\Big)u^{i}$,  
A similar re-computation technique can be applied to wave kernel signatures, or any other features which are computed using finite element-like operators.


\subsection{Sub-Sampling Scheme}\label{sec:subscheme}

The most computationally demanding step of our algorithm the update of $\bar \Psi$ and as a result the time complexity of our algorithm depends on the number of points in the discretization of $\M_2$. However, the overall geometry of the shape can often be closely estimated by a relatively small subset of the points contained in a triangulated mesh or point cloud. Inspired by this observation, we propose a warm start method in which we solve a smaller problem on a subset of the full mesh and use it as a warm start for the full problem. One way to do this would be to sub-sample the mesh and compute a new (local) triangulation \cite{lai2013local}. However, the re-meshing process can be computationally expensive. Therefore, we instead seek a method to approximate $\bar Psi$ on the entire mesh, using only the sub-sampled points.

Given a mesh $M$ with $n$ points, we first compute a sub-sample of points $\bar M$ with $\bar n < n$ points which most articulately represents the original mesh. To do so, we begin with a random seed point and compute the point on the mesh which has the greatest (geodesic) distance from it and include this point in $\bar M$. Then we iteratively add points to $\bar M$ by finding the point on $M$ which has the greatest minimal distance to any point already included in $\bar M$.

To approximate a function $f$ defined on $\M$ with only $\bar n$ variables we define linear projection and reconstruction operations to down-sample the problem. One naive idea would be to restrict the values of $f$ to $\bar M$ and use linear interpolation in the other direction. However, this fails to capture many of the details of functions in the projection step, and doesn't respect the local geometry in the reconstruction step. Instead, we define a new approximate basis with elements, $u_{i,t}(x), i \in \bar M, x \in M$, created by diffusing a delta function, centered at $i$ for a fixed time $t$. The resulting basis contains $\bar n$ elements. We define a projection operation and reconstruction operations as 
\begin{eqnarray}
\left\{
\begin{aligned}
Proj(f)&:= (u^T \textbf{M} u)^{-1} u^T \textbf{M} f = \bar f  \\
Recon(\bar f) &:= \bar f u
\end{aligned}
\right.
\end{eqnarray}
We can then use this new approximate basis we to the dimension of the optimization problem and solve the simplified problem very quickly. We define $\bar \Psi_u$ and $\w_u$ to be the projection of $\Psi$ and $w$ onto the $u_{i}$ set which can be represented as the coefficients $C_{i} = \langle \bar \Psi, u_{i} \rangle$ and $D_{i} = \langle w, u_{ij} \rangle$. Plugging these into our model we get:
\begin{equation}
\label{eqn:submodel}
\begin{split}
(D^*, C^*)& = \arg\min_{D, C}  \E(D, C) = 
\frac{r_1}{2} \| F^T \textbf{M}_{u2} D - G^T D \textbf{L}_u^T \bar\Psi \|_F^2  \\
 & \quad + \frac{r_2}{2} \tr (C \bar{\textbf{S}_{u2}}^2(w) C)
 + \frac{r_3}{2} w^T \textbf{S}_{u2} D, \quad
\\
& \text{s.t.} \quad C^T u^T u C = \textbf{I}_n
\quad   \text{and} \quad
 w^T \textbf{M}_{u2} w = A 
\end{split} 
\end{equation}
Where $\textbf{M}_{u2} = u^T M u$, $\textbf{L}_u = \textbf{L} u$ and $\textbf{S}_{u2} - u^T \textbf{S} u$ can all be precomputed. Note that if $\{u_i\}_{i=1}^{\bar n}$ is, in fact, a tight frame then \eqref{eqn:submodel} is the same as \eqref{eqn:finalmodel}. This problem can be solved with algorithm \eqref{alg:LBBasisPursuit1}, but has significantly fewer variables then \eqref{eqn:finalmodel}. By using using the elongation of the solution to \eqref{eqn:submodel} as an initial guess for $\bar \Psi$ and $w$ we can significantly decrease the time needed to solve the full model. 

With this warm start (\ref{alg:sub}) and the re-initialization procedure described in Section \ref{sec:ReInit}, we propose a modified version of our numerical solver as Algorithm~\ref{alg:LBBasisPursuit2}. 

\begin{algorithm2e}
\lrpicaption{Sub-sampling and Warm Start Algorithm.}
\label{alg:sub}
\KwIn{Set of vertices and faces of source ($\M_1$) and target ($\M_2$) manifolds, number of subsample points $\bar n$, list of known corresponding functions $F$ and $G$, Stiffness and Mass Matrices $S_1, S_2, M_1, M_2$ }
\KwOut{$\Psi^*$, $w^*$, }
Initialize: Let $\Psi^0$ be the LBO eigenfunctions of target surface: $\textbf{M}_2 \Psi = \lambda \textbf{S}_2\Psi$\;
Compute down-sampled points to represents $\M_1$\;
Compute down-sampled bases and representation of $F$\;
Use Algorithm \ref{alg:LBBasisPursuit1} to solve \eqref{eqn:submodel} for $D^*,C^*$\;
Compute $\bar{\Psi} = \sum_{i=1}^{\bar{n}} C_i \u_i$  and $w = \sum_{i=1}^{\bar{n}} D_i \u_i$ \;
\end{algorithm2e}

\begin{algorithm2e}[ht]
\lrpicaption{LB Basis Pursuit Algorithm with warm start and reinitialization.}
\label{alg:LBBasisPursuit2}
\KwIn{Set of vertices and faces of source ($\M_1$) and target ($\M_2$) manifolds and list of known corresponding functions $F$ and $G$}
\KwOut{$\Psi^*$, $w^*$, point-to-point correspondence map}
Compute stiffness and mass matrices for each surface: $\textbf{M}_1, \textbf{M}_2, \textbf{S}_1, \textbf{S}_2$\;
Use stiffness and mass to calculate LBO eigensystems: $\textbf{M}_1 \Phi = \lambda \textbf{S}_1 \Phi$\;
Compute corresponding feature functions $F$ and $G$ on $\M^1$ and $\M^2$ respectively\;
Initialize: Let $\Psi^0$ be the LBO eigenfunctions of target surface: $\textbf{M}_2 \Psi = \lambda \textbf{S}_2\Psi$\;
Compute down sampled bases through downsample and heat diffusion\;
Compute $\bar \Psi^0$ and $\bar w^0$ through warm start through Algorithm \ref{alg:sub}\;
\While{number of re-initialization steps complete $<$ max number of re-initializations}{
Update $\displaystyle \bar\Psi^{j+1}= \arg\min_{\bar\Psi}\mathcal{E} (w^j, \bar \Psi) + \frac{1}{2\eta} ||\bar \Psi - \bar \Psi^j ||^2$ using the curvilinear search algorithm \eqref{eqn:curvilnear}\;
\While{ $s \leq \ell$ }{
Update $\displaystyle w^{j+1,s} = \arg \min_{w} \mathcal{L} (w, \bar \Psi^{j+1};b^{j+1,s}) + \frac{1}{2\eta} ||w - w^{j} ||^2$ using BFGS\;
$\displaystyle b^{j+1,s+1} =  b^{j+1,s} +  (w^{j+1})^T\textbf{M}_2 w^{j+1}  - A$\;
}
$w^{j+1} = w^{j+1,l}$\;

	\If {update $<$ tolerance}{
		Re-Initialize $\bar \Psi$ as $\displaystyle \argmin_{\bar\Psi} ~ \tr (\bar\Psi^T \bar{\textbf{S}}^2(w^{j+1}) \bar\Psi) , ~ \text{s.t.} ~ \bar \Psi^T \textbf{M}_2 \bar \Psi = \textbf{I}$\;
		\If {Using feature functions which depend on local geometry}
		{Re-Compute features using $\textbf{M}_2 \mathrm{diag}(w^2)$ as Mass matrix}
	}
}
Compute correspondence map with KNN-search of coefficient space;
\end{algorithm2e}

\section{Numerical Experiments}
\label{sec:Results}
In this section, we apply our algorithm to several problems. We begin by working on a typical non-isomorphic matching problem for a pair of shapes with a large deformation: a horse and an elephant. We perform tests showing the effectiveness of our approach given different amounts of landmark points, and demonstrate robustness with respect to noise both on the manifold and in the initial correspondences. Finally, we conduct experiments on the Faust benchmark data set \cite{bogo2014faust}. All numerical experiments are implemented in MATLAB on a PC with a 32GB RAM and two 2.6GHz CPUs.

In all of our experiments, we use randomly chosen correspondence points to create indicator functions as the input features.  The first 100 non-trivial LB eigenfunctions are chosen to calculate the coefficient matching term, as well as for computing the final correspondence. We set $r_1=10,\ r_2=10,\ r_3=1,\ r_4=.01, \ell = 1$ for all experiments, even though the data sets and experimental conditions are very different. This choice of $r_1$ and $r_2$ allows the coefficient matching terms and eigenfunction term to balance each other out, with the choice of $r_3$ still being large enough to preserve the area constraint. $r_4$ is chosen to be small so that the harmonic energy, which tends to be quite large, does not dominate the others. In general, we have observed that our algorithm is quite robust to different choices of parameters. 

\subsection{A Large Deformation Pair: Horse to Elephant}
\begin{figure}[ht]
\begin{minipage}{0.49\linewidth}
\begin{center}
\includegraphics[width=.85\linewidth]{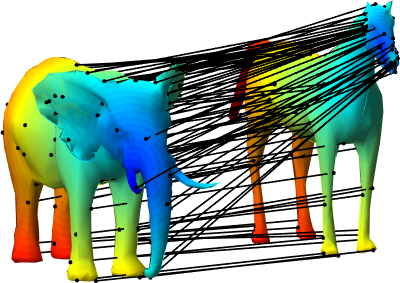}
\end{center}
\end{minipage}
\begin{minipage}{0.49\linewidth}
\begin{center}
\includegraphics[width=.95\linewidth]{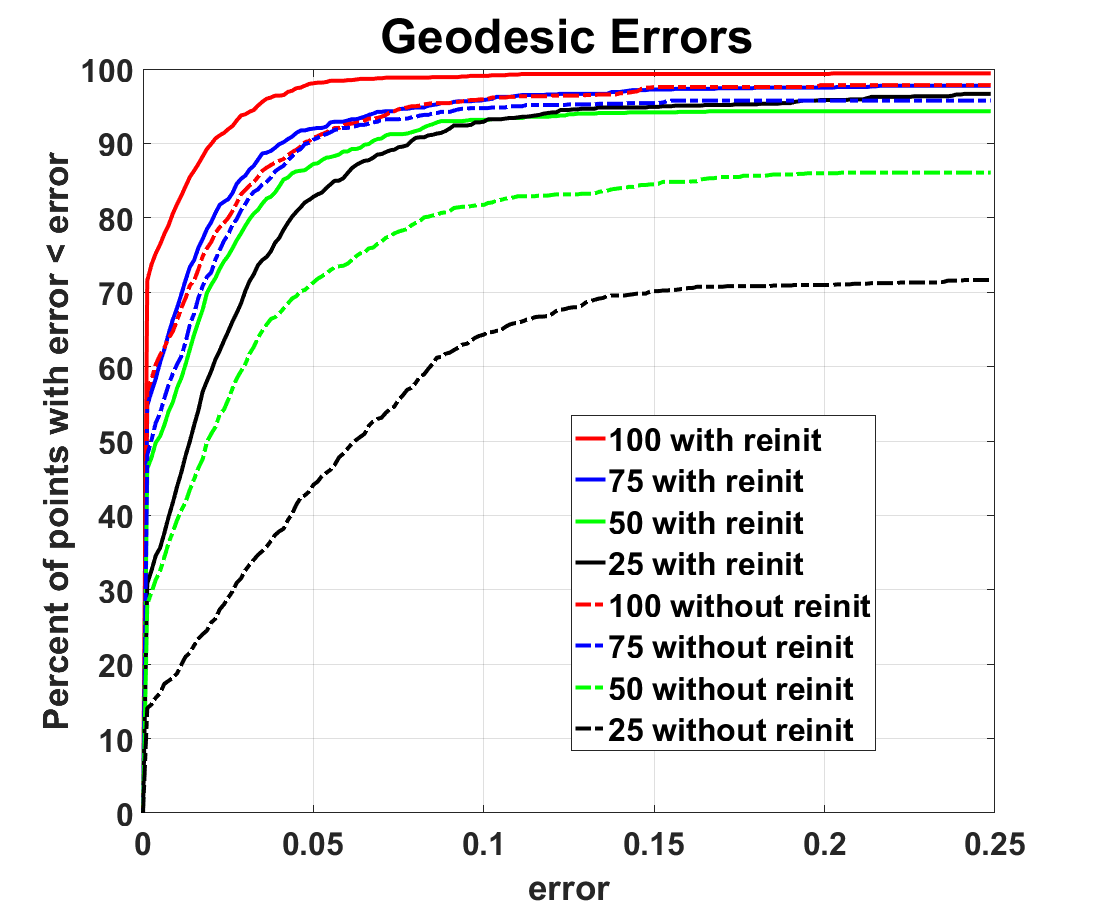}
\end{center}
\end{minipage}
\lrpicaption{Left: visualization of point-to-point map and texture transfer. Right: normalized geodesic errors for various numbers of randomly selected landmarks with and without reinitialization.}
\label{P2Ppic}
\end{figure}
The first experiment is designed to test the effectiveness of the proposed method on a pair of shapes with large deformation. Each surface, a horse and an elephant, is represented by a mesh with 1200 points. One of the challenges in this pair is the large deformations in the sharp corner and elongated regions including ears, teeth, noses and tails on the horse and elephant surfaces. Those regions make the registration problem very difficult.To demonstrate the efficacy of our approach, we perform this experiment under several different conditions. Our algorithm produces excellent results given a sufficient number of landmarks, and it still finds reliable correspondences given limited landmarks. We also show that using our reinitialization scheme (Algorithm \ref{alg:LBBasisPursuit2})  produces a more accurate map than without this extra step (Algorithm \ref{alg:LBBasisPursuit1}). 

\begin{figure}[ht]
\begin{minipage}{.48\linewidth}
\centering
\includegraphics[width=1\linewidth]{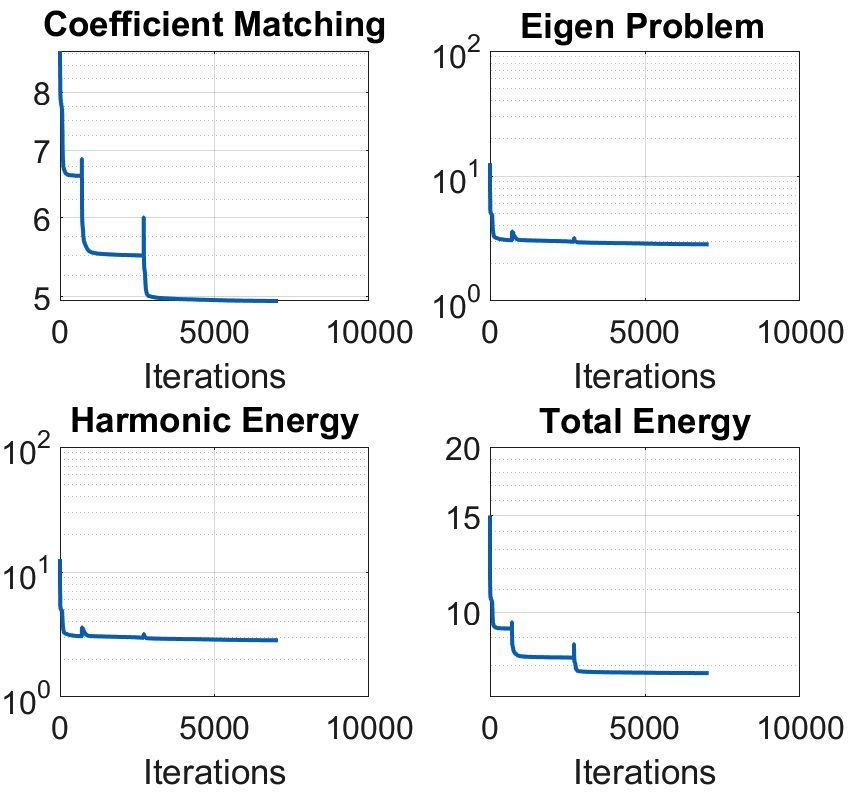}
\end{minipage}
\begin{minipage}{.48\linewidth}
\centering
\includegraphics[width=1\linewidth]{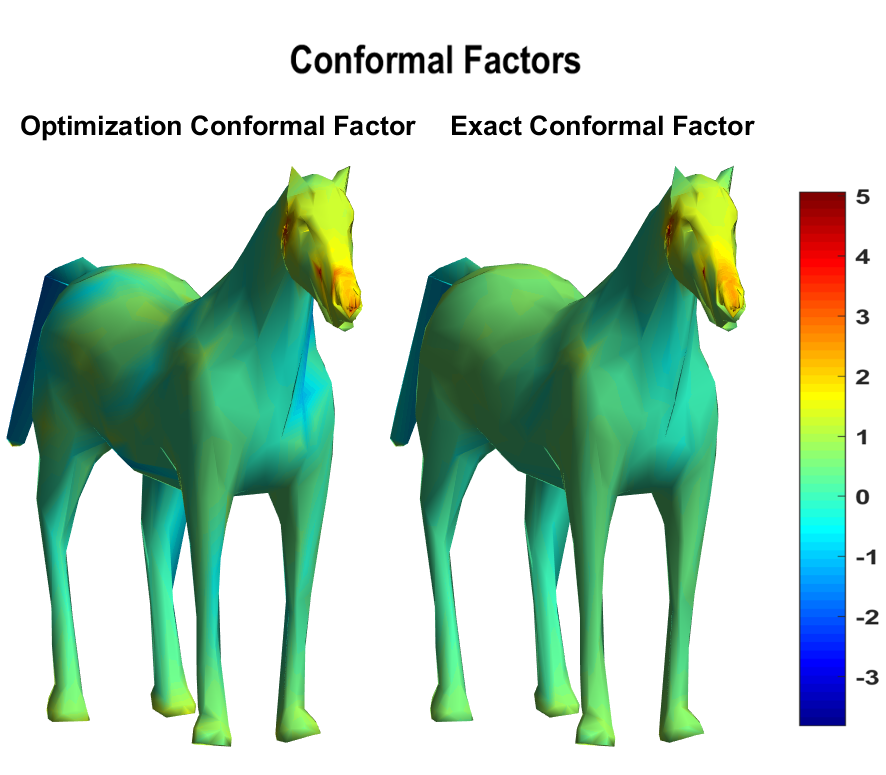}
\end{minipage}
\lrpicaption{Left: convergence curves of our method. The coefficient matching term measures: $\| F^T \textbf{M}_1 \Phi - G^T  \textrm{diag}(w) \textbf{L}^T \bar\Psi \|_F$. The eigen problem is: $(\Psi^T \textbf{S}_2 \Psi)$ and the harmonic energy measures: $w^T \textbf{S}_2 w$ and the total energy is the entire model derived in \eqref{FinalDisc}.  Right: resulting and exact conformal factors.}
\label{convergence}
\end{figure}
Figure \ref{P2Ppic} shows the results of using 100, 75, 50 and 25 known landmark points with and without our reinitialization scheme.  To qualitatively measure the mapping quality, we calculate the normalized geodesic distance from the point on the target surface produced by the map to ground truth following the Princeton Benchmark method \cite{kim2011blended}. These distances are collected into a cumulative error on the right of Figure \ref{P2Ppic} where the $y$-axis measures the percent of points whose distances are less than or equal to the $x$-axis value. For example, in the case of $100$ known landmarks, our algorithm matches over $70\%$ of the points to exact correct point and more than $95\%$ within a $5\%$ error margin. 
\begin{figure}[ht]
\begin{center}
\includegraphics[width=.9\linewidth]{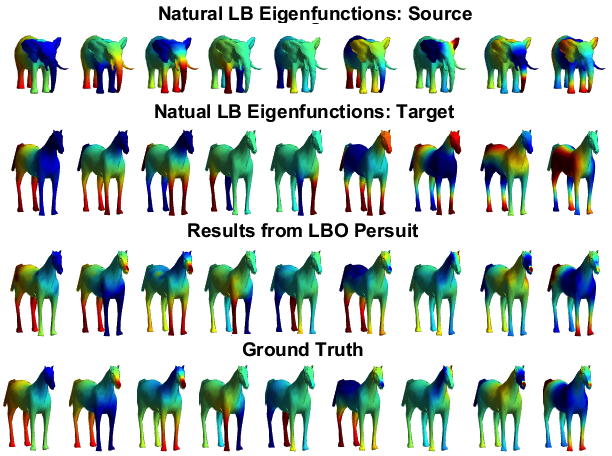}
\end{center}
\lrpicaption{First two rows: The first 9 non-trivial natural LB eigenfunctions of manifolds.
Third row:  results from the proposed basis pursuit algorithm. 
Fourth row: ground truth.}
\label{H2Eeigs}
\end{figure}

\begin{figure}
\centering
\includegraphics[width=.6\linewidth]{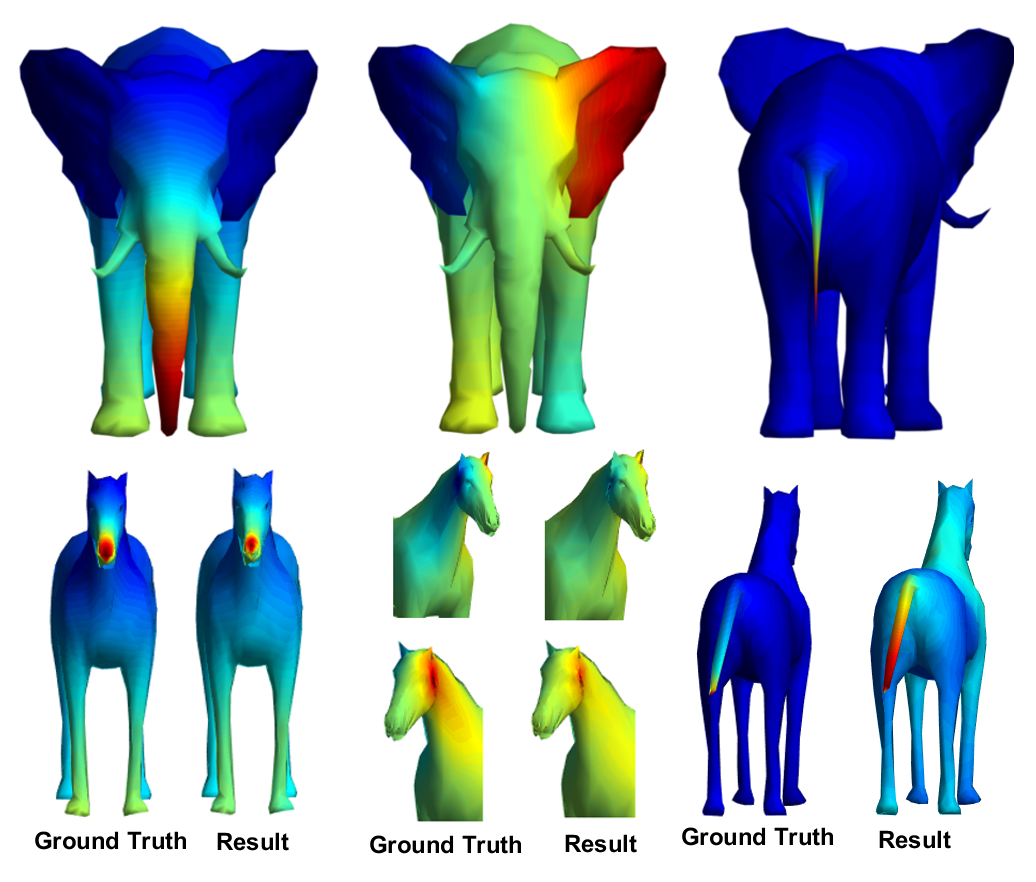}
\lrpicaption{Top row: 9th, 11th and 44th natural LB eigenfunctions on source. Bottom row: results and ground truth.}
\label{closeups}
\end{figure}

\begin{figure}
    \begin{center}
    \includegraphics[width=.6\linewidth]{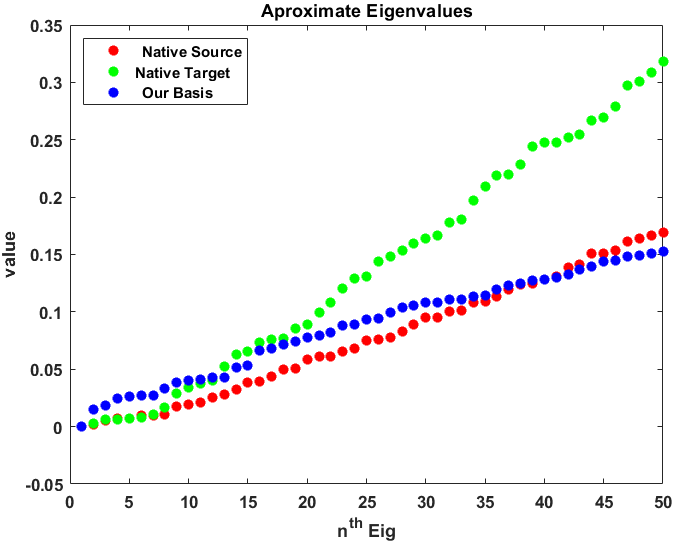}
    \lrpicaption{Alignment of the eigenvalues. Green: native basis, Red: target basis, Blue: deformed basis.}
    \label{isospec}
    \end{center}
\end{figure}

\begin{figure}
    \begin{center}
    \includegraphics[width=.65\linewidth]{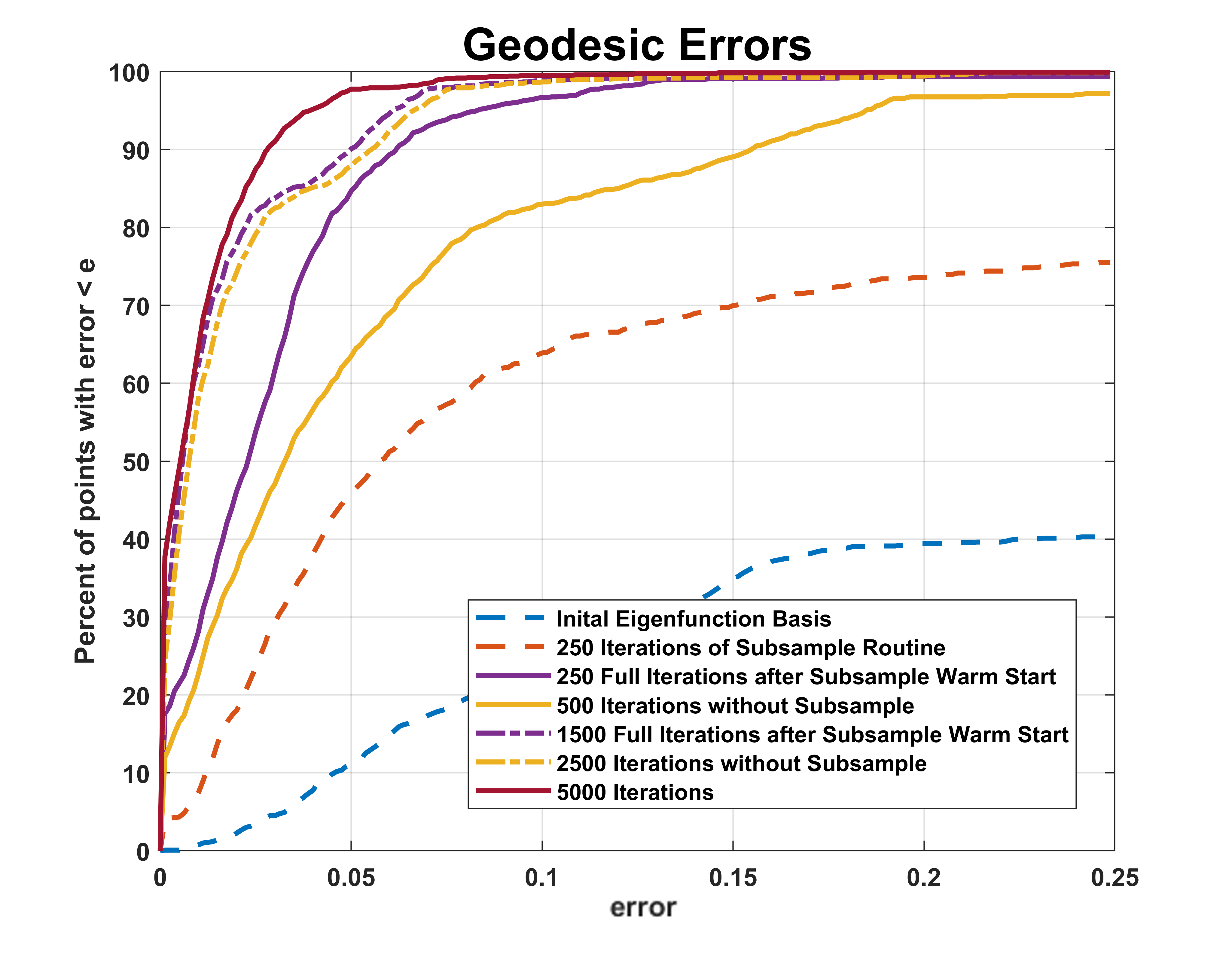}
    \lrpicaption{Quality of correspondences, with and without warm start.}
    \label{fig:suberrors}
    \end{center}
\end{figure}

The left panel of Figure \ref{convergence} shows the convergence of the objective function and illustrates the effectiveness of the reinitialization step. We plot the three terms in the objective function separately as well as the overall objective. We typically observe that the convergence curves in the coefficient matching and total energy flatten quickly as the algorithm tends to a local minimizer. However, each reinitialization
significantly reduces the objective function. We further demonstrate the validity of our algorithm by examining the resulting conformal factor. In the right image of Figure \ref{convergence}, we show the conformal factor calculated by our algorithm as well as the ground truth. The ground truth conformal factor is calculated by using the ground truth point-to-point map to compare the area of the first ring structure around each point on the source and target surface.  Here we plot $u$ where $w^2=e^{2u}$ for better visualization. From this figure we can confirm that the conformal mapping our algorithm produces is very close to the true factor.

Since the elephant and horse are dramatically different shapes, the large dissimilarity of their natural LB eigenfunctions (first two rows of Figure \ref{H2Eeigs}) cannot be expected to produce meaningful correspondence. However, our model overcomes this by capturing the conformal deformation between the surfaces. As a result, the basis computed for the horse (target surface) by our model is consistent with the LB eigenfunctions of the elephant (source surface). We further compare these results with the ground truth which is calculated through the push forward of the LB eigenfunction of the source to the target surface using the \textit{a priori} map.  Figure \ref{closeups}  highlights the consistency of the produced bases on several highly distorted regions. Specifically, we focus on each of the ears, the nose/trunk and the tails. From Figures \ref{H2Eeigs} and \ref{closeups}, we can see that our approach produces a new basis on the target that aligns very closely to the natural LB basis on the sources manifold. Figure \ref{isospec} shows that the eigenvalues of the deformed eigenesytem are much closer to the eigenvalues of the source surface than they are to the target. Although these values are never explicitly taken into account in our numerical algorithm, it is not surprising that aligning the eigenfunctions also aligns their eigenvalues. This close alignment of the eigensystems is the reason that accurate registration results can be obtained using the new basis.

%
%
%

%
\subsection{Sub-Sampling}

\begin{table}
\lrpicaption{Wall-clock time of sub-sampling schemes.}
\begin{tabular}{|c|c|c|c|c|}
\hline
Iterations: & Warm Start (250) & $250$ & $500$  & $1500$  \\ \hline
Time & $322s$       & $2047s$          & $4521s$          & $1485s$ \\ \hline
\end{tabular}
\end{table}

To illustrate the effectiveness of the sub-sampling scheme presented in section \eqref{fig:suberrors}, we repeat the previous experiment twice more, both with and without the sub-sampling warm start, and manually stop the algorithm after 500 iterations. Figure 5 shows the quality of the correspondences produced by the initial basis, the one produced by the subsampling scheme after 250 iterations, the basis produced by algorithms after 250 full iterations using the sub-sampled scheme as a warm start, one produced by the algorithm using 500 iterations of the full scheme without using the warm start and finally results after 1500 and 2500 iterations with and without the warm start. From this figure we observe that the warm start routine can significantly speed up the basis pursuit by providing a good initialization to the full algorithm.
\subsection{Necessity of Conformal Deformation}
\begin{figure}
\begin{center}
\includegraphics[scale=.85]{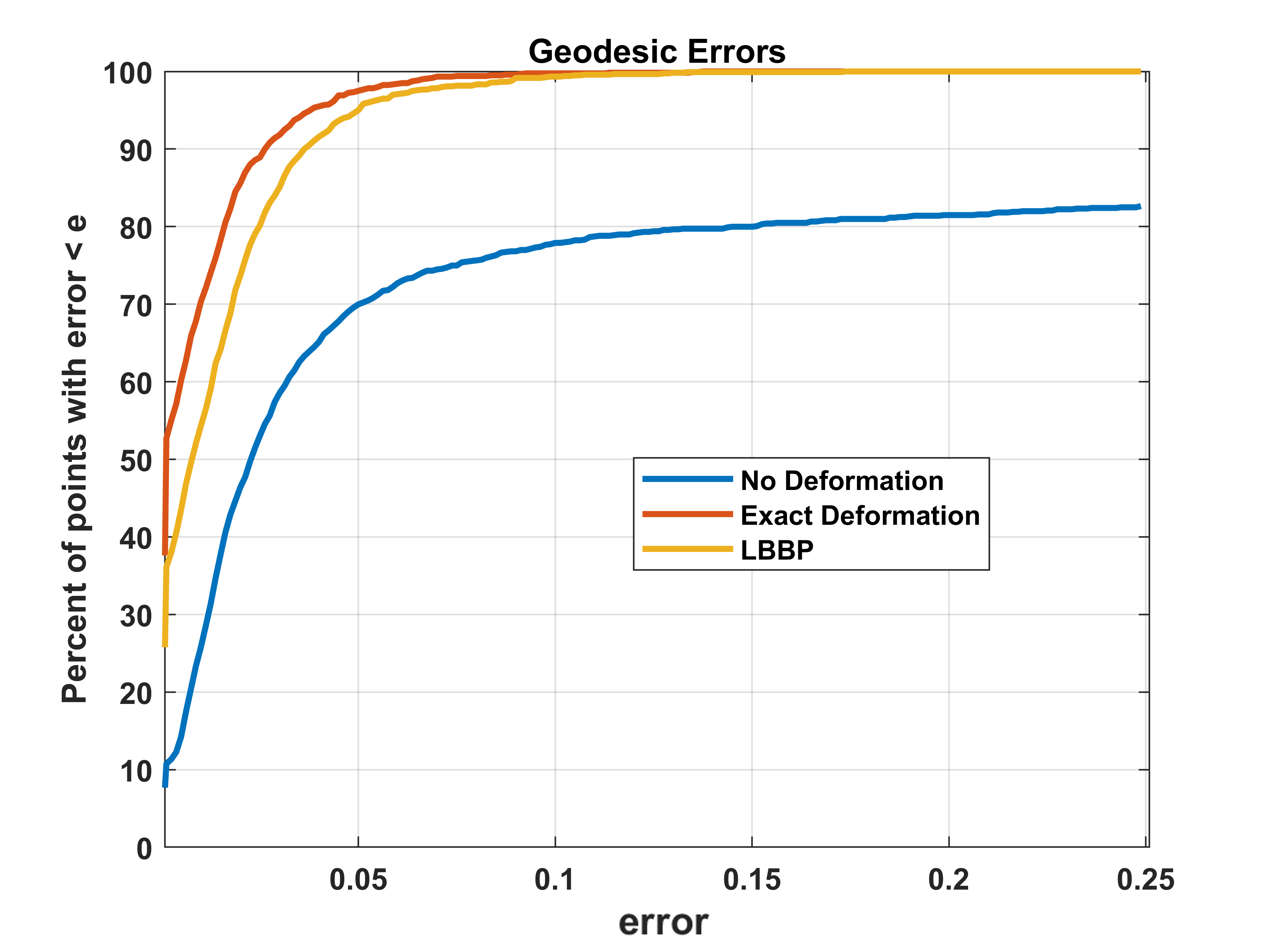}
\lrpicaption{Comparisons of results obtained from basis pursuit without deformation, with oracle deformation and our LBBP method.}
\label{fig:defcomp}
\end{center}
\end{figure}
To show the importance of understanding the deformation between surfaces when using a spectral based method, we run two tests for finding correspondence between horse and elephant using LB basis pursuit algorithm but freezing the conformal deformation. We first set the conformal factor to be $1$ everywhere. This mean no deformation is imposed in the procedure of the LB basis pursuit. Next, we use the exact deformation, which can be computed as \textit{a priori} using the exact correspondence. Figure \ref{fig:defcomp} shows the geodesic errors of the correspondence produced by the optimized bases when using each of these fixed conformal factors, as well as the result of our algorithm referred as LBBP. Although, our algorithm does not achieve the same performance as using the oracle deformation (which is not obtainable in practice), we vastly outperform the non-deformation case.

\subsection{Robustness of Conformal Laplace-Beltrami Basis Pursuit}
\paragraph{Noisy data.}In this experiment, we demonstrate that our algorithm can handle noisy data. Since noise on the surfaces can be viewed as local deformations, our algorithm is automatically robust to geometric noise. Medical scans often have noise resulting from the imaging instruments and manual segmentation. Our model can solve registration problems for this type of data. To demonstrate this, we generate noisy data by adding noise along the normal of each point. The top panel in Figure \ref{fig:robustness} shows the results of two experiments: a noisy elephant to an elephant and a noisy horse to an elephant. We observe that our algorithm still produces very accurate results despite this noise.

\paragraph{Landmark perturbation.}We also demonstrate the robustness of our algorithm to landmark perturbations. Working again on the horse and elephant, we test cases where the landmarks are perturbed to another vertex within the first ring. The magnitude of these perturbations depends on the uniformity and meshing of the surface. The bottom left graph in Figure \ref{fig:robustness} shows the size of the perturbations of the landmarks points as well as the error in their final mapping. The bottom right graph in Figure \ref{fig:robustness} compares the geodesic error of the for all points when 25\%, 50\% and 100\% of the landmarks points are perturbed. From these tests we conclude that our method can successfully reduce the error introduced in the perturbed landmarks and still produce accurate maps in the presence of perturbations.
\begin{figure}[ht]
\begin{minipage}{.45\linewidth}
\includegraphics[width=.95\linewidth]{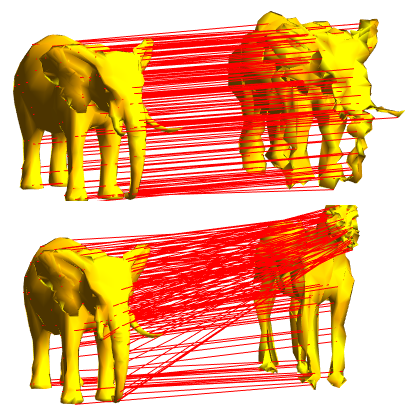}
\end{minipage}\hfill
\begin{minipage}{.45\linewidth}
\includegraphics[width=.95\linewidth]{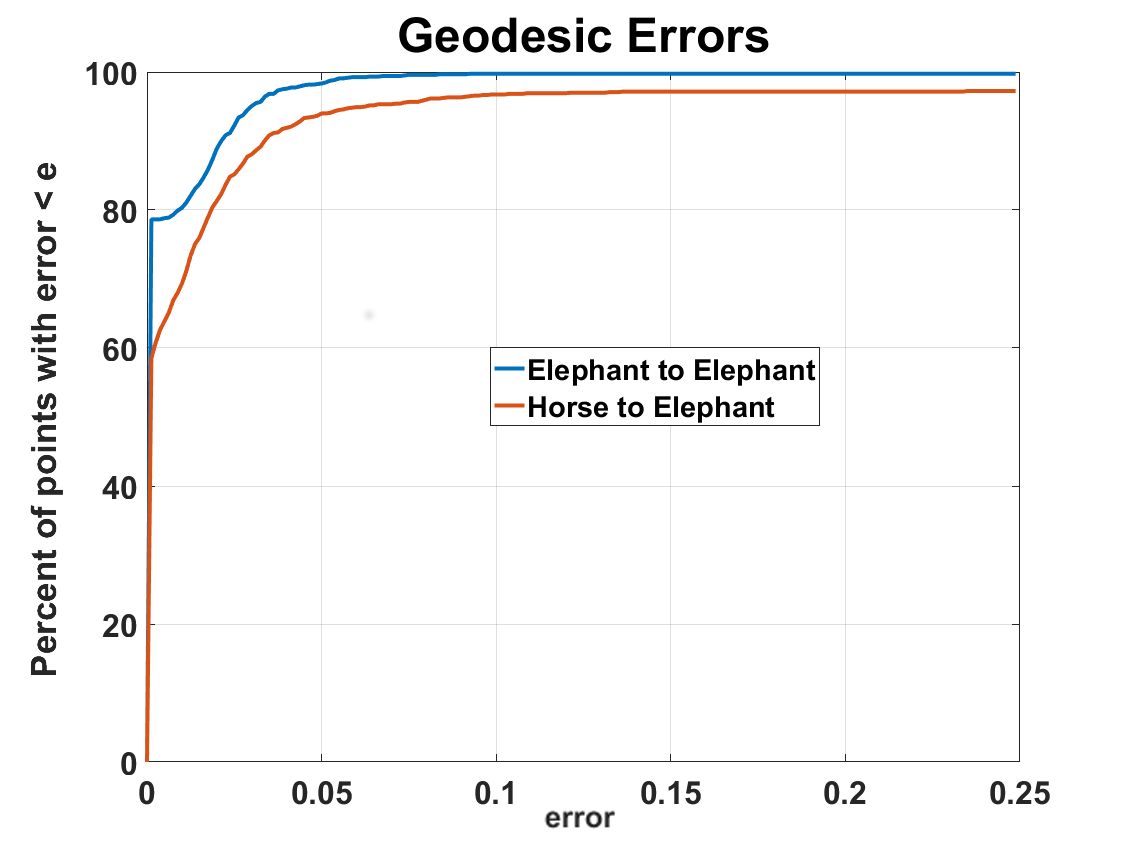}
\end{minipage}
\lrpicaption{Left: point-to-point maps for noisy data. Tight: normalized geodesic errors for noisy data.}
\end{figure}

\begin{figure}
\begin{minipage}{.45\linewidth}
\includegraphics[width=.95\linewidth]{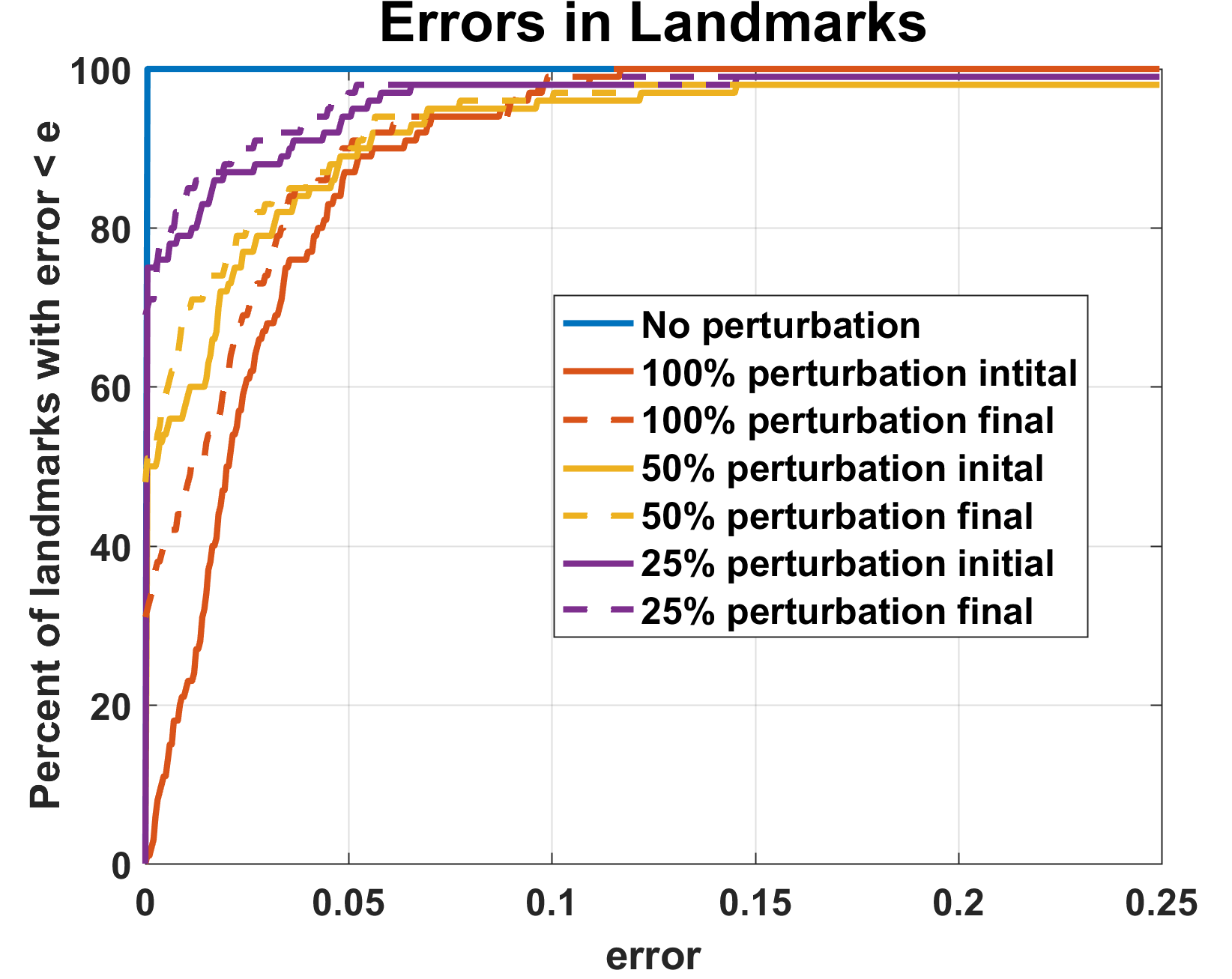}
\end{minipage}
\begin{minipage}{.45\linewidth}
\includegraphics[width=.95\linewidth]{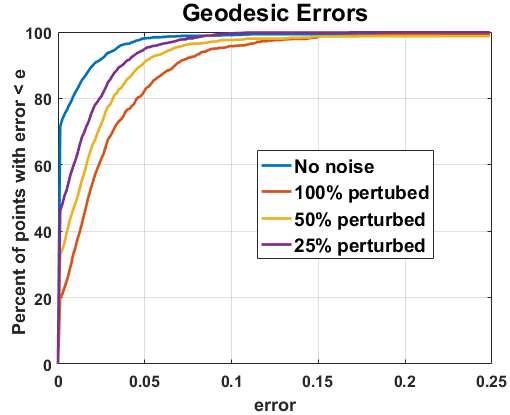}
\end{minipage}
\lrpicaption{Left: Initial perturbations to landmarks and final error of landmarks. Right: Final registration geodesic errors for all points using perturbed landmarks.}
\label{fig:robustness}
\end{figure}
%
%
%
%
%

\subsection{Benchmark Test Using the Faust Dataset} 
In our next experiment, we test our algorithm on a larger dataset to demonstrate its effectiveness and robustness on a variety of shapes. The Faust dataset is a collection of 100 3D shapes composed of 10 real individuals in 10 distinct poses 
Instead of testing all 9900 possible correspondences be each of the pairs, we select two smaller subsets of shapes to formulate to smaller test sets. For the first test, we randomly choose 100 pairs of shapes and compute the correspondences. In the second test, we choose l0 scans and ensure that each individual and each pose is represented exactly once in the test set and compute all 90 correspondence maps. (Figure \ref{Faustex}) \cite{bogo2014faust}.
This selection criteria ensures that no pairs are from the same the pose or individual. The bottom left graph in Figure \ref{Faustex} shows the average error of the mappings for each of these tests. We see that our algorithm again computes very accurate correspondences for both tests. 
\begin{figure}
\centering
\includegraphics[width=.95\linewidth]{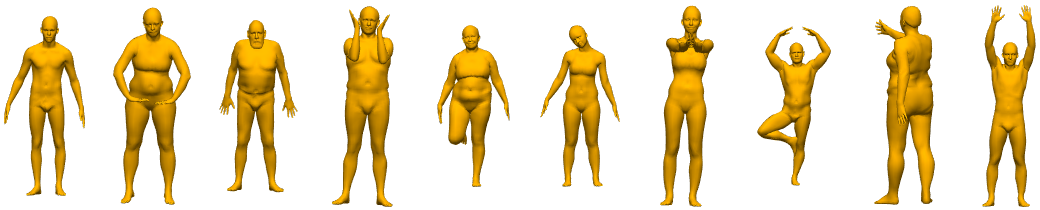}
\includegraphics[width=.45\linewidth]{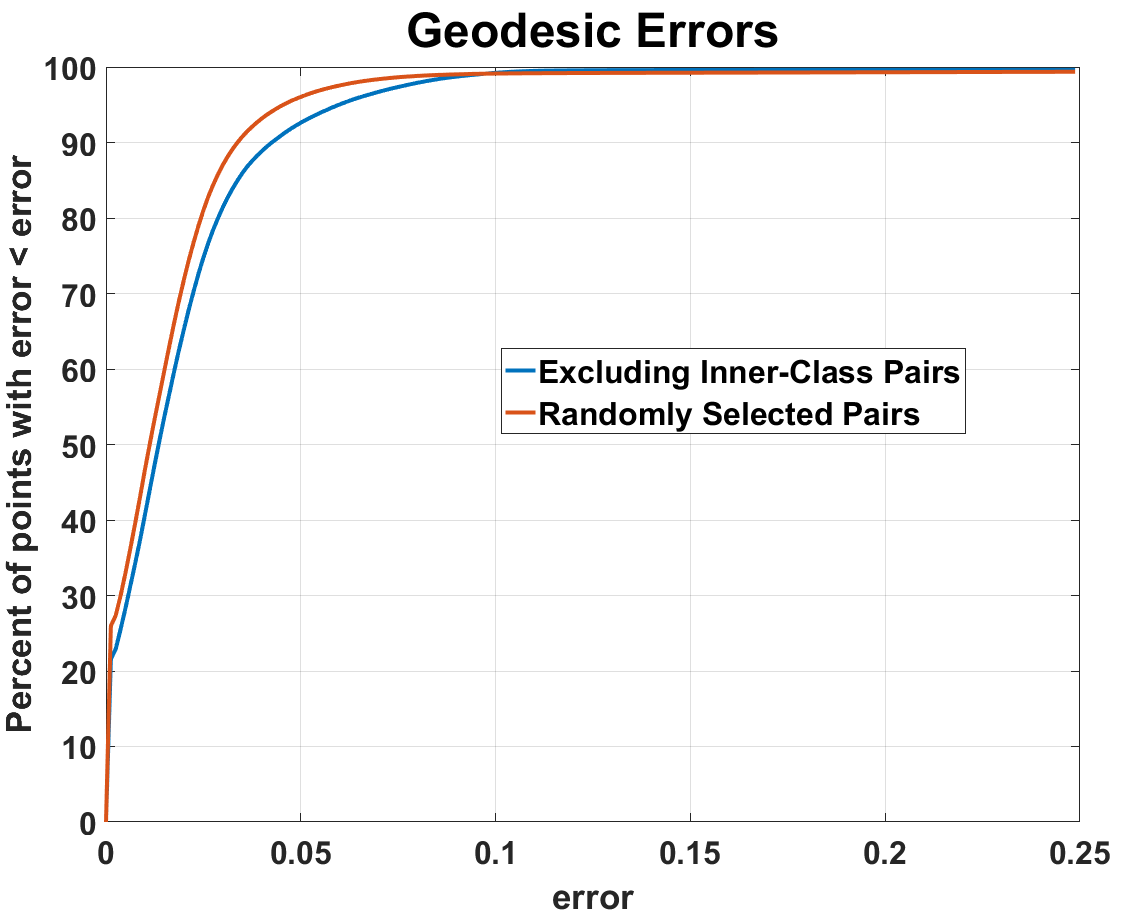}
\lrpicaption{Top: selected subjects from the Faust data set \cite{bogo2014faust}. Bottom: geodesic errors for randomly selected and least isomorphic pairs.} 
\label{Faustex}
\end{figure}
Furthermore, we see that the results for the harder test set are very close to the results for the first test set. This indicates that our approach can effectively handle non-isometric matching problems with large deformations. For each of these test we employ our sub-sampling scheme outlined in Algorithm \ref{alg:sub}, using a subsample of 1000 points to compute a basis which we use as a warm start for the dense meshes. Each pair took roughly 45 minutes to compute.


\subsection{Comparisons with Other Nonisometric Techniques}
Figure \ref{fig:comps} shows the a comparison our algorithm and that of the kernel matching \cite{vestner2017efficient}, coupled quasi-harmonic basis \cite{kovnatsky2013coupled}, basis matching (no deformation in 6.2) and functional maps \cite{ovsjanikov2012functional} approaches on the non-isometric horse to elephant problem and on a nearly isometric problem taken from the Faust dataset. For each test the algorithms used 100 randomly generate heat diffusion functions as corresponding features and solve the minimization problem until the relative objective function update falls below 10e-6. 

The horse-to-elephant test has a much larger deformation, but is also much less densely meshed. As a result the algorithms which are able to encapsulate the change in local geometry, kernel matching and our approach perform much better than methods developed for near-isometric surfaces. On the other had the problem taken from the Faust data set has a much smaller deformation, so methods which rely on the native eigensystems being closely aligned(functional maps and coupled basis) perform much better on this test then on the horse-to-elephant case.

\begin{figure}[ht]
\begin{center}
\begin{minipage}{.49\textwidth}
\includegraphics[scale=.54]{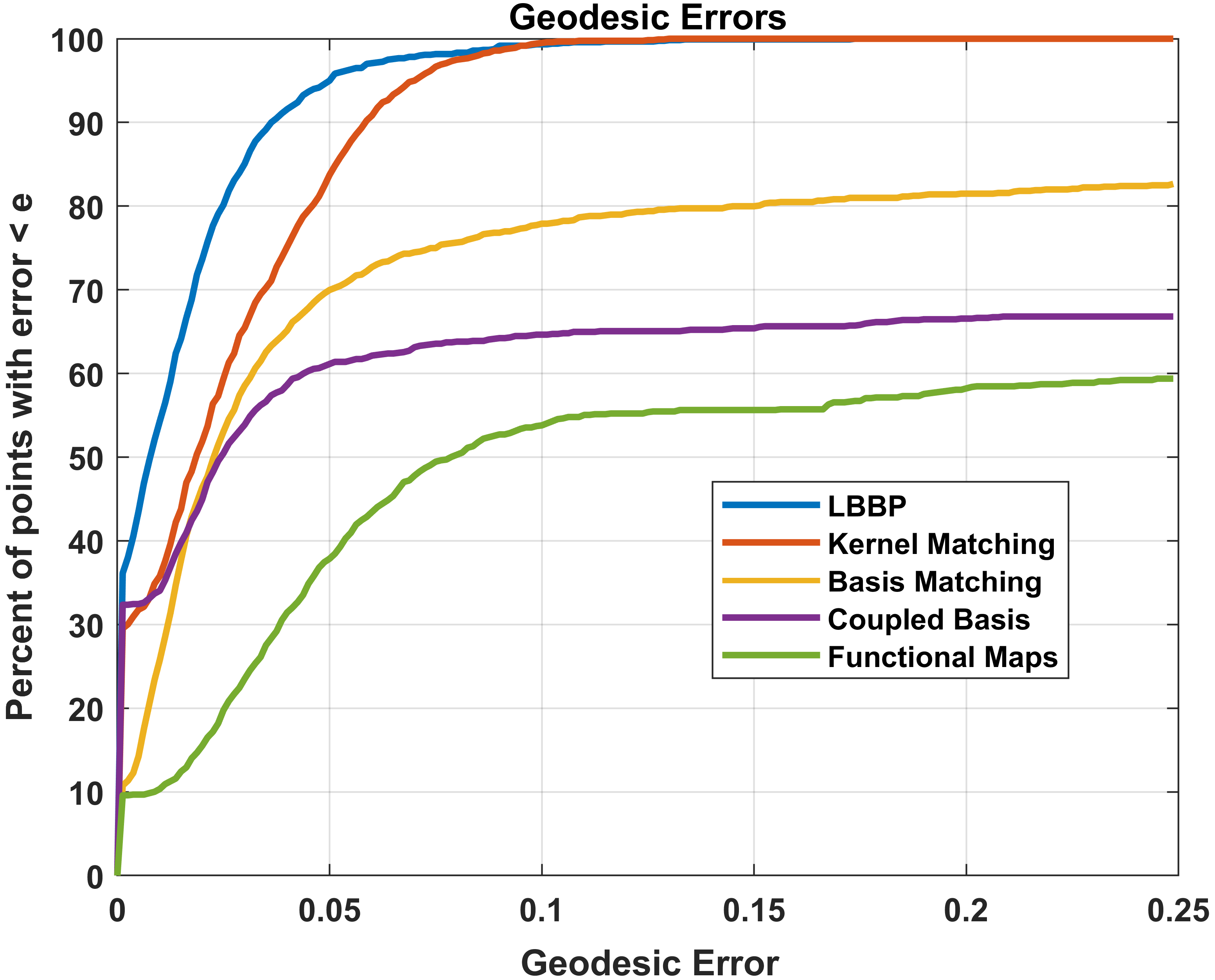}
\end{minipage}
\begin{minipage}{.49\textwidth}
\includegraphics[scale=.45]{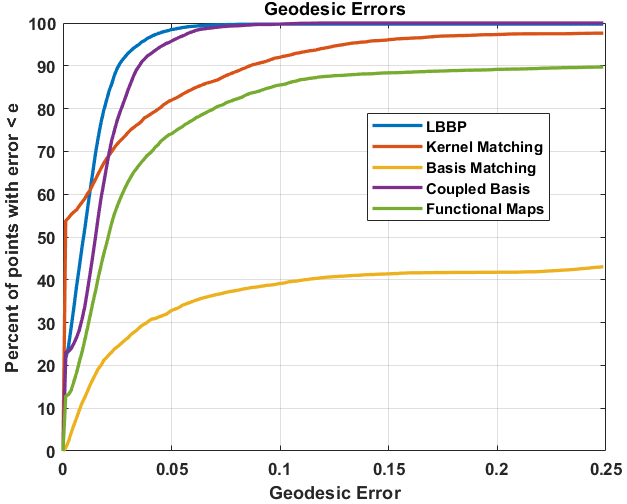}
\end{minipage}
\end{center}
\lrpicaption{Left: Comparison of methods on non-isometric horse-to-elephant. Right: Comparison of methods on Faust example.}
\label{fig:comps}
\end{figure}


\section{Final Remarks on LBBP}
\label{sec:LBBPconclusion}
In this chapter, we have developed a variation method for computing correspondence between pairs of largely deformed non-isometric manifolds. Our approach considers conformal deformation of the manifolds and combines with traditional LB spectral theory. This method naturally connects metric deformations to the spectrum of the manifold and therefore allows us to register manifolds with large deformations. Our approach simultaneously aligns the bases of the manifolds and computes a conformal deformation without having to explicitly reconstruct the deformed manifolds. We have also proposed an efficient, locally convergent method to solve this model based on the PAM framework. Finally, we have conducted intensive numerical experiments to demonstrate the effectiveness and robustness of our methods.

 
\chapter{PARALLEL TRANSPORT CONVOLUTION} \label{Chap:PTC}

Convolution has played a prominent role in various applications in science and engineering for many years and has become a key operation in many neural networks. There has been a recent growth of interests of research in generalizing convolutions on 3D surfaces, often represented as compact manifolds. However, existing approaches cannot preserve all the desirable properties of Euclidean convolutions, namely: compactly supported filters, directionality, transferability across different manifolds. In this paper we develop a new generalization of the convolution operation, referred to as parallel transport convolution (PTC), on Riemannian manifolds and their discrete counterparts. PTC is designed based on the parallel transportation which is able to translate information along a manifold and to intrinsically preserve directionality. PTC allows for the construction of compactly supported filters and is also robust to manifold deformations. This enables us to preform wavelet-like operations and to define convolutional neural networks on curved domains.

\blfootnote{Portions of this chapter previously appeared as:  S. C. SCHONSHECK, B. DONG, AND R. LAI, \textit{Parallel Transportation Convolution: a New Tool for Convolutional Neural Networks on Manifolds},  arXiv preprint, arXiv:1805.07857, 2019. \newline
Portions of this chapter have been submitted as S. C. SCHONSHECK, B. DONG, AND R. LAI, \textit{Parallel transportation convolution: deformable convolutional networks on manifold structured data}, SIAM J. Imaging Sci (2020).}

\section{Introduction to Non-Euclidean Convolution}

Convolution is a fundamental mathematical operation that arises in many applications in science and engineering. Its ability to effectively extract local features, as well as its ease of use, has made it the cornerstone of many important techniques such as numerical partial differential equations and wavelets~\cite{Daubechies:1992ten,lecun1990handwritten,mallat2008wavelet}. More recently, convolution plays a fundamentally important role in convolutional neural networks (CNN)~\cite{lecun1990handwritten} which have made remarkable progress and significantly advanced the state-of-the-art in image processing, analysis and recognition~\cite{lecun1990handwritten,bengio2015deep,krizhevsky2012imagenet,ciresan2012deep,hinton2012deep,sermanet2013overfeat,leung2014deep,sutskever2014sequence}.

In Euclidean space $\RR^n$, the convolution of a function $f$ with a kernel (or filter) $k$ is defined as:
\begin{equation}\label{EucConv}
(f * k)(x) := \int_{\RR^n} k(x-y)f(y) dy.
\end{equation}
Unlike signals or images whose domain is shift invariant (such as images in the plane), functions defined on curved domains do not always have shift-invariance. To define robust convolutional operators on these curved domains the key challenge is to properly define the translation operation. This is one of the main obstacles of generalizing CNN to manifolds.

There has been a recent surge of research in designing CNNs on manifolds or graphs. We refer the interested readers to \cite{bronstein2017geometric} for a review of recent progress in this area. These approaches can be classified into three categories: spectral patch based and group action methods. Spectral methods are based on projecting a signal onto the eigen (Fourier) space and using the convolution theorem to define convolution. Patch based methods use a patch operator to interpolate local geodesic discs on a certain given template. Group action based methods are defined on homogeneous space with a transitive group action. Here, we briefly review some of these approaches.

Spectral methods for manifold convolutions are based on the Fourier transform. The convolution theorem states that, for any two functions $f$ and $g$: 
$\F (f * g) = \F(f)\cdot \F(g)$
where $\F$ is the Fourier transform and $\cdot$ denotes pointwise multiplication. This theorem can be naturally generalized to functions on manifolds if we let $\F$ to be the projection operator onto the Laplace-Beltrami (LB) eigensystem. This method has proven effective to handle functions on a fixed domain, and can be applied to graphs as well \cite{hammond2011wavelets,bruna2013spectral,dong2015sparse,henaff2015deep}. However, these methods have two fundamental limitations. First, the uncertainly principle states that a function can have compact support in either the time or frequency domain, but not both. These methods normally use only a finite number of eigenfunctions in the Fourier domain. As a result the kernels that arise from these methods are not localized (i.e. not compactly supported in the spatial domain). The second major drawback to these methods is that since they rely on the eigensystem of the domain, any deformation of the domain will change the eigensystem which in turn changes the filters. The high-frequency LB eigenfunctions of a manifold are extremely sensitive to even small deformations. This means that anything designed for, or learned on, one manifold can only be applied to problems on the same domain. This limits the transferability of the spectral based methods, and makes them inefficient for working on large collections of shapes.

Patch based methods are originally proposed in \cite{masci2015geodesic}. In this work the authors propose the use of a local patch operator to interpolate local geodesic discs of the manifold to a fixed template and develop a Geodesic Convolutional Neural Network (GCNN). Then for each point on the manifold, the convolution is calculated as the multiplication between the values of the kernel and the extracted patch on the template. To do so, they create a local polar coordinate system at each point. One drawback to this approach is that there is no natural way to choose the origin of the polar coordinate. To overcome this, the authors consider an angular pooling operation that evaluates all rotations of their kernel at each point and selects the orientation which maximizes the convolution in a point-wise fashion. Since the angular pooling operation is computed independently at each point, the selected orientation does not reflect the geometric structure of the base manifold and may not be consistent even for nearby points.
More recently, \cite{boscaini2016learning} proposes an anisotropic convolutional neural network (ACNN) by replacing the aforementioned patch operator with an operator based on anisotropic heat kernels with the direction of anisotropy fixed on the principle curvature at each point. Although this introduces a new hyper-parameter (the level of anisotropy), it allows the kernels to be directionally aware. However, 
filters developed for applications on one manifold can only be applied to manifolds in which the local directions of principal curvature are the same. 
In \cite{monti2016geometric}, the authors proposed a mixture model network (MoNet) whereby they learn a patch operator to interpolate the functional value to a template. The convolution kernel is set to be a Gaussian function with learnable mean and covariance matrices. However, MoNet requires a choice of local coordinates that may suffer from the same drawback as GCNN and ACNN.

Group action-based methods are recently discussed in several works ~\cite{kondor2018generalization,cohen2018spherical,chakraborty2018cnn}. A typical application of these methods is to extend CNN on the unit sphere~\cite{cohen2018spherical}, where convolutional operations can be defined by transferring kernels on the unit sphere through the rotation group. This idea can be generalized to a manifold $\M$ with a transitive group action $G$, where any two points $p, q \in\M$ can be connected by some group element, i.e. there exists $g\in G$ such that $p = g\cdot q$. In this setting, the manifold is called a homogeneous space which essentially equivalent to a quotient group $G/G_p$ where $G_p$ is the stabilizer of the group action at $p$. However, the general manifolds considered in this paper often do not have an associated transitive group action. Therefore, it is still necessary to consider a new method to apply convolution on manifolds without group action structure. 

In Euclidean space $\RR^n$, the convolution of a function $f$ with a kernel (or filter) $k$ is defined as:.

\begin{table*}
\lrpicaption{Comparison on different generalizations of convolutional operator on general manifolds.}
\begin{center}
\resizebox{\textwidth}{!}{
    \begin{tabular}{| l || c | c | c | c  | c | c |}
    \hline
                            Method   & Filter Type & Support  & Extraction & Directional   & Transferable & Deformable \\ 		\hline
    Spectral~\cite{bruna2013spectral} & Spectral    & Global   & Eigen          & \cmark          & \xmark         &\xmark  \\
    TFG~\cite{dong2015sparse} & Spectral    & Global   & Eigen          & \cmark            & \xmark          & \xmark  \\
    WFT~\cite{shuman2016vertex} & Spectral    & Local   & Windowed Eigen          & \cmark            & \xmark          & \xmark  \\
    GCNN~\cite{masci2015geodesic}     & Patch       & Local    & Variable       & \xmark              & \cmark         & \cmark \\
    ACNN~\cite{boscaini2016learning}  & Patch       & Local    & Fixed          & \cmark          & \cmark         & \xmark  \\
    \hline
    \textbf{PTC} & Geodesic    & Local    & Embedded & \cmark            & \cmark        & \cmark\\
    \hline
    \end{tabular}
}

\label{CompTable}
\end{center}
\end{table*}

In the Euclidean setting, convolution operators that are frequently used in practice have compactly supported filters which allow for fast and efficient computations on both CPUs and GPUs. Furthermore, they are directionally aware, deformable and can be easily transferred from one signal domain to another. Previous attempts to generalize the convolution operator on manifolds have failed to preserves one or more of these key properties. In this project, we propose a new way of defining the convolution operation on manifolds based on parallel transportation. We shall refer to the proposed convolution as the parallel transportation convolution (PTC). The proposed PTC is able to preserve all of the aforementioned key proprieties of Euclidean convolutions. This spatially defined convolution operation enjoys flexibility of conducting isotropic or anisotropic diffusion, and it also enables us to perform wavelet-like operations as well as defining convolutional neural networks on manifolds. Additionally, PTC can be shown to simplify to the Euclidean convolution when the underlying domain is flat. Therefore, the PTC can be used to define natural generalizations of common Euclidean convolution-like operations on manifolds.

To be more precise, we seek a general convolution operator of the form:
\begin{equation}\label{ManConv}
(f *_{\M} k) (x) := \int_{\M} k(x,y)f(y) d_{\M}y.
\end{equation}
where $k(x,\cdot)$ is the parallel transport of a compact support kenrel $k(x_0,\cdot)$ to $x$. In the Euclidean case \eqref{EucConv}, the term $x-y$ encapsulates the direction from $x$ to $y$, while on manifold such a vector can be understood as a tangent direction at $x$ pointing to $y$. The crucial idea of PTC is to define a kernel function $k(x,y)$ which is able to encode the direction $x-y$ using a parallel transportation in a way which naturally incorporates the manifold structure.

Table \ref{CompTable} compares the proposed PTC with previous approaches. Since the group action methods are limited to homogenous spaces, which do not fit our objective of designing convolution on more general manifolds, we do not include these methods in the table. A method is called directional if the filters are able to characterize non-isotropic features of the data. A method is transferable if the filters can be applied to manifolds with different LB eigensystems. Finally, a technique is said to be deformable if large deformations in the manifold (i.e. those which change properties such as curvature or local distances) do not drastically affect the convolution.




\section{Mathematical Background of PTC}
\label{sec:Math}
In this section, we discuss some background of differential manifolds and parallel transportation. This provides a motivation and theoretical preparation for the proposed convolutional operation.


\subsection{Manifolds, Tangent Spaces and the Exponential Map}

Let $\M$ be a two dimensional differential manifold associated with a metric $g_{\M}$. For simplicity we assume that $(\M,\delta_M)$ is embedded in $\RR^3$. We write the set of all tangent vectors at any point $x\in\M$ as $\T_x\M$ which we refer to as the tangent plane of $\M$ at $x$. The disjoint union of all tangent planes, $\bigcup_{x} \{(x,v)\in\M\times \RR^3~|~x\in\M, v\in\T_x\M\}$, forms a four dimensional differential manifold called the tangent bundle $\T\M$ of $\M$. A vector field $X$ is a smooth assignment $X: \M\rightarrow\T\M$ such that $X(x)\in\T_x\M, ~\forall x\in\M$. We denote the collection of all smooth vector fields on $\M$ as $C^{\infty}(\M, \T \M)$.

Let $\T_{x,\delta} \M = \{v\in \T_{x}\M ~|~ \langle v,v\rangle_{g_{\M}} \leq \delta\}$ be a $\delta$-neighborhood of the tangent space at a given point $x$. The {\it exponential map}, $\exp: \T_{x,\delta} \M \rightarrow \M_{x,\delta}$, maps vectors from the tangent space back onto a nearby region $\M_{x,\delta}$ of $x$ on the manifold. Formally, given $v\in \T_{x,\delta}\M$ there exists a unique geodesic curve $\gamma$ with $\gamma(0)= x$ and $\gamma'(0) = v$ such that $\exp_{x}(v) = \gamma(1)$.
Note that this map is defined in the local neighborhood where the differential equation: $\gamma'(0) = v$ with initial condition $\gamma(0) = x$ has a unique solution. The size of this neighborhood depends on the local geometry of the manifold. In fact, the exponential map defines a one-to-one correspondence between $\T_{x,\delta}\M$ and $\M_{x,\delta}$ if $\delta$ is smaller than the injective radius of $\M$~\cite{kobayashi1969foundations,chavel2006riemannian}. Since this map is a bijection, there is a natural inverse (sometimes called the \textit{logistic map}) which we denote as $\exp_{x}^{-1}: \M \rightarrow \T_{x,\delta} \M $.


\subsection{Parallel Transportation}
\label{subsec:PT}
Parallel transportation is a method of translating a vector, based an affine connection, along a smooth curve so the resulting vector is `parallel'. An affine connection translates the tangent spaces of points on a manifold in a way that allows us to differentiate vector fields along curves. Formally, an {\it affine connection} is a bilinear map $\nabla: C^{\infty}(\M, \T \M) \times C^{\infty}(\M, \T \M) \rightarrow C^{\infty}(\M, \T \M)$, such that for all smooth functions $f, g$ and all vector fields $X,Y, Z$ on $\M$ satisfy:
\begin{equation}
\left\{\begin{array}{c}
\nabla_{fX + gY} Z = f \nabla_X Z + g \nabla_Y Z \qquad \quad \\
\nabla_{X} (aY + bZ) = a \nabla_X Y + b \nabla_X Z  \quad a, b \in \RR \\
\nabla_X (fY) = df(X)Y+f\nabla_X Y  \qquad
\end{array}
\right.
\end{equation}
In particular, an affine connection is called the Levi-Civita connection if it is torsion free ($\nabla_X Y - \nabla_Y X = [X,Y]$) and compatible with the metric ( $X\langle Y,Z\rangle_{g_\M} = \langle\nabla_X Y,Z\rangle_{g_\M} + \langle Y,\nabla_X Z\rangle_{g_\M}$).
In this case, the transport induced by the connection preserves both the length of the transported vector and the angle it makes with the path it is transported along.

A curve $\gamma:[0,\ell] \rightarrow \M$ on $\M$ is called geodesic if $\nabla_{\dot \gamma(t)} \dot \gamma(t) = 0$. More precisely, using local coordinate system, we can write $\displaystyle \dot \gamma (t) = \sum_{i=1}^2 \frac{dx^i}{dt}\partial x^i$, then plugging in the covariant derivative leads to the following ordinary differential equation for a geodesic curve:
\begin{equation}\label{Geodesic Equation}
\frac{d^2x^k(t)}{dt^2} + \sum_{i,j=1}^2 \Gamma^k_{ij}(t) \frac{dx^i(t)}{dt} \frac{dx^j(t)}{dt} = 0, \qquad k = 1, 2
\end{equation}
where $\Gamma^k_{i,j}$ is the Christoffel symbols associated with the local coordinate system.
For any two points $x_0$ and $x_1$ on a complete manifold $\M$, there will be a geodesic $\gamma:[0,\ell] \rightarrow\M$ connecting $x_0$ and $x_1$.
A vector field $X(t)$ on $\gamma(t)$ is called {\it parallel}  if $\nabla_{\dot \gamma} X = 0$. Therefore, given any vector $v \in \T_{x_0} \M$,  we can transport $v$ to a vector $v'$ in $\T_{x_1}\M$ by defining $v' = X(\ell)$ from the solution of the initial value problem $\nabla_{\dot \gamma(t)} X(t) = 0$ with $X(0) = v$.
In other words,
If we write $X(t) = \sum_{i=1}^2 a^i(t)\partial x^i$, the problem of solving $X$ reduces to find the appropriate coefficients $\{a^k(t)\}$ satisfying the parallel transport equation. This can be written as the following first order linear system:
\begin{equation}
\left\{\begin{array}{c}
\displaystyle \dfrac{d a^k(t)}{dt} +  \sum_{i,j=1}^2 \dfrac{d \gamma^i}{dt} a^j(t)\Gamma ^k_{ij} = 0, \quad k = 1, 2 \vspace{0.2cm}\\
\sum_{i=1}^2 a^i(0) \partial x^i = v
\end{array}\right.
\label{eqn:PTC_ODE}
\end{equation}
Solving this equation finds a parallel vector field $X$ along $\gamma(t)$ which provides parallel transportation of $v = X(0) \in\T_{x_0}\M$ to $X(\ell)\in\T_{x_1}\M$. We denote the parallel transportation of a vector from $x_0$ to $x_1$ along the geodesic as $\PT_{x_0}^{x_1}:\T_{x_0,\delta}\M \rightarrow\T_{x_1,\delta}\M$.


\section{Parallel Transport Convolution (PTC)}
\label{sec:PTC}
In this section, we introduce parallel transport convolution on manifolds which provide a fundamental important building block of designing convolutional neural networks on manifolds. After that, we discuss a useful numerical discretization of PTC.

\subsection{Mathematic Definition of PTC}
Unlike one-dimensional signals or images whose base space is shift invariant, many interesting geometric objects modeled as curved manifolds do not have shift-invariance. This is an essential barrier to adopt CNN to conduct learning on manifolds and graphs except for a few recent work where convolution is defined in the frequency space of the LB operator~\cite{bruna2013spectral,shotton2013real,rodola2014dense,rustamov2013wavelets}. These methods only manipulates the LB eigenvalues by splitting the high dimension information to LB eigenfunctions. Limitations include that it is always isotropic due to the LB operator and can only approximate the even order differential operators~\cite{dong2015sparse}. In addition, there is another recent method discussed in~\cite{masci2015shapenet}, in which convolution is directly considered on the spatial domain using local integral on geodesic disc although it does not involve manifold structure as transportation on manifold is not considered. The lack of an appropriate method of defining convolution on manifolds motivates us to introduce the following way of defining convolution on manifolds through parallel transportation. This geometric way of defining convolution naturally integrates manifold structures and enables us to apply established euclidean learning techniques on non-euclidean problems.

Let $\M(x_0, \delta) = \{y \in \M \ | \ d_{\M}(x_0,y) \leq \delta \}$ and $k(x_0,\cdot): \M(x_0,\delta)\rightarrow \RR$ be a compactly supported kernel function centered at $x_0$ with raduis $\delta$. We assume $k(x_0,y) = 0$ for $y \notin \M(x_0,\delta)$ and require the radius of the compact support parameter $\delta$ be smaller than the injective radius of $\M$ to guarantee the bijectivity of the exponential map. Note that this is a very mild assumption, since most modern CNN architectures use filters which are much smaller than the entire image. It is also important to remark that parameterization of $k(x_0,\cdot)$ can be determined by user. It may be designed hand designed for specific applications, or be learned as a component of a neural network.

Our idea of defining convolution on manifolds relies on transporting this compactly supported kernel $k(x_0,\cdots)$ to every other point on $\M$ in a way which reflects the manifold geometry.
More specifically, given any point $x\in\M$, we first construct a vector field transportation $\PT_{x_0}^x:\T_{x_0,\delta}\M \rightarrow \T_{x,\delta}\M$ using the parallel transportation discussed in Section \ref{subsec:PT}. Then $k(x_0,\cdot)$ can be transported on $\M$ as:
\begin{equation}
\label{eqn:kernel}
k(x,\cdot):\M_{x,\delta} \rightarrow  \RR 
\end{equation}
\begin{equation}
y \mapsto  k\left(x_0, \exp_{x_0}\circ (\PT_{x_0}^{x})^{-1}\circ \exp_x^{-1}(y)\right)
\end{equation}
Note that the above definition is analogous to convolution in the Euclidean space \eqref{EucConv}. Here, the exponential map $\exp_x^{-1}(y)$ mimics the vector $x - y$, and $\PT_{x_0}^x$ is a generalizes the translation operation. In fact, it can be easily checked that the above definition is compatible with Euclidean case by setting the manifold $\M$ to be $\RR$.

By plugging \eqref{eqn:kernel} into \eqref{ManConv}, we can now formally define the {\it parallel transport convolution} operation of $f$ which a filter $k$, centered at $x_0$:
\begin{equation}
\begin{split}\label{eqn:ManConvRot}
(f *_{\M} k) (x) := \int_{\M} f(y)~k(x,y) d_\M y =  \\
\int_{\M} f(y) ~k\left(x_0, \exp_{x_0}\circ (\PT_{x_0}^{x})^{-1}\circ \exp_x^{-1}(y)\right) \ d_{\M}y
\end{split}
\end{equation}
As natural extensions, this approach can also be used to define dilations, reflections and rotations of the kernel by simply manipulating the reference vector $\exp^{-1}_{x}(y)$. More specifically, shrinking or expanding the kernel by a factor of $s$ is defined by multiplying the lengths of the vectors in the tangent space by $s$. If $s$ is chosen to be negative then the kernel is reflected through its center and dilated by a factor of $|s|$. Similarly, rotating the kernel can be achieved by multiplying a rotation matrix $R_{\theta}$ to the reference vectors on the tangent plane. In summary, the scaling of $k$ by $s$ with a rotation of $\theta$ is defined as:
\begin{equation}\label{RotandDil}
k_{s,\theta}(x,y) := \frac{1}{C_{x}} k\Big(x_0,\ \exp_{x_0}\circ (\PT_{x_0}^{x})^{-1}(s\  R_{\theta} \ exp^{-1}_{x}(y)\big)\Big)
\end{equation}
where $R_\theta$ is a rotation matrix and $\displaystyle \frac{1}{C_{x}}$ is a normalization constant that can be used to preserve volume of the kernel.

\begin{theorem}
Parallel transport convolution is invariant under isomorphism.
\end{theorem}
\begin{proof}
By definition isomorphims preserve the Riemannin metric and therefore distances and geodesic paths. Then both the paths $\PT_{x_0}^{x}$ and the metric $d_{\M}y$ are invariant to isomorphims, therefore so is \eqref{eqn:ManConvRot}.
\end{proof}


\subsection{Numerical Discretization of PTC}

In stead of solving the system of ODEs~\eqref{eqn:PTC_ODE} on manifolds, we novelly propose the following method to compute parallel transport by considering transition matrices among local frames generated by the vector field obtained from the distance function on manifolds. Our idea is motivated from the following fact. Given smooth vector fields $\{\vb^1,\vb^2\}$, one can define linear transformation among tangent planes $\mathcal{L}(\gamma)_s^t:\mathcal{T}_{\gamma(s)}\mathcal{M} \rightarrow\mathcal{T}_{\gamma(t)}\mathcal{M}$, then the corresponding parallel transport through the associated infinitesimal connection $\nabla_{\dot{\gamma}}V = \lim_{h\rightarrow 0}\frac{1}{h} (\mathcal{L}(\gamma)_0^h(V_{\gamma(0)}) - V_{\gamma(0)})$ can be induced \cite{knebelman1951spaces}. Therefore, construction of parallel transport is essentially equivalent to design vector fields on manifolds.

For convenience, we represent a two-dimensional manifold $\M$ using triangle mesh $\{V, E, T\}$. Here $V = \{v_i \in \RR^3\}_{i=1}^n$ denotes vertices and $T = \{\tau_s \}_{s=1}^l$ denotes faces. 
First we compute the geodesic distance function from $x_0$ to every other point by solving the Eikonal equation $|\nabla_\M D(x) | = 1$ using the fast marching method \cite{sethian1996fast,kimmel1998computing}. Next we calculate $\nabla_\M D$ and its orthonormal direction on each triangle $\tau_s$. Together with the face normal direction $\vn_{s}$,
for each triangle $\tau_s$, we construct a local orthonormal frame $\mathfrak{F}_s = \{\vb_{s}^1,\vb_{s}^2,\vn_{s}\}$ where $\vb_{s}^1,\vb_{s}^2$, reflecting the intrinsic information, are tangent to $\tau_s$,  and $\vn_{s}$, reflecting the extrinsic information, is orthogonal to $\tau_s$. For an edge adjacent with $\tau_s$ and $\tau_t$, we write $R_{st}$ as an orthonormal transition matrix such that $R_{st}\mathfrak{F}_t = \mathfrak{F}_s$. Then any vector in $\mathrm{Span}\{\vb_{s}^1,\vb_{s}^2\}$ can be transported to $\mathrm{Span}\{\vb_{t}^1,\vb_{t}^2\}$ using the transition matrix $R_{st}$. This can be viewed as a discretization of connection and used to transport a vector on the tangent space of one given point to all other points.
The compatibility condition of all $R_{st}$ discussed in~\cite{wang2012linear} can guarantee  that no ambiguity will be introduced in this way. We remark this idea can be also used for manifolds represented as point clouds by combining with the local mesh method for manifold represented as point cloud developed in \cite{lai2013local}. 


After the transportation is conducted, the convolution kernel can be transported to a new point by interpolating the transported vectors in the local tangent space at the target point.
Computationally, we define a sparse matrix $K$ where the $i^{th}$ column is the transportation of the kernel to the $i^{th}$ vertex. Thus, we have the following definition of discrete parallel transport convolution:
\begin{equation}\label{eqn:discretePTC}
(f *_{\M} k) (x) :=  K^T \textbf{M} F
\end{equation}
where $F$ is column vector representation the function $f$ at each vertex and $\textbf{M}$ is the mass matrix. Note that once we have computed the vector field of the geodesic equation, the transportation of the kernel to each new center and multiplication with $F$ is independent and can therefore be parallelized efficiently. Additionally, by discretizing the kernel function $k$ as a fixed stencil, we can precompute the transportation and interpolation of the stencil once, before training. Then, PTC can be computed very efficiently using sparse matrices products. We provide detailed implementation about computing these sparse matrices in appendix A.  

Figure~\ref{fig:ConvKernel} illustrates the effect of the proposed method of transporting a kernel function on a manifold. This result shows that the proposed method produce an analogy of the behavior of a kernel function $k(x-y)$ operating in the Euclidean domain. More importantly, we would like to emphasize that number of degrees of freedom in PTC is essentially the same as the classical convolution on Euclidean domain.

\begin{figure}[htp]
 \centering
\includegraphics[width=.7\textwidth]{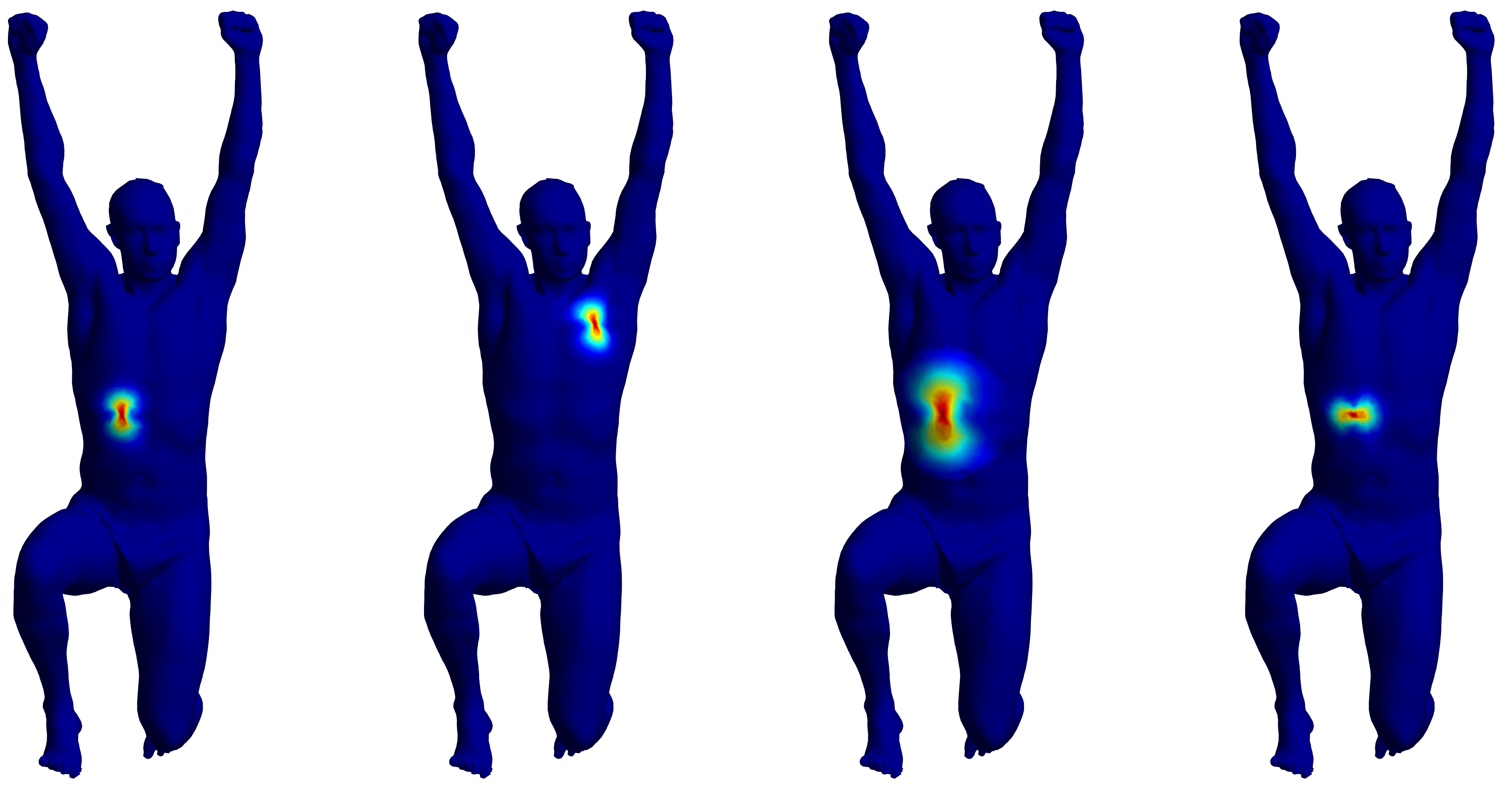}\\
{\centering (a)\hspace{2cm} (b) \hspace{2cm} (c) \hspace{2cm} (d)}
\lrpicaption{A compactly supported kernel (a) is transported on a manifold from the FAUST data set~\cite{bogo2014faust} through translation (b), translation + dilation (c) and translation + rotation (d).}
  \label{fig:ConvKernel}
\end{figure}


\section{Convolutional Tools on Manifolds}
In this section, we discuss two more important ingredients in comomon CNN architectures: stride and transposed convolution and discuss their theoretical properties of adjointness and invariance under isometric transformations.

\subsection{Strided PTC}
\label{sec:stride}

In the discrete Euclidean setting, the {\it stride} of a convolution is the distance, usually measured in pixels, which the kernel is translated on the image between each multiplication with the images \cite{glorot2010understanding}. The numerical discretization of PTC presented thus far evaluates the transported kernel $k$ at each point on the discretized point cloud (or vertex of the mesh). When the manifold is uniformly sampled, this results in an consistent distance between centers of the transported patch, and therefore a consistent stride. However, when the surface is discretized with inconsistent sampling, the distance between evaluation points will also be inconsistent. This inconsistency is overcome by the inclusion of the mass matrix into the discrete PTC formulation \eqref{eqn:discretePTC}, which normalizes the integral by the size of the the local area elements.

Our proposed strided PTC formulation is based on the following observation: A Euclidean strided convolution is evaluated by transporting a kernel to an `evenly spaced' subset of points from the euclidean domain. If conducted without padding, then this creates a contraction information and the resulting output of the convolution is both more compact (information from pixels which are far apart in the input become closer in the output) and smaller (in number of total number of pixels) than the input. To mimic this effect we compute a heretical sub-sampling of the mesh (sometimes called mesh coarsening) through a farthest point sampling (FPS) \cite{moenning2003fast} method. Each level of subsapling corresponds to each level of strided convolution. Let the original discrete manifold be represented as a set of points $\M_0$, and each hierarchical sub-sampling be computed such that $\M_0 \supset \M_1 \supset \M_2... \supset \M_k$. Then the convolution from $\M_{i}$ to $\M_{i+j}$ can be defined as:
\begin{equation}\label{eqn:stride}
(f*_{\M_{i}\rightarrow \M_{i+j}} k) (x) = \sum_{z \in \M_{i+j}} k(x,z) f(z) M_k(z) \quad \forall x \in \M_{i+j}
\end{equation}
Where $M_k(z)$ is the local mass element at $z$ from the $\M_k$ level of sampling. These mass elements can be recomputed from the sub-sample point-cloud $\M_k$ or can be aggregated by assigning each of the mass elements from the $\M_{k-1}$th sampling to its nearest neighbour in $\M_k$ sampling.


\subsection{Transposed PTC}
\label{sec:transposed}

Transposed convolution is often thought of as the the opposite (or more formally as the adjoint) of strided convolution as it is a convolution which expands the size of the input. In the Euclidean setting, this is achieved by padding a signal (most often with zeros) the performing convolution with a fixed filter. The result of this operation is a dilation of information. To mimic this operation we reverse the subsampling scheme presented in \ref{sec:stride} and define a convolution which takes signals form the $i^{th}$ level to the $(i-j)^{th}$. Given a signal $f$ defined on $\M_i$ and a kernel $K$, the transposed convolution of $f$ from $M_i$ to $M_{i-j}$ (for any $(0 < j \leq i$). That is:
\begin{equation}\label{eqn:transpose}
(f*_{\M_{i}\rightarrow \M_{i-j}} k) (x) = \sum_{z \in \M_{i-j}} k(x,z) f(z) M_k(z) \quad  \forall x \in \M_{i-j}
\end{equation}
To achieve this we need extend $f$ to all of the points in $\M_{i-j}$. This can be done either through zero padding, which is analogous to most common euclidean operations, or through harmonic extension. In either case, once the function $f$ is well defined on the up-sampled mesh, the convolution is as simple as plugging in the correct mass matrix into equation \eqref{eqn:discretePTC}.


\subsection{Adjointness}
Next we show some useful adjoined proprieties of PTC which are similar to those for euclidean convolution. In the continuous case we show the existence and provide a formula for the construction of an ad joint filter. For convolutions on euclidean surfaces, the adjoint filter is a rotation of the original. On manifolds this condition becomes $k(x,y) = k'(y,x)$. Next, we show that in the discrete case, this adjoined property can be extended to apply to strided and transposed PTC.

\begin{proposition}\label{eqn:prop1}
Given a Riemenaian, geodesic complete 2-manfiold $(\M,g)$ and a compactly supported filter $k(x_0,\cdot)$ then there exists an adjoint filter $k'(x_0,\cdot)$ such that: 
\begin{equation}
\langle x *_{\M} k ,y \rangle  = \langle x, y*_{\M} k'  \rangle
\end{equation}
Where $\langle \cdot, \cdot \rangle$ denotes the inner-product induced by the metric $g$ on $\M$
\end{proposition}
\begin{proof} Define $k'(m,n) = k(n,m)$, then:
\begin{equation} \begin{split}
\langle x *_{\M} k ,y \rangle & = \int_{\M} y(n) \int_{\M} f(m)~k(n,m) d_\M m~ d_\M n \\
& = \int_{\M}  \int_{\M} y(n) x(m)~k(n,m) d_\M m~ d_\M n \\
& = \int_{\M}  \int_{\M} y(n) x(m)~k'(m,n) d_\M n~ d_\M m \\
& = \int_{\M} x(m) \int_{\M} y(n)~k'(m,n) d_\M n~ d_\M m \\
& = \langle x  ,y *_{\M} k'\rangle 
\end{split}
\end{equation}
Note that the condition $k(m,n) = k'(n,m)$ reduces to a reflection about the center of the kernel if the manifold is flat. 
\end{proof}

Similarly, for discrete strided and transposed convolutions we have the following analogous result which includes the implicit up and down-sampling involved in these operations.

\begin{proposition}\label{eqn:prop2}
Given a two discretizations of $\M$, $\M_i$ and $M_j$ with $\M_j \subset \M_i$ with $supp(x) \in \M_i$ and $supp(y) \in M_j$. Then for discrete PTC we have: 
$$ \langle x *_{\M_i} k ,y \rangle_{M_j}   = \langle x, y*_{\M_j} k'  \rangle_{M_i} $$
Where $\langle \cdot, \cdot \rangle_{M_k}$ denotes the inner-product induced by the sample level $k$.
\end{proposition}
\begin{proof} Similar to the proof of \eqref{eqn:prop1} we have:
\begin{equation} \begin{split}
\langle x *_{\M_i} k ,y \rangle_{\M_j} 
& =\sum^{\M_j} _n  \left( y(n) \sum^{\M_j} _m x(m) k(n,m) D(m) \right) D(n)\\
& = \sum^{\M_i}_n  \left( y(n) \sum^{\M_j} _m x(m) k(n,m) D(m) \right) D(n)\\
& \overset{*}{=}  \sum^{\M_i}_n  \sum^{\M_i}_m  y(n)  x(m) k(n,m) D_n(m) D(n)\\
& = \sum^{\M_i}_m  x(m) \left( \sum^{\M_i}_n  y(n) k'(m,n)  D(n) \right) D(m)\\
& = \sum^{\M_i}_m  x(m) \left( \sum^{\M_i}_n  y(n) k'(m,n)  D(n) \right) D(m)\\
& = \langle x  ,y *_{\M} k'\rangle_{\M_i}
\end{split}
\end{equation}
with $D(n)$ being the local area element at $n$ and $k'(m,n) = k(n,m)$. Note that the step $\overset{*}{=}$ is possible since $y(i) = 0$ for $i \in M, i \notin n$ 
\end{proof}

\section{Convolutional Neural Networks on Manifolds Through PTC}
Using the proposed PTC, we can define convolutional neural networks on manifolds. We shall refer these network as PTCNets.
Similar as CNNs on Euclidean domains, a PTCNet consists of an input and an output layer, as well as multiple hidden layers including fully connected layers, nonlinear layers, pooling layers and PTC layers listed as follows.
\begin{itemize}
\item \textbf{Fully Connected}: $f_i^{out}(x) = \sum_{j=1}^{N} w_{ij} f^{in}_j(x),  \quad i = 1,\cdots,L$. This layer connects every neuron in one layer to every neuron in the previous layer.
The coefficient matrix $(w_{ij})$ parameterizes this layer and will be trained by a training data set.
\item \textbf{Vector Connected (VC)}: $f^{out} = \sum_j+1^n w_j f*^{in}_j$. This layer linearly combines channels independent of the ordering of the discretization of points. This can also be thought of as a special case of the fully connected layer, in which each column of of the weight matrix is a constant. 
\item \textbf{ReLu}: $f^{out}_i(x) = \max\{0,f_i^{in}(x)\}, \quad i = 1,\cdots,L$. This is a fixed layer applying the nonlinear Rectified Linear Units function $\max\{0,x\}$ to each input.
\item \textbf{PTC}: $f^{out}_{i,\alpha}(x) = \int k_\alpha(x,y) f^{in}_i(y) ~\mathrm{d} y \approx K_{\alpha}\textbf{M} F^{in}_i, \quad \alpha = 1,\cdots, m$.
This layer applies the proposed PTC to the input, passes the result to the next layer. By choosing the correct mass matrix, these convolutions can be strides or transposed. 
Each $k_\alpha$ is determined by the proposed PTC on manifolds with an initial convolution kernel $k_\alpha(x_0,\cdot)$, which parametrize the parallel transport convolution process and will be learned based on a training data set. 
\item \textbf{Vector Field Pooling}: $f^{out}_i(x) = \max_\alpha f^{in}_{i,\alpha}(x)$. The pooling layer can be implemented using several non-linear functions among which the max pooling is the most common way. By pooling over multiple vectorfield, we can avoid troubles caused by singularities in the vector field. See section \eqref{sec:sub:sing} for more details.
\end{itemize}

Using these layers it is straightforward to adapt established network architectures in Euclidean domain cases to manifolds case as the only change is to replace traditional convolution by PTC. In addition, back-propagation can be achieved by taking derivation of $K$. The compact support of the convolution kernel is represented as a sparse matrix which makes computation efficient.

\begin{remark}[Vector Fields]
Thus far we have only considered transportation along the geodesic from some chosen seed point. In practice we can compute the parallel transportation along any given vector field. For some applications it may be more natural to use another vector field. To do so we follow the same process except using this new vector field to form the first basis vector in $V$. This can be extremely beneficial in dealing with areas in which our geodesic vector field has a singularity. Around the singularity the direction of the vector field is often highly variable. We can simply define another vector field which is more regular in this area (but may have singularities elsewhere) to analyze information near the singularity in the first field. The problem of designing and controlling the singularities of vector fields on surfaces is a well studied problem for which many approaches already exist (see \cite{DeGoe2015vector} for a review of such techniques). It is important to note that if we would like our the results of our training to be generalizeable (i.e. when working with multiple domains) then we need to the vector fields to be generalizeable as well. For this reason using geodesic distances from canonically chosen points is a natural choice. This choice of paths is both highly non-trivial, and problem dependent. In the future we will further explore options for making this choice.
\end{remark}


\section{Numerical Experiments}
\label{sec:experiments}
To illustrate the effectiveness of the proposed PTC,  we conduct numerical experiments including processing images on manifolds using PTC, classifying images on manifolds using PTCNets and learning features on manifolds for registration and defining variational autoencoders. All numerical experiments on MNIST data were implemented in MATLAB on a PC with a 32GB RAM and 3.5GHz CPU, while the final experiment was implemented in Tensorflow with a NVIDA GTX 1080 Ti graphics card.
We remark that these experiments aim to demonstrate capabilities the proposed PTC for manipulating functions on curved domains by naturally extending existing wavelet and learning methods from Euclidean domains to curves domains. It is by no means to show that the experiments achieve state-of-the-art results on euclidean problems. 

\subsection{Wavelet-Like Operations}
\begin{figure}[ht]
\begin{center}
\includegraphics[width=.85\linewidth]{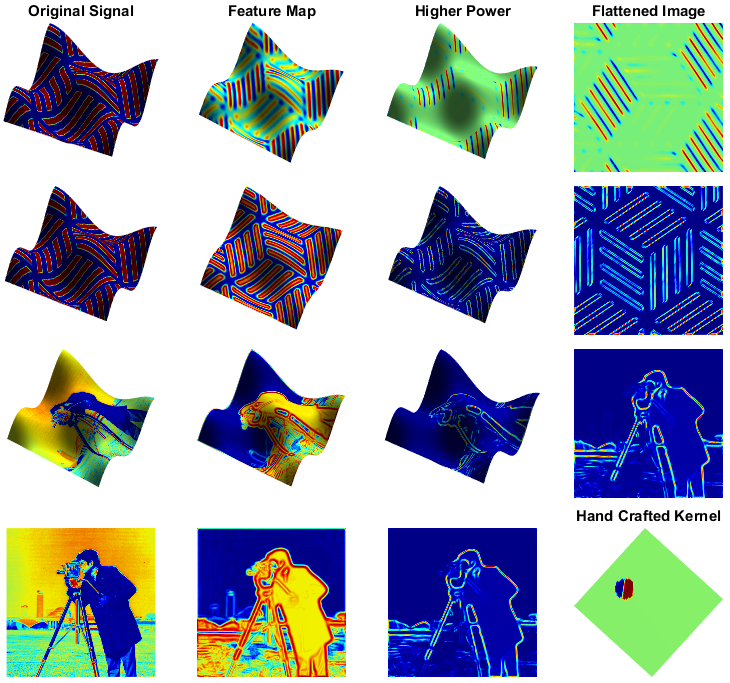}
\end{center}
\lrpicaption{First Row: Convolutions without rotation on test image. Second Row: Convolutions with rotation on test image. Third Row: Convolutions with rotation on a cameraman image. Fourth row: Traditional Euclidean convolution and the edge detector used in PTC.}
\label{fig:wavelet}
\end{figure}
In the first experiment, we demonstrate the effectiveness of our approach by performing simple signal processing tasks on manifolds. Then we  compare the PTC results to those produced by traditional techniques applied to Euclidean domains. First we apply PTC with a hand crafted edge detection filter to images on a manifold. By convolving this filter with the input image, we obtain an output feature function whose higher values indicate similarity to the predefined edge. In the first row of Figure~\ref{fig:wavelet}, it is clear that the proposed convolution successfully highlights the edges with similar orientation of the input filter. In the second row of Figure~\ref{fig:wavelet}, we allow additional rotations as we discussed in \eqref{RotandDil}. We observe that the additional rotation flexibility can reliably capture all of the edges regardless of orientations. This illustrates the directional awareness of our method. Furthermore, we apply this edge detector using PTC to a more realistic problem in the third row of Figure~\ref{fig:wavelet}. It shows that the results are very close to those produced in an analogous Euclidean setting (fourth row). In the third column, we show the feature map raised to the fifth power for better contrast and the last column shows a flattened version for easier visualization.


\subsection{Single Manifold MNIST}
\label{subsec:SingleMfd}
In this test, we conduct experiments to demonstrate the effectiveness of PTCNets to handle signals on manifolds. The most highly celebrated early applications of CNNs was the recognition of hand written digits \cite{lecun1998gradient}. We map all MNIST data to a curved manifold plotted in the left image of Figure~\ref{tab:SingleMNIST}. We use a simple network architecture consisting of a single convolution layer with 16 filters followed by a ReLu non-linear layer and then a fully connected layer which outputs a 10 dimensional vector of predictions. We apply this network architecture to four scenarios including MNIST data on a Euclidean domain using traditional convolution, MNIST data on a Euclidean domain using PTC,  MINST data on a curved domain using PTC, and MINST data on the same curved domain using spectral convolution.

\begin{table}[htp]
\lrpicaption{Comparison of our PCTNet to Euclidean case and a spectral based method on a single manifold.}
\begin{center}
\resizebox{\textwidth}{!}{
  \begin{tabular}{ c | l |c|c | }
 \cline{2-4}
\multirow{ 5}{*}{\includegraphics[width=.15\textwidth]{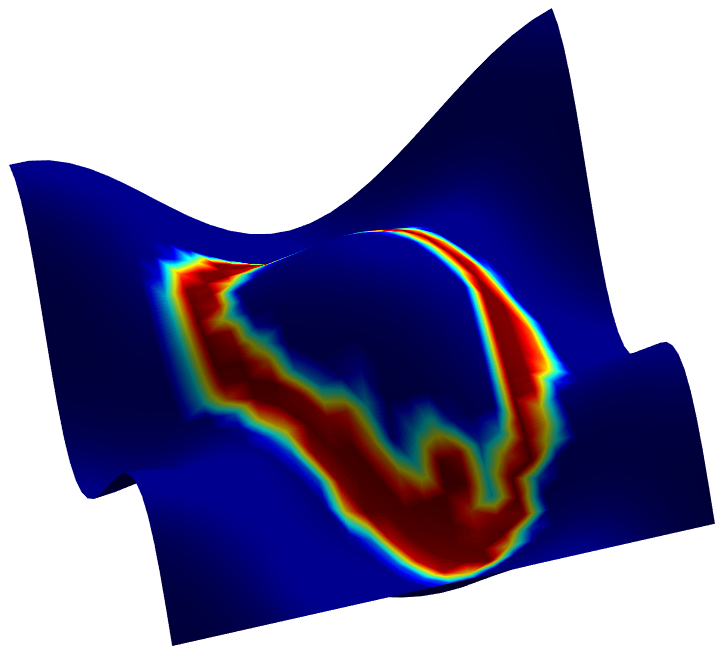}} \quad &  Network  & Domain & Accuracy \\ \cline{2-4}
   & Traditional              & Euclidean & \textbf{98.85} \\ \cline{2-4}
    & \textbf{Flat PTCNet} & Euclidean  & 98.10  \\ \cline{2-4}
   & Spectral             & Manifold  & 95.35 \\ \cline{2-4}
   &  \textbf{PTCNet}   & Manifold     & 97.96  \\ \cline{2-4}
  \end{tabular}
}
\end{center}

\label{tab:SingleMNIST}
\end{table}

Each network is implemented in MATLAB using only elementary functions and  is trained using batch stochastic gradient descent with batch size 50 and a fixed learning rate $\alpha = 10^{-3}$. We also use the same random seed for the batch selection and the same initialization. We choose such a simple training regime in order to make the effects of different convolution operations as clear as possible. We measure the results by the overall network error after 5,000 iterations.

The table~ in Figure~\ref{tab:SingleMNIST} shows the accuracy of the traditional CNN on a flat domain, a spectral net applied to a simple manifold as well as our network applied to both a Euclidean domain (Flat PTCNet) and the manifold. Similar performance of Flat PTCnet to traditional CNN illustrate that our method is an appropriate generalization of convolution from flat domains to curved domains. In addition, we observe that our method out performs the spectral network for this classification task on a curved domain.


\subsection{Multi-Manifold MNIST}
One of the advantages of our method is that filters which are learned on one manifold can be applied to different domains. The spectral convolution based methods do not have this transferability as different domains are unlikely to share the same eigensystem. In this experiment, we first directly apply the network learned by the PTCNet and Spectral networks from Section~\ref{subsec:SingleMfd} to a new manifold. As we illustrate in the first two rows of the table in Fig.~\ref{fig:MultiMNIST}, the accuracy of the spectral convolution based method is dramatically reduced since the two manifolds have quite different eigensystems. However, our PTCNet can still provide reasonable accurate results since the underlying geodesic vector fields of these manifolds is more stable to deformations than eigensystems are.

\begin{table}[htp]
\lrpicaption{Comparison of results from learning on single and multiple domains and then testing on a new manifold.}
\begin{center}
  \begin{tabular}{ |c | c | }
  \hline
 Training      &  Success Rate  \\ \hline
 Spectral      & 88.50        \\ \hline
 Single Manifold     & 95.65  \\ \hline
 \textbf{Multiple Manifolds} & \textbf{97.32} \\ \hline
  \end{tabular}
\end{center}
\end{table}

\begin{figure}
\centering
\includegraphics[width=.9\linewidth]{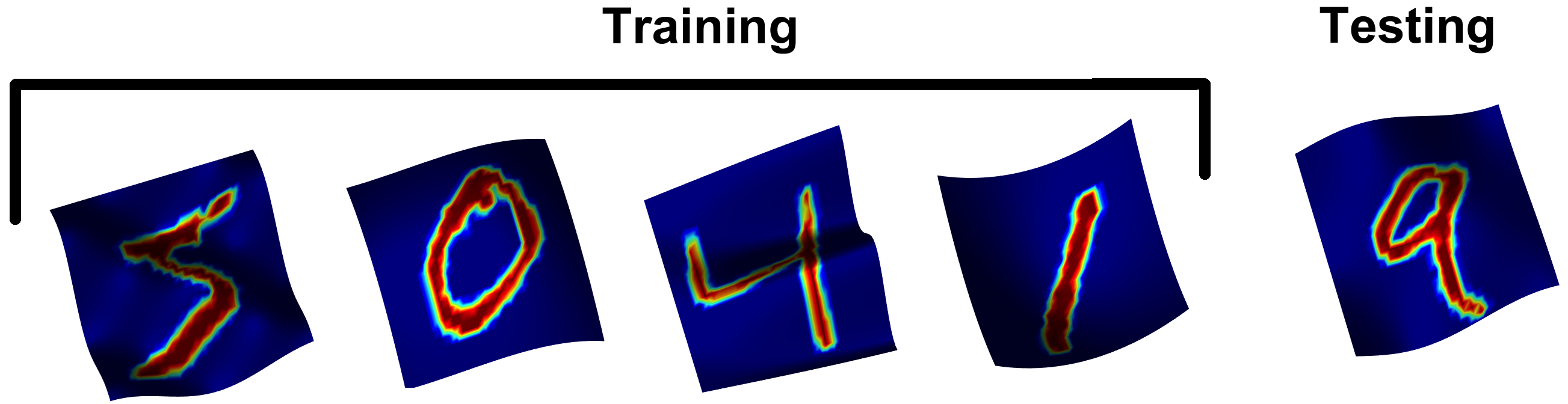}
\lrpicaption{Manifolds used for multi-manifold tests. The first four are used for training and the last is used for testing.}
\label{fig:MultiMNIST}
\end{figure}

Furthermore, we conduct a new experiment in which we train our PTCNet on a variety of manifolds and test on a different manifolds as showed in the bottom picture of Figure \ref{fig:MultiMNIST}, where the first four manifolds are used as training domains, and the fifth one is used for testing. From these pictures, it is clear that the training manifolds are quite different and therefore the spectral methods and definitions of convolution which require curvature to set their direction \cite{boscaini2016learning} cannot be applied to these problems. However the geodesic vector fields  of the manifolds are quite similar and therefore filters learned through our technique should apply to the new problem. As we can see in the last row of the Table in Figure~\ref{fig:MultiMNIST}, the network achieves a 97.32$\%$ success rate since training on multiple manifolds allows PTC network to learn greater invariance to local deformation in the metric, which enables great transferability.

\subsection{Singularities of Vector Fields}
\label{sec:sub:sing}
In each of the previous experiments the vector field used to translate the convolutional kernels is choosen to be the gradient of the geodesic from one corner of the manifold. Although our convolution is well defined everywhere on these manifolds, the filters may be more variable near this singularity. To investigate the effects that these singularities may have on,  we next test our network using  different types of vector fields. PTC1 uses the vector field chosen as in the previous experiments. PTC2 uses a vector field with a singularity in the center of the domain. The next test (PTC3) has two separate vector fields each with a singularity at different point on the interior of the domain. For this test, half the kernels are assigned to one vector field and half to the other. The last test uses four vector fields, each with a singularity at a different point on the interior of the manifold. Table~\ref{table:Sing} shows the results of using these vector fields on the single and multiple manifold problems described previously. We observe that the presence of singularity can negatively effect the performance, while using multiple vector fields can overcome these difficulties.
\begin{table}[htp]
\lrpicaption{Success rate (SR) comparison of several of our networks on a single (the $4^{th}$ coloum) and on multiple manifolds (the $5^{th}$ coloumn).}
\begin{center}
  \begin{tabular}{ | c || c | c | c | c | }
  \hline
    Implementation      &  VF & Sings per VF   & Single:  Accuracy  & Multi:  Accuracy   \\ \hline
    Spectral                 & - & -     & 92.10  & 88.50   \\ \hline
    PTC1                     & 1 & 0     & \textbf{96.36}  & \textbf{97.32}   \\ \hline
    PTC2                     & 1 & 1     & 94.92  & 94.51   \\ \hline
    PTC3                     & 2 & 1     & 95.89  & 95.02  \\ \hline
    PTC4                     & 4 & 1     & 96.01  & 95.28   \\ \hline
  \end{tabular}
\end{center}
\label{table:Sing}
\end{table}


\subsection{MNSIT Convolutional Variational Auto-Encoder (CVAE)}

Variatinonal Auto-encoders \cite{kingma2014} are a generic tool used for data compression and generation. Given some input signal $x$ one wishes to compute some encoder function $f:x\rightarrow \hat x$ which greatly reduces the dimension of $x$ ($dim(x) >> dim(\hat x)$) and decoder function $f^*: \hat x \rightarrow x$ which recovers $x$. Variatonal autoe-ncoders also require that the latent variable $x$ follow some unit normal distribution: $\hat x \sim N(0,I)$. This requirement allows for the creation of new data by passing random samples from $N(0,I)$ into the decoder as $\hat x$. A auto-encoder is also called convolutional, if the feature extraction in the encoder is done through strided convolution, and the upsampling in the decoder is done through transposed convolution.

\begin{figure}[ht]
\includegraphics[scale=.5]{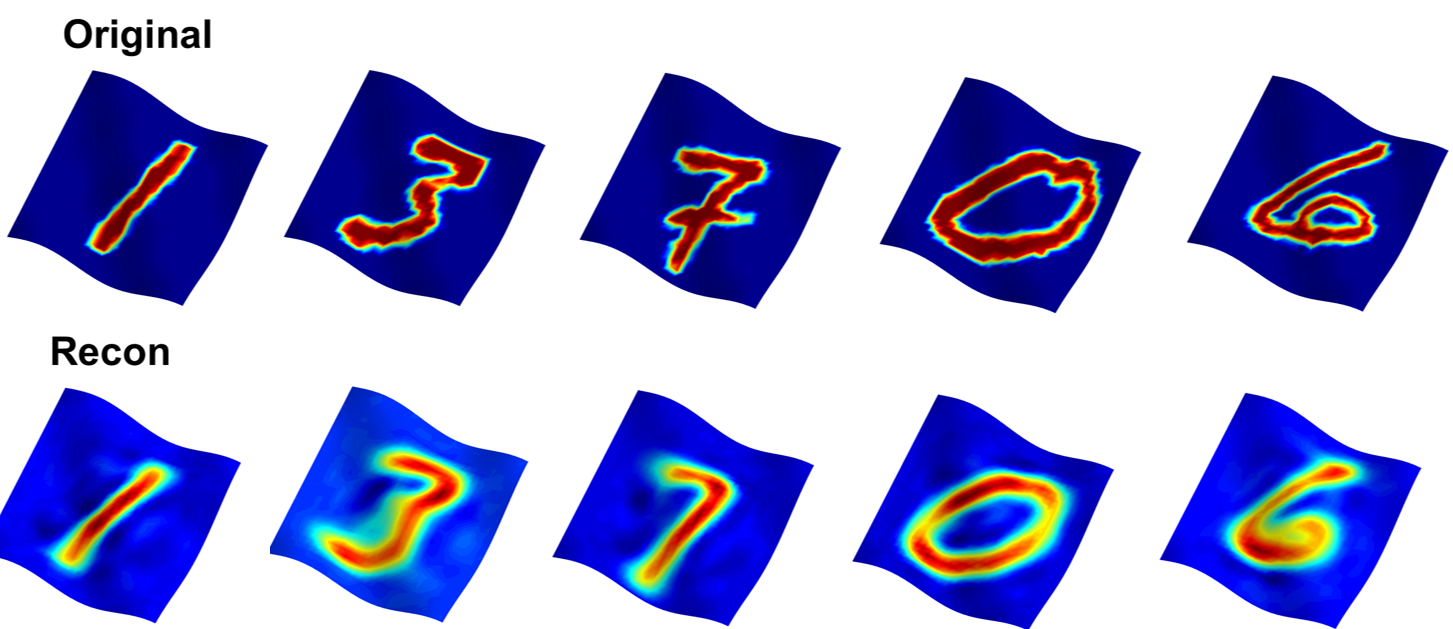}
\lrpicaption{Results of auto-encoding on a single manifold. Row 1: Input image, Row 2: Reconstructed image.}
\label{fig:MNISTPTCGen}
\end{figure}

In this test we use the MNSIT handwritten digit data base to validate our proposed definition by creating an CVAE on a manifold embediding of the MNSIT data set. We denote a PTC convolution layer as $PTC_{a}B$ where $a$ is the number of points in the discretized domain and $B$ is the number of filters in this level. Then our architecture for the encoder is:
\begin{equation}
x \rightarrow PTC_{784} 16 \rightarrow PTC_{196} 16 \rightarrow PTC_{49} 16\rightarrow FC(10,2)=(\mu,\Sigma)
\end{equation}
Similarly the decoder is defined by:
\begin{equation}\hat x \rightarrow N(\mu,\Sigma) \rightarrow FC(49,1)  \rightarrow PTC_{49} 16 \rightarrow PTC_{196} 16 \rightarrow PTC_{784} 16 \rightarrow \sum_{axis = 0} = x_{out}
\end{equation}
The network is trained by simultaneously minimizing the KL divergence between $ N(\mu,\Sigma)$ and the $L_2$ loss between $x$ and $x_{out}$. Figure \ref{fig:MNISTPTCGen} shows several examples on pairs of input signals and their recover as well as several figures generated by randomly sampling latent variables from the unit normal distribution.

\begin{figure}[h]
\includegraphics[scale=.38]{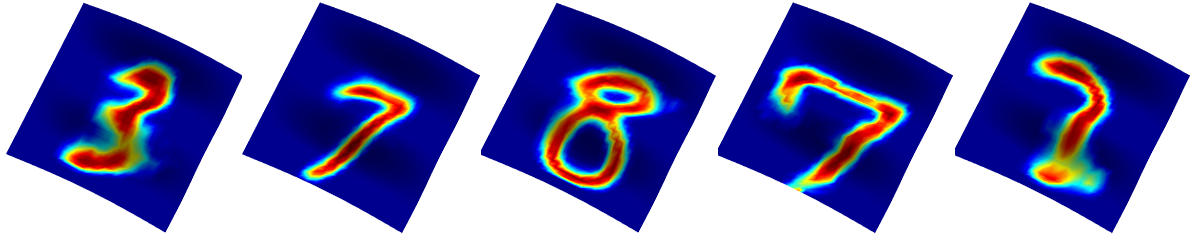}
\lrpicaption{Images generated by randomly sampling latent variables as input to trained model and applying PTC to new surface not used during training.}
\label{fig:NewSurfMNIST}
\end{figure}

One additional advantage of this framework is that, since we only use fully connected layers at the coarsest level of sampling, we only need coarse correspondences to apply a trained model to a new manifold. Since the PTC layers are agnostic to re-indexing of the data points, we can use the filters learned on on domain to apply to another. The fully connected layer still requires a correspondence in order to be consistent, but the sparse correspondences required by this approach are much easier to compute then the dense correspondences which would be required in a method without intrinsic down/up sampling. Figure \ref{fig:NewSurfMNIST} shows an example of several additional digits on a new manifold, given by a model trained on the previous surface.


\subsection{Feature Learning for Shape Correspondence}
\begin{figure*}[t]
\begin{minipage}{0.55\linewidth}
\centering
\includegraphics[width=1\linewidth]{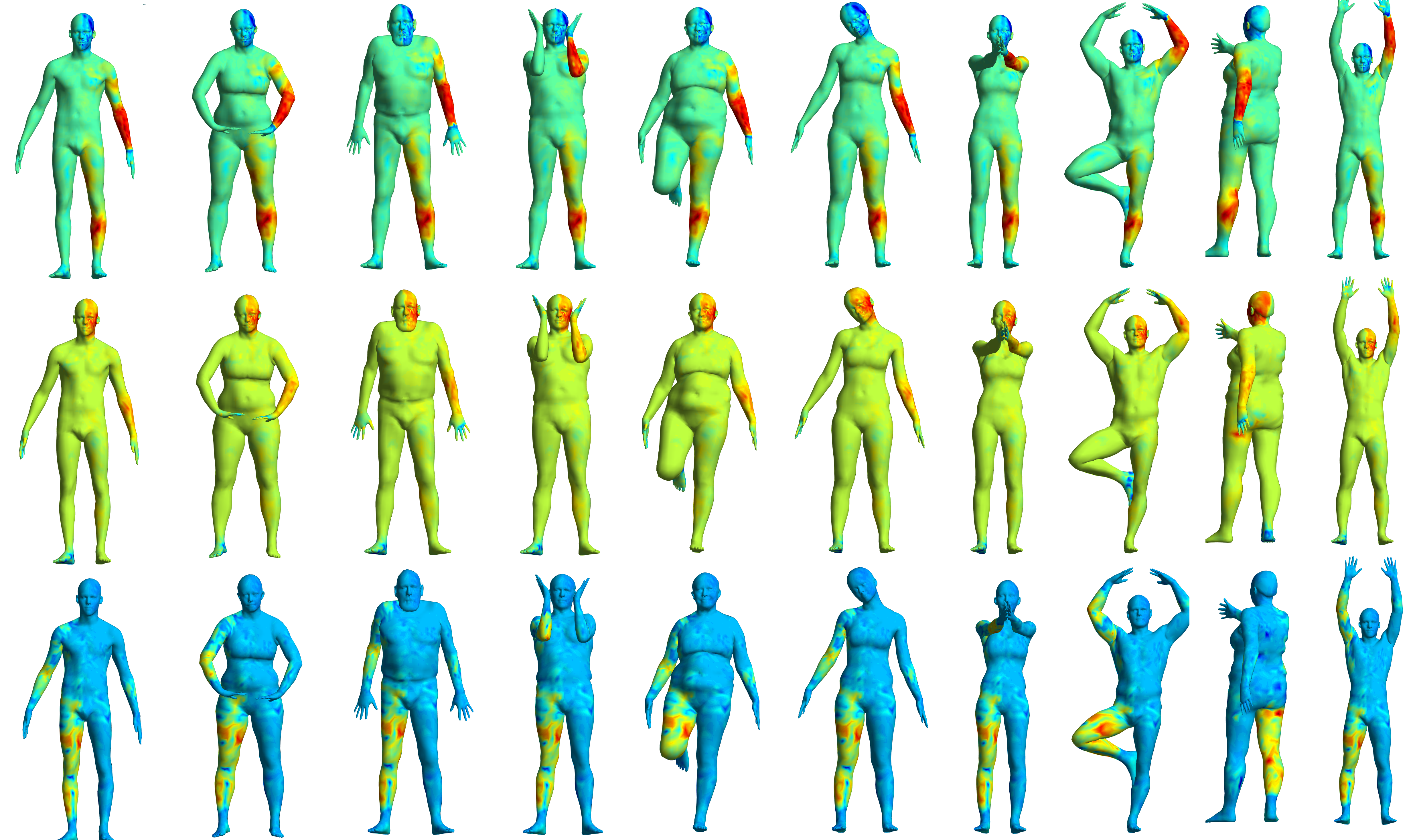}
\end{minipage}\hfill
\begin{minipage}{0.45\linewidth}
\includegraphics[width=1\linewidth]{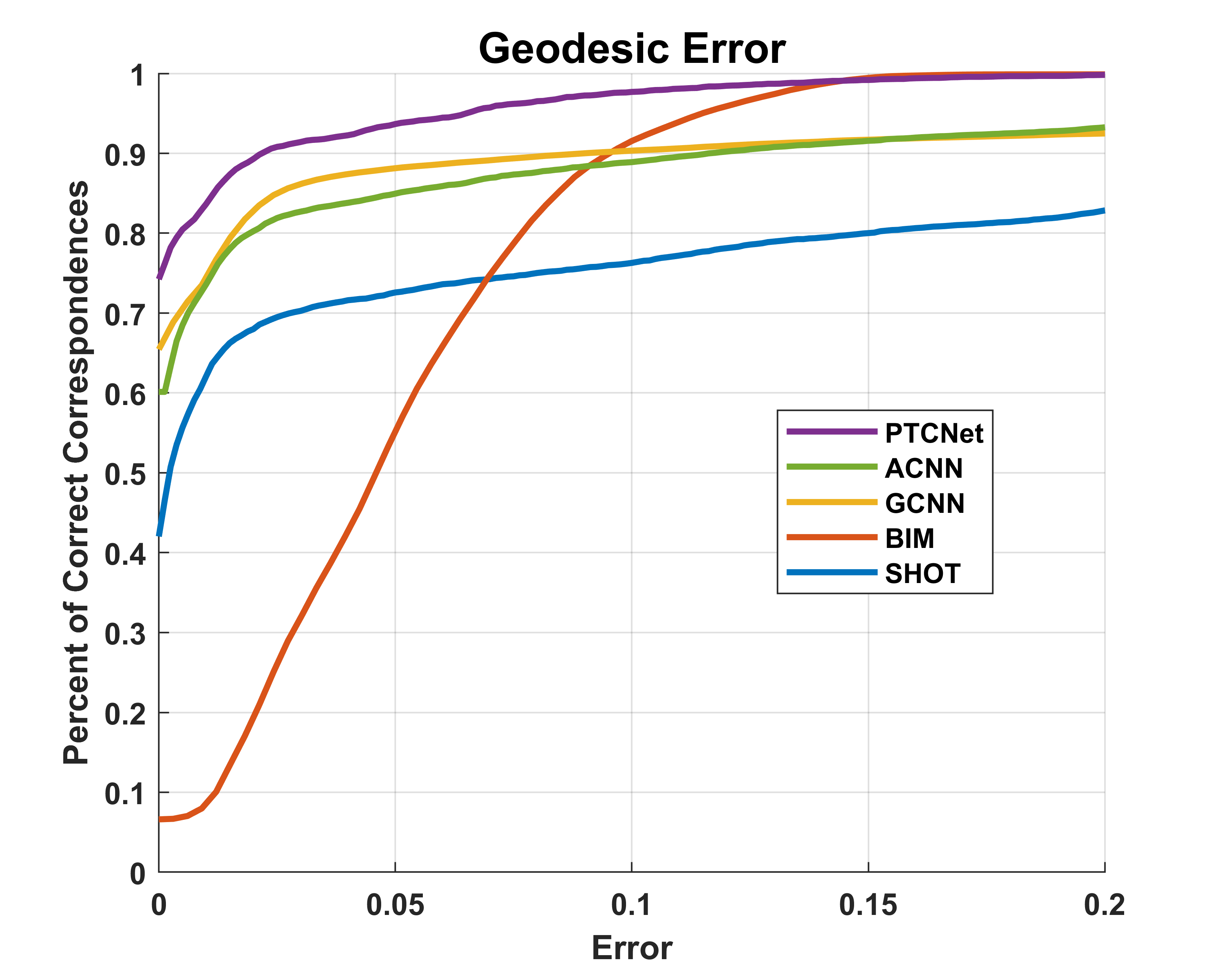}
\end{minipage}
\lrpicaption{Left: Example feature functions for shape correspondence on the Faust dataset. Right: geodesic errors in predicted correspondence of our method and several others.}
\label{fig:FaustFeats}
\end{figure*}
One important application of convolution neural networks in shape processing is the creation of geometric features \cite{bronstein2017geometric}. The goal of these networks is to output descriptor functions, $F:(\M) \rightarrow  \RR$, which accurate describe the local and global geometry of a manifold. In this section we implement a network based on the 'ShapeNet2' architecture original presented in \cite{masci2015geodesic} for shape registration, substituting in our proposed definition of convolution. We remark that this architecture is not state-of-the art, but provides a good framework for comparing geometric convolutions. In this network we input a 150 dimensional geometry vector into a vector connected layer which linearly combines these input features into a 16 dimensional signal. This signal is then passed through two layers of PTC  (each followed by a Relu non-linearity) with 16 filters in each layer. The final features are the output of the second convolution layer. The network is trained by minimizing the following triplet loss:
\begin{equation}
\begin{split}
L(S;\Theta) &= 
\sum_{x_1,x_2 \in S \times S} ||F(x_1;\Theta) - F(x_2;\Theta)||^2 
\\
&+ \lambda \sum_{P \in \Pi}(\mu_1 - || F(x_1;\Theta) -  F(x_3;\Theta))||)^2 
\end{split}
\end{equation}
where $\{x_1,x_2\}$ are similar pairs of shapes, $\{x_1,x_3\}$ are dissimilar and $\mu$ is the user parameter representing the margin. We evaluated this model the Faust dataset which contain 100 real world scans (each with $n = 6890$ points) of 10 individuals in 10 poses~\cite{bogo2014faust}. We use the first 80 figures for training, 10 for validation, and 10 for testing. Using the sparse matrix operations described in the appendix, each forward and backward propagation through a two layer network, defined on a mesh containing 6890 points, can be calculated in less than half a second. The whole training process is completed in 8 hours using the ADAM algorithm~\cite{kingma2014adam}.  Figure \ref{fig:FaustFeats} shows three of the output feature functions across different individuals in the dataset, where the first 7 individuals (10 surfaces for each individual) are used for training, the $8^{th}$ and $9^{th}$ individuals are used for validation, and the last individual is used for testing. These consistent features lead to satisfactory registration results by simply conducting the nearest point search in the feature space.  Figure \ref{fig:FaustFeats} shows our registration performance, measured by the geodesic error between the predicted correspondence and the actual correspondence, compared to error from use the heat kernel signatures which were used as input layer. We compare results with the original GCNN implementation of the ShapeNet2 \cite{masci2015shapenet}.

Finally we note that there are many more advanced architectures for shape correspondence \cite{litany2017deep, vestner2017efficient} which involve solving some version (often relaxed) of the quadratic assignment problem based on some starting map. Since these methods require a geodesic convolution network as some component of their overall architecture we do not make direct comparisons with their overall results.


\section{Conclusions on PTC}
\label{sec:PTCconclusion}
In this chapter we proposed a generalization of the convolution operation on smooth manifolds using parallel transportation and discuss its numerical implementation. Using the proposed PTC, we have performed wavelet-like operation of signals and built convolutional neural networks on curved domains. Our numerical experiments have shown that the PTC can perform as well as Euclidean methods on curved manifolds, and is capable of including directional awareness, handling problems involving deformable manifolds, in particular, learning features for deformable manifolds registration. In our future works, we will apply our PTC to different applications of comparing, classifying and understanding manifold-structured data by combining with recent advances of deep learning architectures.


 
\chapter{CHART AUTO-ENCODERS} \label{Chap:CAE}

Deep generative models have made tremendous advances in image and signal representation learning and generation. These models employ the full Euclidean space or a bounded subset as the latent space, whose flat geometry, however, is often too simplistic to meaningfully reflect the manifold structure of the data. In this work, we advocate the use of a multi-chart latent space for better data representation. Inspired by differential geometry, we propose a \textbf{Chart Auto-Encoder (CAE)} and prove a universal approximation theorem on its representation capability. We show that the training data size and the network size scale exponentially in approximation error with an exponent depending on the intrinsic dimension of the data manifold. CAE admits desirable manifold properties that auto-encoders with a flat latent space fail to obey, predominantly proximity of data. We conduct extensive experimentation with synthetic and real-life examples to demonstrate that CAE provides reconstruction with high fidelity, preserves proximity in the latent space, and generates new data remaining near the manifold. These experiments show that CAE is advantageous over existing auto-encoders and variants by preserving the topology of the data manifold as well as its geometry.

\blfootnote{ Portions of this chapter previously appeared as:   S. C. SCHONSHECK, J. CHEN, AND R. LAI \textit{Chart Auto-Encoders for Manifold Structured Data},  arXiv preprint, arXiv:1912.10094, 2020. \newline
Portions of this chapter have been submitted as S. C. SCHONSHECK, J. CHEN, AND R. LAI , \textit{Chart auto-encoders for manifold structured Data}, NeuRIPS (2020).}

\section{Chart Parameters and Generative Models}

Auto-encoding~\cite{bourlard1988auto,hinton1994autoencoders,liou2014autoencoder} is a central tool in unsupervised representation learning. The latent space therein captures the essential information of a given data set, serving the purposes of dimension reduction, denoising, and generative modeling. Even for models that do not employ an encoder, such as generative adversarial networks~\cite{Goodfellow2014}, the generative component starts with a latent space. A common practice is to model the latent space as a low-dimensional Euclidean space $\RR^d$ or a bounded subset of it (e.g., $[0,1]^d$), sometimes equipped with a prior probability distribution. Such spaces carry simple geometry and may not be adequate for representing complexly structured data. In this work, we are concerned with a widely studied structure: manifold.

A commonly held belief, known as the \emph{manifold hypothesis}~\cite{Belkin2003, fefferman2016testing}, states that real-life data often lies on, or at least near, some low-dimensional manifold embedded in a high-dimensional ambient space. Hence, a natural approach to representation learning is to introduce a low-dimensional latent space to which the data is mapped. It is desirable that such a mapping possesses basic properties such as invertibility and continuity. In differential geometry, this notion is coined \emph{homeomorphism}. Challengingly, it is known that even for simple manifolds, there does not always exist a homeomorphic mapping to the Euclidean space whose dimension is the intrinsic dimension of the data. 

We elaborate on two examples here. Consider a data set $X$ lying on the 2-dimensional sphere $S^2$ embedded in the ambient space $\RR^n$ where $n > 2$. It is well known that there exist no homeomorphic maps between $S^2$ and an open domain on $\RR^2$~\cite{rotman2013introduction}. Therefore, it is impossible for a traditional auto-encoder with a 2-dimensional latent space to faithfully capture the structure of the data. Consequently, the dimension of the latent space needs to be increased beyond the intrinsic dimension.

For another example, consider a double torus shown in Figure~\ref{fig:Eight}. When one uses a plain auto-encoder to map uniform points on this manifold to $\RR^2$, the distribution of the points is distorted and the shape destroyed; whereas if one maps to $\RR^3$, some of the points depart from the mass and become outliers. Generalization suffers, too. In Figure \eqref{fig:Eight} as well as in  the Supplementary Materials (Section~\ref{sec:vae.torus}, Figure~\ref{fig:vae}), we show the results of several variational auto-encoders with increasing complexity. They fail to generate data to cover the whole manifold; worse, the newly sampled data do not all stay on the manifold.

\begin{figure}[ht]
  \centering
  \includegraphics[width=\linewidth]{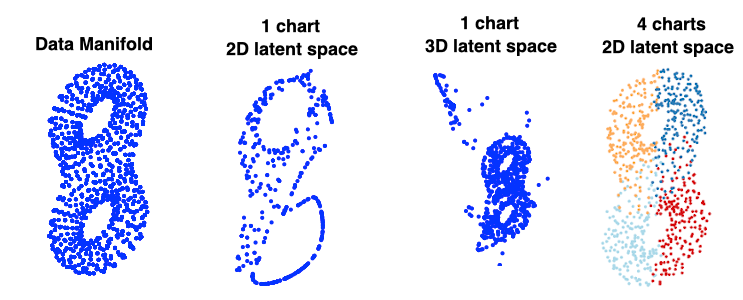}
  \vskip -0.15in
  \lrpicaption{Left: Data on a double torus. Middle two: Data auto-encoded to a flat latent space. Right: Data auto-encoded to a 4-chart latent space.}
  \label{fig:Eight}
\end{figure}

To circumvent the drawbacks of existing auto-encoders, in this work, we follow the definition of manifolds in differential geometry and propose a \textbf{Chart Auto-Encoder (CAE)} to learn a low-dimensional representation of the data. Rather than using a single function mapping, the manifold is parameterized by a collection of overlapping \emph{charts}, each of which describes a local neighborhood. Collectively cover the entire manifold. The reparameterization of an overlapping region shared by different charts is described by the associated \emph{transition function}.

As an illustration, we show to the right of Figure~\ref{fig:Eight} the same double torus aforementioned, now parameterized by using four color-coded charts. This example exhibits several characteristics and benefits of the proposed work: (i) the charts collectively cover the manifold and faithfully preserve the topology (two holes); (ii) the charts overlap (as evident from the coloring); (iii) new points sampled from the latent space remain on the manifold; and (iv) because of the preservation of geometry, one may accurately estimate geometric proprieties (such the geodesics).

These advantages are achieved through parameterizing the chart functions and the transition functions. We develop the neural network architecture of CAE and propose a training method. We conduct a comprehensive set of experiments on both synthetic and real-life data to demonstrate that CAE captures the structure of the manifold much better than do plain auto-encoders and variational auto-encoders.

\subsection{Related Work on Manifold Parameterization}\label{sec:related}

Exploring the low-dimensional structure of manifolds has led to many dimension reduction techniques in the past two decades~\cite{tenenbaum2000global, roweis2000nonlinear, cox2001, Belkin2003, He2003, zhang2004principal, Kokiopoulou2007, maaten2008visualizing}. Isomap~\cite{tenenbaum2000global} divides a data set into local neighborhoods, which are embedded into a low-dimensional space that preserves local properties. Similarly, Laplacian Eigenmaps~\cite{Belkin2003} use embeddings induced by the Laplace--Beltrami eigenfunctions to represent the data. These methods employ a flat Euclidean space for embedding and may lose information as aforementioned.

Auto-encoders use an additional decoder to serve as the reverse of a dimension reduction. The latent space therein is still flat Euclidean. One common approach to enhancing the capability of auto-encoders is to impose a prior distribution on the latent space (e.g., VAE~\cite{kingma2014}). The distributional assumption (e.g., Gaussian) introduces low-density regions that sometimes depart from the manifold. Then, paths in these regions either trace off the manifold or become invariant.

\cite{falorsi2018explorations} introduce a non-Euclidean latent space to guarantee the existence of a homeomorphic representation, realized by a homeomorphic variational auto-encoder. There are two limitations of this approach. First, it requires the knowledge of the topological class of the data set, which is generally impossible in practice. Second, it requires the estimation of the Lie group action on the latent space. If the topology of the data is relatively simple (e.g., a sphere or torus), the computation is amenable; but for more complexly structured data sets, it is rather challenging. Similarly, several recent work~\cite{davidson2018hyperspherical, rey2019disentanglement, falorsi2018explorations} studies auto-encoders with (hyper-)spherical latent spaces. These methods allow for the detection of cyclical features but offer little insight into the homology of the manifold.

Recently, \cite{lei2019geometric} established a relationship between manifolds and a generative model---the Wasserstein GAN---through the use of optimal transport that minimizes the distance between the manifold parameterized by neural networks and one estimated from training data.

Under the manifold hypothesis, \cite{chen2019efficient} extend the work of \cite{shaham2018provable} and theoretically show the existence of neural networks that approximate functions supported on low-dimensional manifolds, with a number of parameters only weakly dependent on the embedding dimension. A key feature in their proposal is a chart determination sub-network that divides the manifold into charts and a pairing sub-network that re-combines them. The premise of this approach is that the data manifold in question is known, which hinders practical application. Thus, the multi-chart latent space representation in this approach has been neither implemented nor conducted computationally. Our work introduces an implementable neural network architecture addressing these challenges.

\section{Background on Based Parameterization}\label{sec:Manifolds}

\begin{figure*}[ht]
  \centering
   \includegraphics[width=.65\linewidth]{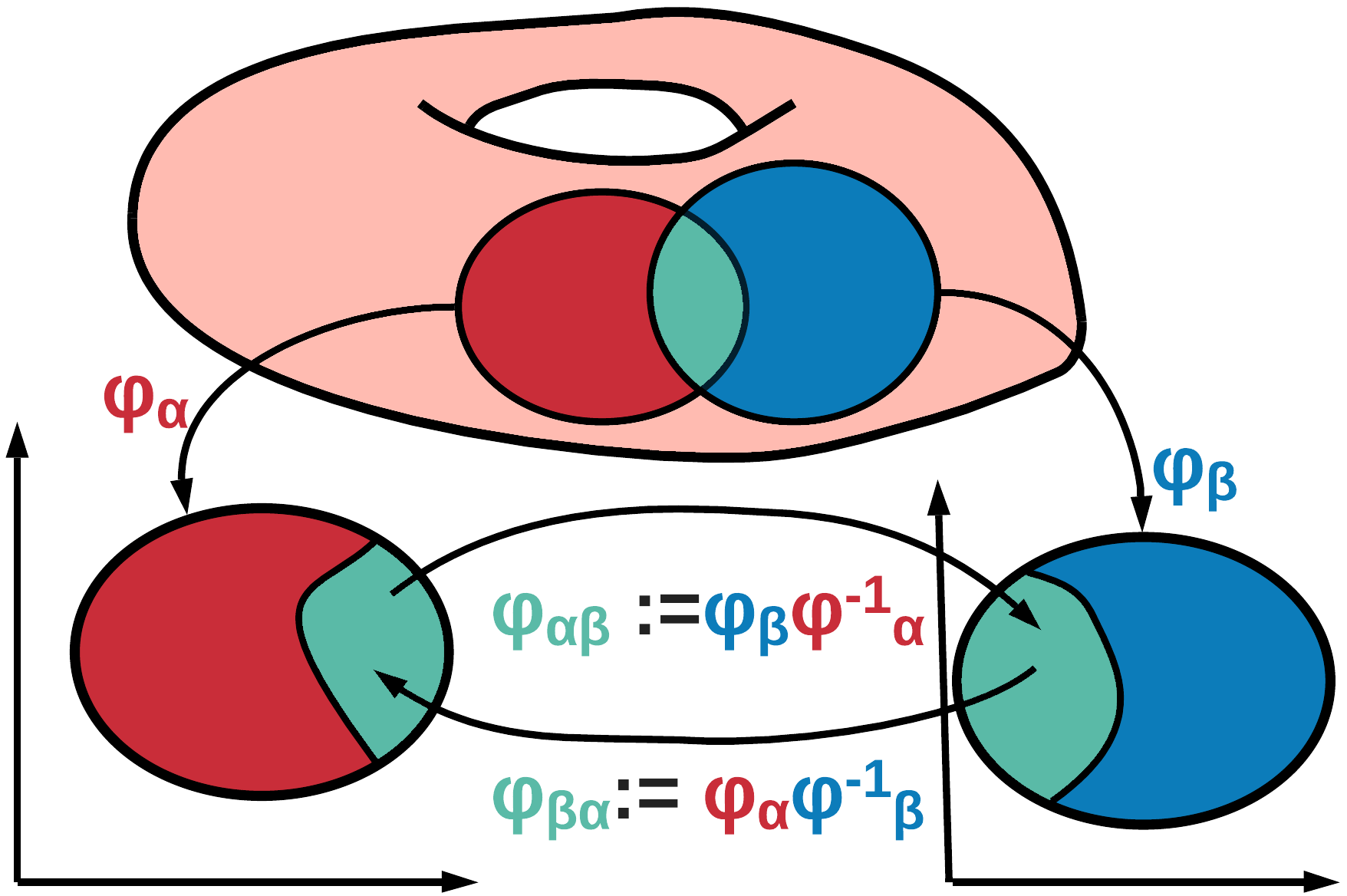}
   \lrpicaption{Illustration of a Manifold with a manifold parameterized by two overlapping charts.}
   \label{fig:ChartManifold}
\end{figure*}

\begin{figure*}[ht]
   \centering
  \includegraphics[width=1\linewidth]{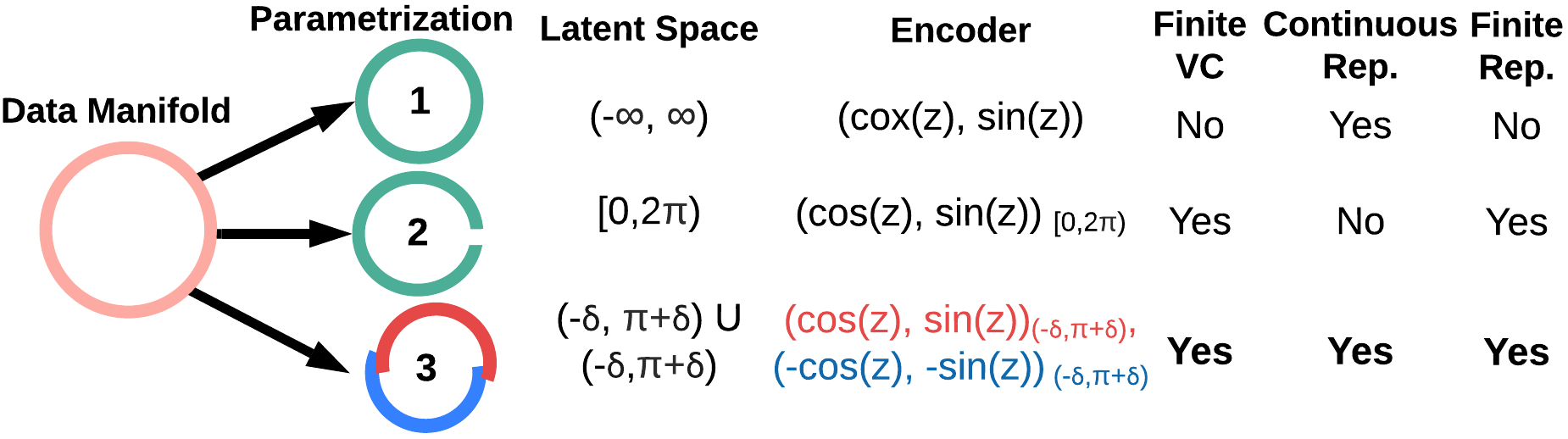}
  \lrpicaption{Possible parameterizations of a circle. The manifold approach (bottom) preserves all desired properties.}
  \label{fig:CircleManifold}
\end{figure*}

A manifold is a topological space locally homeomorphic to a Euclidean domain. More formally, a $d$-dimensional manifold is defined as a collection of pairs $\{(\M_\alpha,\phi_\alpha)\}_{\alpha}$, referred to as \emph{charts}, where $\{\M_\alpha\}_\alpha$ are open sets satisfying $\M= \bigcup_\alpha \M_\alpha$. Each $\M_\alpha$ is homeoporphic to an open set $U_\alpha\subset \RR^d$ through the coordinate map $\phi_\alpha:\M_\alpha\rightarrow U_\alpha$. Different charts can be glued together through \emph{transition functions} $\phi_{\alpha\beta}:\phi_\alpha(\M_\alpha \cap \M_\beta) \rightarrow \phi_\beta(\M_\alpha \cap \M_\beta)$ satisfying cyclic conditions (see Figure~\ref{fig:ChartManifold} left). Smoothness of the transition functions controls the smoothness of the manifold. A well-known result from differential geometry states that any compact manifold can be covered by a finite number of charts which obey these conditions. The \emph{intrinsic dimension} of the manifold is the dimension of  $U_\alpha$. See \cite{lee2013smooth} for a thorough review.

In practice, the coherent structure of data motivates us to model a given data set as samples from an unknown ground manifold. One crucial task in machine learning is to explore the topological (e.g., genus) and geometric (e.g., curvature) structure of the manifold and perform tasks such as classification and data generation. Mathematically, we explain the encoding and decoding process for a manifold as follows. Given a manifold $\M$, typically embedded in a high dimensional ambient space $\RR^n$, the encoding network constructs a local parameterization $\phi_\alpha$ from the data manifold to the latent space $U_\alpha$; and the decoding network maps $U_\alpha$ back to the data manifold $\M$ through $\phi_\alpha^{-1}$. In standard auto-encoders~\cite{bourlard1988auto,hinton1994autoencoders,liou2014autoencoder}, only one single chart is used as the latent space. In our work, multiple charts are used. Different from classical dimension reduction methods where distance preservation is preferred, we do not require the local parameterization $\phi_\alpha$ to preserve metric, but only bound its Lipschitz constant to control the regularity of the parameterization.

To illustrate the utility of such a multi-chart parameterization, we consider a simple example: finding a latent representation of data sampled from the 1-dimensional circle $S^1$ embedded in $\RR^2$. See Figure~\ref{fig:CircleManifold}. A simple (non-chart) parameterization is $(\cos(z),\sin(z))$, with $z\in (-\infty,\infty)$. However, approximating this parameterization with a finite neural network is impossible, since $z$ is unbounded and hence any multi-layer perceptron will have an infinite Vapnik-Chervonenkis dimension~\cite{blumer1989learnability}. One obvious alternative is to limit $z \in [0,2\pi)$, but this parameterization introduces a discontinuity and breaks the topology (it is theoretically known that the closed circle is not homeomorphic to $[0,2\pi)$).
Following the definition of manifolds, we instead parameterize the circle as:
\begin{equation}
    \phi_{\alpha}:  ( 0 -\delta ,\pi + \delta)  \rightarrow S^1 ,   \quad z_\alpha \mapsto (\cos(z),\sin(z)) 
\end{equation}
\begin{equation}
    \phi_{\beta}:  (0 - \delta ,\pi + \delta) \rightarrow S^1,   \quad z_\beta\mapsto  (-\cos(z),-\sin(z))
\end{equation}
\begin{equation}
    \phi_{\alpha\beta}:  (-\delta,\delta) \rightarrow (\pi-\delta, \pi+\delta),   \quad z_\alpha \mapsto z_\alpha + \pi
\end{equation}
\begin{equation}
    \phi_{\alpha\beta}:  (\pi - \delta, \pi+ \delta) \rightarrow ( -\delta,  \delta),  \quad z_\alpha \mapsto z_\alpha - \pi
\end{equation}
Although this function is cumbersome to write, it is more suitable for representation learning, since each encoding function can be represented with finite neural networks. Moreover, the topological and geometric information of the data is maintained.

Thus, instead of using only one chart as in standard auto-encoders, we propose to model the latent space with multiple charts glued by their transition functions, akin to the concept of manifolds. This geometric construction reflects the intrinsic structure of the manifold. Therefore, it is able to achieve a more accurate approximation of the data and generate realistic new ones. Moreover, once the charts and the transition functions are learned, the geometric information of the manifold, including metric, geodesic, and curvature, can be approximated according to their definition in differential geometry.

\section{Network Architecture}\label{sec:Arch}

\begin{figure*}[ht]
  \centering
  \includegraphics[width=.9\textwidth]{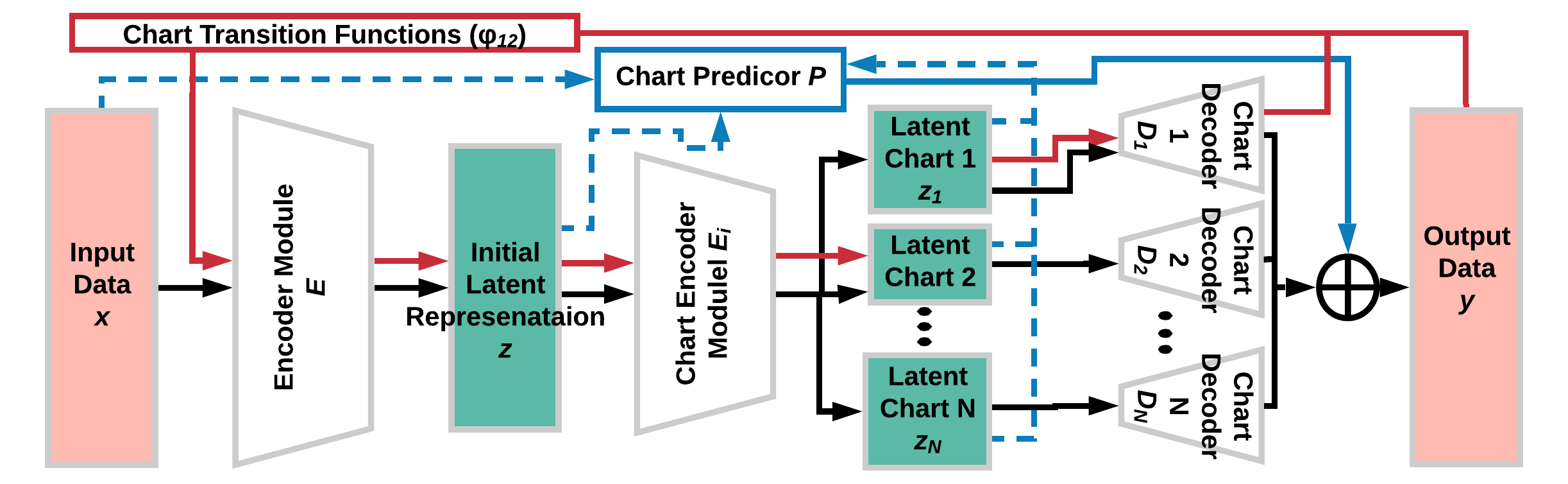}
  \vskip -0.15in
  \lrpicaption{Architecture diagram of CAE and transition functions. The red path illustrates the computation of transition function $\phi_{12}$.}
  \label{fig:Archetecture}
\end{figure*}

To integrate the manifold structure in the latent space, we propose CAE as illustrated in Figure~\ref{fig:Archetecture}. An input data point $x\in\RR^n$ is passed into an encoding module $\textbf{E}$, which creates an initial latent representation $z\in\RR^l$. Next, a collection of chart parameterizations---encoders $\textbf{E}_{\alpha}$ as analogy of $\phi_{\alpha}$---map $z$ to several chart spaces $U_{\alpha}$, which collectively define the multi-chart latent space. Each chart representation $z_{\alpha}\in U_{\alpha}$ is then passed into the corresponding decoding function---a chart decoder $\textbf{D}_{\alpha}$ as analogy of $\phi_{\alpha}^{-1}$---which produces an approximation $y_{\alpha}$ of the input data $x$. Finally, a chart prediction module $\textbf{P}$ decides which chart(s) $x$ lies on and consequently selects the corresponding $y_{\alpha}$('s) as the reconstruction of $x$. The chart transition functions may be recovered by composing the chart decoders, initial encoder, and the chart encoders. Hence, their explicit representations are not essential to the neural network architecture and we defer the discussion to Supplementary Material (Section~\ref{sec:trans}).

\paragraph{Initial Encoder.}

The initial encoder $\textbf{E}$ serves as a dimension reduction step to find a low dimensional \emph{isometric} embedding of the data from $\RR^n$ to $\RR^l$. For example, given an $\RR^3$ torus embedded in $\RR^{1000}$, the initial encoder maps from $\RR^{1000}$ to a lower-dimensional space, ideally $\RR^3$. Note that however three is not the \emph{intrinsic} dimension of the torus (rather, two is); hence, a subsequent chart encoder to be discussed soon serves the purpose of mapping from $\RR^3$ to $\RR^2$. Ideally, the initial dimension reduction step preserves the original topological and geometric information of the data manifold by reducing to the minimal isometric embedding dimension. A benefit of using an initial encoder is to reduce the subsequent computational costs in decoding. This step can be replaced with a homeomorphic variational auto-encoder~\cite{falorsi2018explorations} when the topology is known, or with an appropriately chosen random projection~\cite{baraniuk2009random,cai2018enhanced}.

\paragraph{Chart Encoder.}

This step locally parameterizes the data manifold to the chart space, whose dimension is ideally the intrinsic dimension of the manifold. The chart splits are conducted through a small collection of networks $\{\textbf{E}_\alpha\}_{\alpha}$ that takes $z \in \RR^l$ as input and output several local coordinates $z_\alpha\in U_\alpha$. The direct sum $\mathcal{U} = \bigoplus_{\alpha=1}^N U_\alpha$ is the multi-chart latent space. In practice, we set $U_\alpha = (0,1)^d$ for each $\alpha$ and regularize the Lipschitz constant of the corresponding encoding map to control the size and regularity of the region $\M_\alpha \subset \M$.

\paragraph{Chart Decoder.}

Each latent chart is equipped with a decoder function $\textbf{D}_\alpha$, which maps from the chart latent space $U_\alpha$ back to the ambient space. We denote the output as $y_{\alpha}$.

\paragraph{Chart Prediction.}

The chart prediction module $\textbf{P}$ produces confidence measure $p_{\alpha}$ for the $\alpha$-th. For simplicity we let the $p_{\alpha}$'s be probabilities that sum to unity. Ideally, if the input point lies on a single chart, then $p_{\alpha}$ should be one for this chart and zero elsewhere. If, on the other hand, the input point lies on more than one overlapping chart (say, $m$), then the ideal $p_{\alpha}$ is $1/m$ for these charts. In implementation, one may use the normalized distance of the data point to the chart center as the input to $\textbf{P}$. However, for complexly structured data, the charts may have different sizes (smaller for high curvature region and larger for flat region), and hence the normalized distance is not a useful indication. Therefore, we use $x$, $z$, and/or $z_{\alpha}$ as the input to $\textbf{P}$ instead. Several examples are given in Supplementary Material (Section~\ref{app:Networks}).

\paragraph{Final Output.}

If we summarize the overall pipeline, one sees that CAE produces $y_\alpha = \textbf{D}_\alpha\circ\textbf{E}_\alpha\circ\textbf{E}(x)$ for each chart as a reconstruction to the input $x$. Typically, the data lies on only one or at most a few of the charts, the confidence of which is signaled by $p_{\alpha}$. If only one, the corresponding $y_{\alpha}$ should be considered the final output; whereas if more than one, each of the correct $y_{\alpha}$'s should be similarly close to the input and thus taking either one is sensible. Thus, we select the $y_{\alpha}$ that maximizes $p_{\alpha}$ as the final output.

All modules of the CAE may be implemented by using fully connected and/or convolution layers (with ReLU activation). Details of the implementation are given in Section~\ref{app:Networks}.

\section{Network Training}\label{sec:Details}

In this section, we discuss the details of the training scheme, including the loss function, regularization, and pre-training. We also discuss how the number of charts is obtained.

\subsection{Loss Function}\label{sec:Loss}

Recall that a chart decoder output is $y_\alpha = \textbf{D}_\alpha\circ\textbf{E}_\alpha\circ\textbf{E}(x)$; hence, $e_\alpha = \|x - y_\alpha \|^2$ denotes the reconstruction error for the chart indexed by $\alpha$. If $x$ lies on only one chart, this chart should be the one that minimizes $e_{\alpha}$. Even if $x$ lies on more than one chart, the minimum of $e_{\alpha}$ is still a sensible reconstruction error overall.

Furthermore, to obtain sensible chart prediction probabilities $\{p_{\alpha}\}$, we will take the cross-entropy between them and $\{\ell_\alpha = \mathrm{softmax}(-e_\alpha)\}$ and minimize it. If $x$ lies on several overlapping charts, on these charts the $y_\alpha$'s are similar and off these charts, the $y_\alpha$'s are bad enough that the softmax of $-e_\alpha$ is close to zero. Hence, minimizing the cross-entropy ideally produces equal probabilities for the relevant charts and zero probability for the irrelevant ones.

Summarizing these two considerations, we use the loss function
\begin{equation}\label{eqn:LossPP}
  \mathcal{L}(x,W):=   \Big( \min_\alpha e_\alpha \Big) - \sum_{\beta=1}^N  \ell_{\beta} \log (p_{\beta}),
\end{equation}
where $W$ denotes the network parameters and $N$ is the number of charts.

\subsection{Regularization}\label{sec:Reg}

We introduce regularization to stabilize training by balancing the size of $\M_\alpha$ and avoiding a small number of charts dominating the data manifold. For example, a sphere $S^2$ needs at least two 2-dimensional charts. However, if we regularize the network with only $l_2$ weight decay, it may be able to well reconstruct the training data by using only one chart but badly generalizes, because the manifold structure is destroyed.

The idea is to add a Lipschitz regularization to the chart encoders to penalize mapping nearby points far away. Formally, the \emph{Lipschitz constant} of a function $f$ is $\sup_{x \neq y} |f(y)-f(x)|/|x-y|$. Since the chart spaces are fixed as $(0,1)^d$, controlling the Lipschitz constant of a chart function will control the maximum volume of decoding region $\textbf{D}_{\alpha}((0,1)^d)$ on the data manifold. 

The Lipschitz constant of a composition of functions can be upper bounded by the product of those of the constituent functions. Moreover, the Lipschitz constant of a matrix is its spectral norm and that of ReLU is 1. Hence, we can control the upper bound of the Lipschitz constant of a chart encoder function by regularizing the product of the spectral norms of the weight matrices in each layer.

To summarize, denote by $W_{\alpha}^k$ the weight matrix of the $k$th layer of $\textbf{E}_{\alpha}$. Then, we use the regularization
\begin{equation}\label{eqn:regularization}
  \mathcal{R}_{Lip}:= \left(\max_\alpha \prod_{k} ||W_\alpha^k ||_2\right) + \frac{1}{N} \sum_{\beta=1}^N  \prod_{k} ||W_\beta^k ||_2
\end{equation}

\subsection{Pre-Training}\label{sec:Initialization}

Since CAE jointly predicts the chart outputs and chart probabilities, it is important to properly initialize the model, so that the range of each decoder lies somewhere on the manifold and the probability that a randomly sampled point lies in each chart is approximately equal. To achieve so, we use furthest point sampling (FPS) to select $N$ points $x_{\alpha}$ from the training set as seeds for each chart. Then, we separately pre-train each chart encoder and decoder pair, such that $x_{\alpha}$ is at the center of the chart space $U_{\alpha}$. We further define the chart prediction probability as the categorical distribution and use it to pre-train the chart predictor. The loss function for each $\alpha$ is
\begin{equation}\label{eqn:InitLoss}
  \mathcal{L}_{init}(x_\alpha) := \|x_\alpha - \textbf{D}_\alpha\circ \textbf{E}_\alpha \circ \textbf{E}(x_\alpha)\|^2  
  + \|\textbf{E}_\alpha\circ \textbf{E}(x_\alpha)
  - [.5]^d\|^2 + \sum_{\beta=1}^N \delta_{\alpha \beta} \log (p_\beta).
\end{equation}
We can extend this pre-training idea to additionally ensure that the charts are oriented consistently, if desirable. See Supplementary Material (Section~\ref{sec:PCA}) for details.

We remark that although the training and pre-training altogether share several similarities with clustering, the model does more than that. The obvious distinction is that CAE eventually produces overlapping charts, which are different from either hard clustering or soft clustering. One may see a deeper distinction from the training insights. The pre-training ensures that each decoder is on the manifold, so that when training begins no decoder stays inactive. However, during training the charts may move, overlap, and even disappear. The last possibility enables us to obtain the correct number of charts \emph{a posteriori}, as the next subsection elaborates.

\subsection{Number of Charts}\label{sec:num.chart}

Since it is impossible to know \emph{a priori} the number $N$ of charts necessary to cover the data manifold, we over-specify $N$ and rely on the strong regularization~\eqref{eqn:regularization} to eliminate unnecessary charts. During training, a chart function $\textbf{E}_{\alpha}$ not utilized in the reconstruction of a point (i.e., $p_{\alpha} \approx 0$) does not get update from the loss function. Then, adding any convex penalty centered at 0 to the weights of $\textbf{E}_{\alpha}$ will result in weight decay and, if a chart decoder is never utilized, its weights will go to zero. In practice, we can remove these charts when the norm of the chart decoder weights falls below some tolerance. This mechanism offers a means to obtain the number of charts \emph{a posteriori}. We will show later a numerical example that illustrates that several charts do die off after training.

\section{Numerical Results}\label{sec:results}

In this section, we analyze the performance of the proposed CAE on synthetic and benchmark data. We begin by studying geometric objects and illustrating the important properties of CAE. Then, we demonstrate its use on MNIST and Fashion MNIST and compare the performance with plain auto-encoders and variational auto-encoders.

The implementation uses Tensorflow \cite{abadi2016tensorflow} and the built in ADAM optimizer with learning rate \text{3e-4} and batch size 64 to train for 100 epochs. The standard train/test split was used for MNIST and Fashion MNIST. The penalty for the Lipschitz regularization was set to \texttt{1e-2} for all tests. Demo code is available at 

\subsection{Illustrative Examples}

\paragraph{Chart Overlap and Transition.}
As motivated earlier, even a circle cannot be mapped to a 1-dimensional latent space homeomorphicly, which motivates the use of a multi-charted latent space. In Figure~\ref{fig:overlap}, we show a 4-chart result trained with 1000 points. On the top row, for each chart we decode points whose latent values are between 0.1 and 0.9. These charts overlap and the chart probabilities for each point are shown on the bottom row of the figure. One sees the smooth transition of the probabilities. By taking the argmax of the chart probabilities, the lower right corner of the figure illustrates the reconstruction of the entire circle.

\begin{figure}[ht]
  \centering
  \includegraphics[width=1\linewidth]{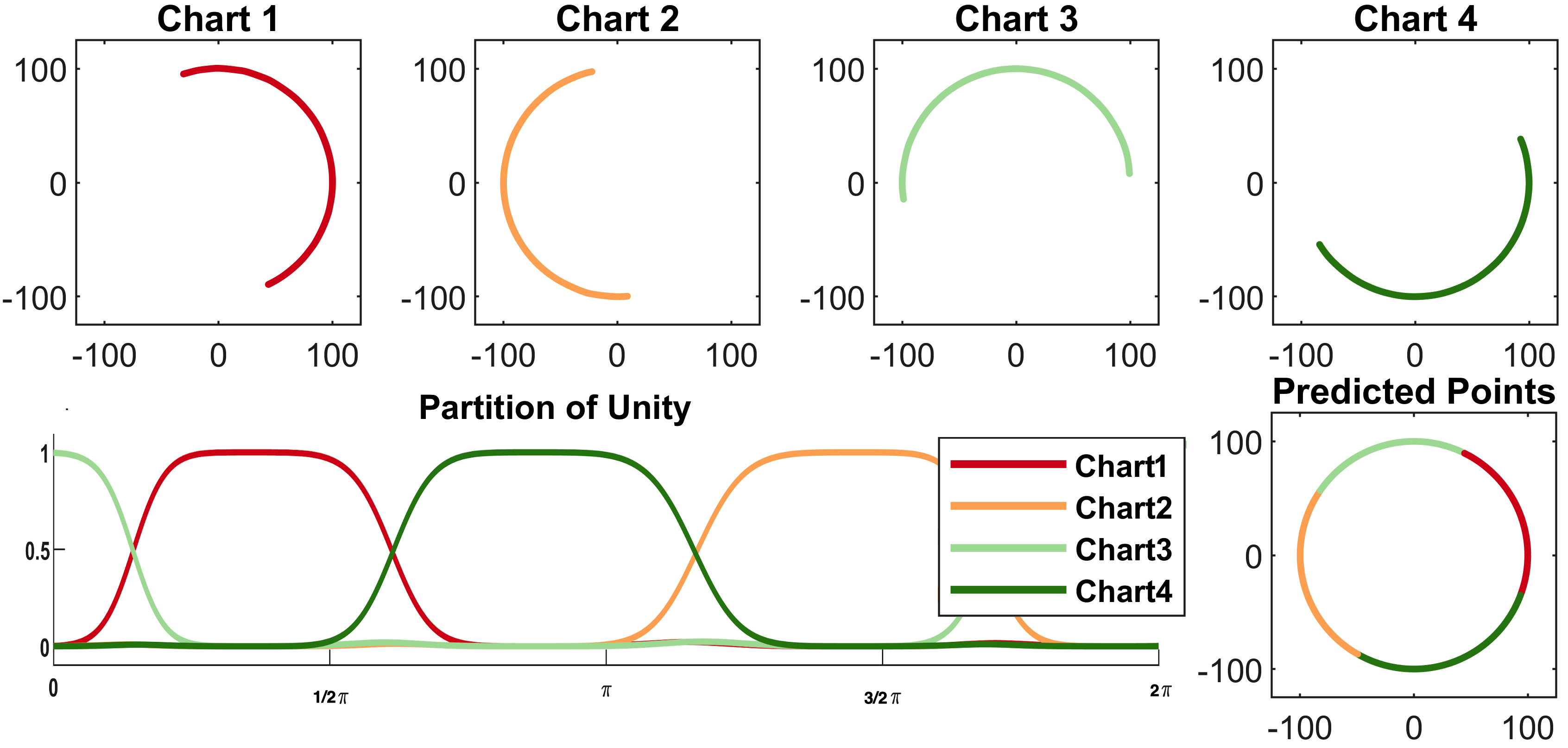}
  \vskip -0.15in
  \lrpicaption{Top: Individual charts. Bottom Left: Transition of the chart probabilities. Bottom Right: Charts after taking max of $p_\alpha$.}
  \label{fig:overlap}
\end{figure}

\paragraph{Effects of Lipschitz Regularization.}
In Section~\ref{sec:Reg} we mentioned the use of Lipschitz regularization as an important tool to stabilize training and encourage reasonable chart size. In Figure~\ref{fig:SphereEncoders} we show the result of autoencoding a sphere in $\mathbb{R}^3$ using a 2-dimensional charted latent space. The top row shows the charts trained with Lipschitz regularization and bottom without. One clearly sees that with Lipschitz regularization the charts are well localized, whereas without such regularization each chart spreads over the sphere but none covers the entire sphere well.

\begin{figure}[ht]
  \centering
  \includegraphics[width=.8\linewidth]{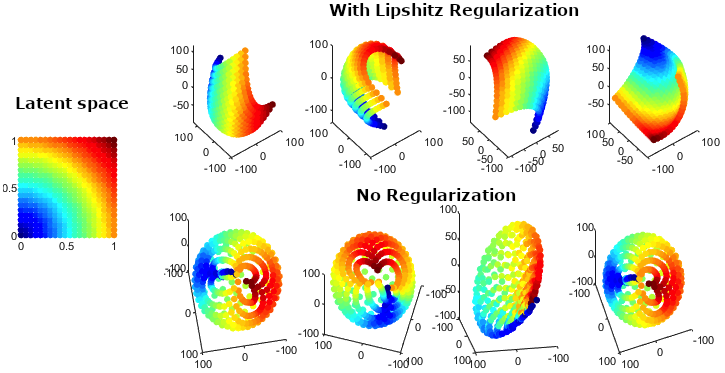}
  \vskip -0.1in
  \lrpicaption{Left: Chart latent space. Top: Model with Lipschitz regularization. Bottom: Model without Lipschitz regularization.}
  \label{fig:SphereEncoders}
\end{figure}

\paragraph{Automatic Chart Removal.}
As discussed in Section~\ref{sec:num.chart}, it is hard to know a priori the sufficient number of charts necessary to cover an unknown manifold. Hence, we propose over-specifying a number and relying on regularization to eliminate unnecessary charts. In Figure~\ref{fig:kill} we illustrate such an example. Pretrainining results in four charts but subsequent training removes two automatically.

\begin{figure}[ht]
  \centering
  \includegraphics[width=.5\linewidth]{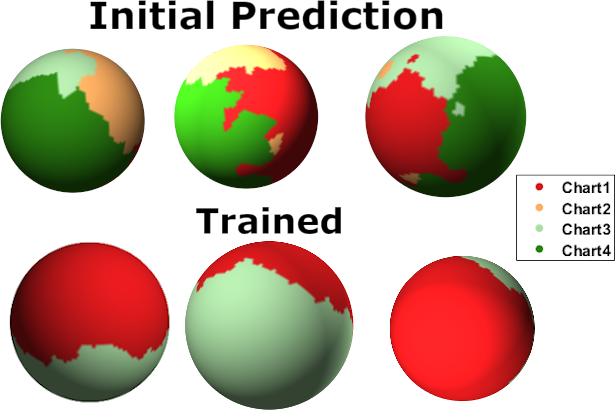}
  \vskip -0.1in
  \lrpicaption{Results of patch-removal techniques. Top: Pre-trained charts. Bottom: Final charts after training.}
  \label{fig:kill}
\end{figure}

\paragraph{Measuring Geodesics.}
One advantage of CAE compared with plain auto-encoders and VAEs is that it is able to measure geometric properties of the manifold, e.g., geodesics. In Figure~\ref{fig:geo} we illustrate such an example. To measure the geodesic distance of two points, we encode each point, connect them in the latent space, and sample points along the connection path. We then approximate the geodesic distance by summing the Euclidean distances for every pair of adjacent points. By increasing the number of sampling points we can improve the approximation quality. The figure shows a few geodesic curves and their approximation error, decreasing with denser sampling.

\begin{figure}[ht]
  \centering
  \includegraphics[width=.7\linewidth]{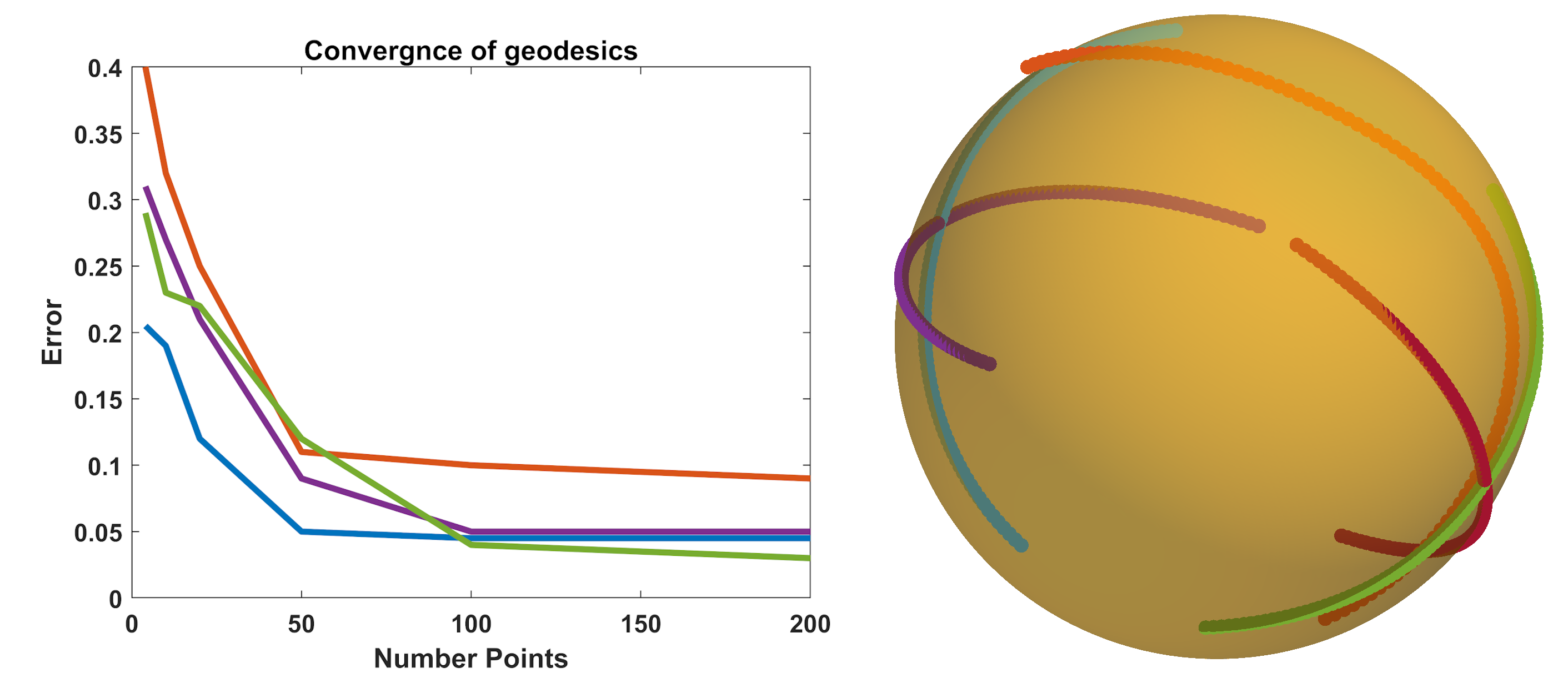}
  \vskip -0.1in
  \lrpicaption{Left: Geodesic approximation error v.s. number of points sampled in the latent space. Right: Geodesic curves generated from the chart decoders.}
  \label{fig:geo}
\end{figure}

\paragraph{Complex Topology.}
Beyond circles and spheres, CAE can handle increasingly more complex manifolds. Figure~\ref{fig:Gen3Results} shows a genus-3 manifold example, which is the surface of a pyramid with three holes. We use ten 2-dimensional charts to cover the entire manifold. The figure illustrates that new points generated by CAE stay close to the data manifold.

\begin{figure}[ht]
  \centering
  \includegraphics[width=0.9\linewidth]{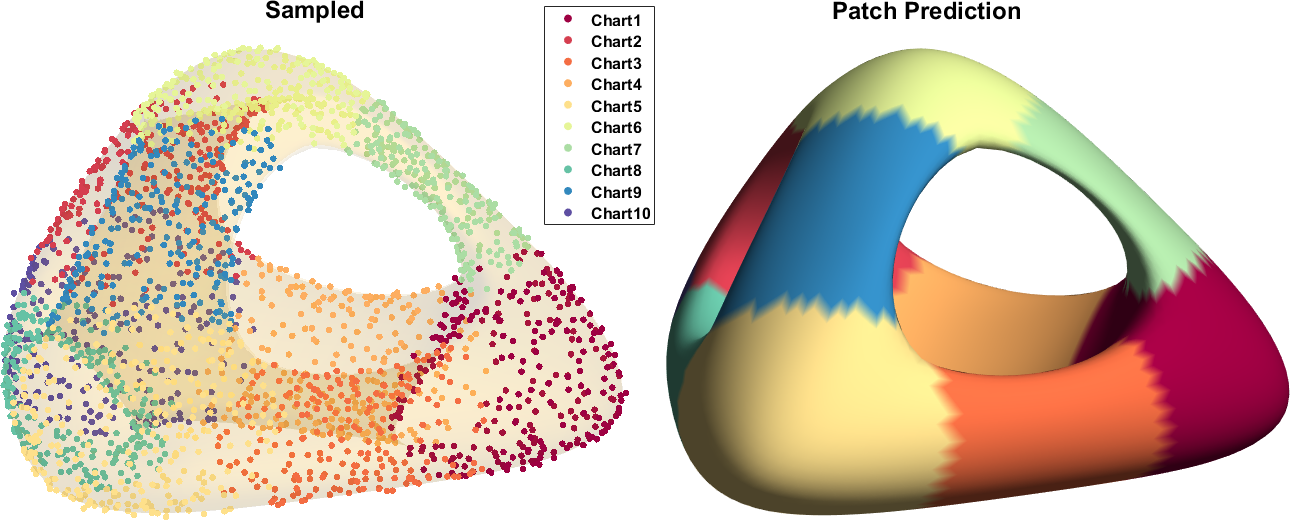}
  \vskip -0.1in
  \lrpicaption{Left: Points sampled from high probability regions. Right: Charts after taking max.}
  \label{fig:Gen3Results}
\end{figure}

\subsection{The MNIST and Fashion MNIST Manifolds}

In this subsection, we train a 4-chart CAE on MNIST and Fashion MNIST and explore the data manifold.

\paragraph{Decoder Outputs.}
Figure~\ref{fig:MNIST} illustrates several decoding results. Each column corresponds to one example. One finds that each chart decoder produces a legible digit, which may or may not coincide with the input. However, the maximum probability always points to the correct digit. Moreover, in some cases several chart decoders produce similar correct results (e.g., `7', `1', `9', and `6'), which indicates that the corresponding charts overlap in a region surrounding this digit.

\begin{figure}[ht]
  \centering
  \includegraphics[width=.85\linewidth]{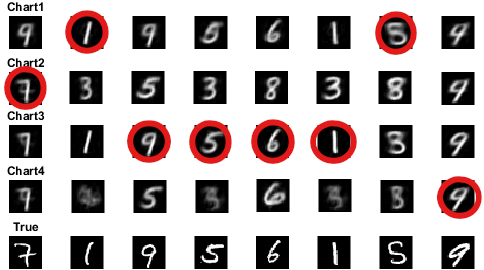}
  \vskip -0.1in
  \lrpicaption{Decoder outputs for a few digit examples. The circled outputs receive the highest probability and serve as the final reconstruction.}
  \label{fig:MNIST}
\end{figure}

\paragraph{Morphing Along the Geodesics.}
As demonstrated earlier, an advantage of CAE is that it is able to trace the geodesics. Here, we use the geodesic path between two data points to generate a morphing sequence between them. We compare such a morphing sequence with the sequence interpolated in the latent space of a VAE. The latter sequence does not necessarily stay on the manifold.

A few examples are illustrated in Figures~\ref{fig:MNIST2}--\ref{fig:FMNIST2}. In all examples, the sequences appear smooth. On MNIST, one sees that while the VAE sequence contains many ``ghost'' images, each of which looks like an overlay of two or more digits, the CAE sequence consists of cleaner digits. The transitions of the digits are also intuitive. The poor quality of VAE interpolation is more apparent on Fashion MNIST. Therein, the interpolated results are blurry overlaid images, as opposed to meaningful objects whose shapes smoothly vary, exhibited by the CAE sequence.

\begin{figure}[ht]
  \centering
  \includegraphics[width=1\linewidth]{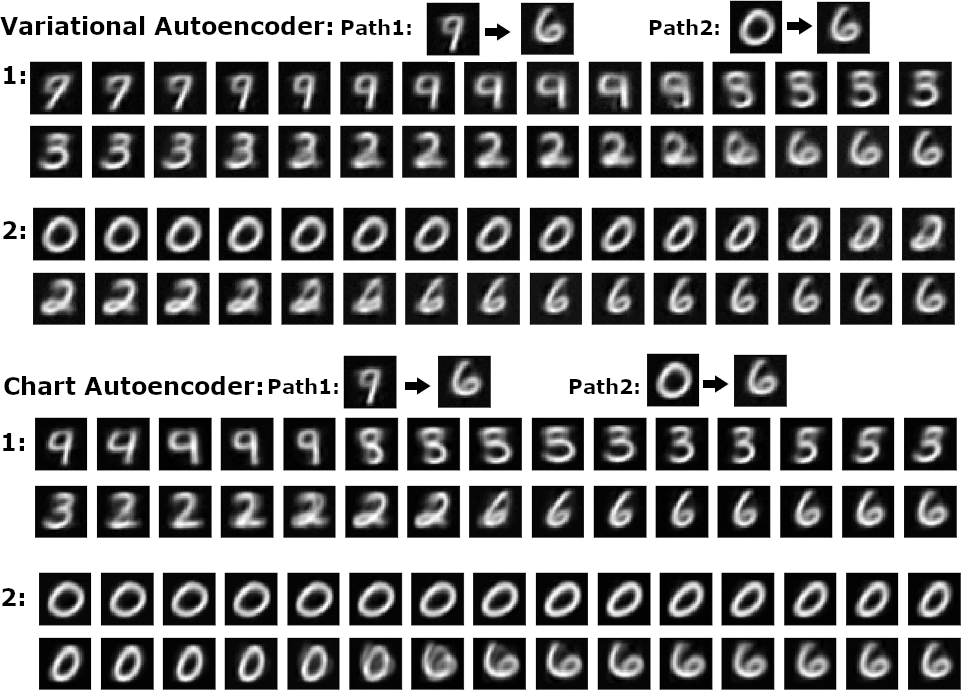}
  \vskip -0.1in
  \lrpicaption{Morphing on MNIST. Top: Morphing obtained by VAE. Bottom: Morphing obtained by the proposed CAE.}
  \label{fig:MNIST2}
\end{figure}

\begin{figure}[ht]
  \centering
  \includegraphics[width=1\linewidth]{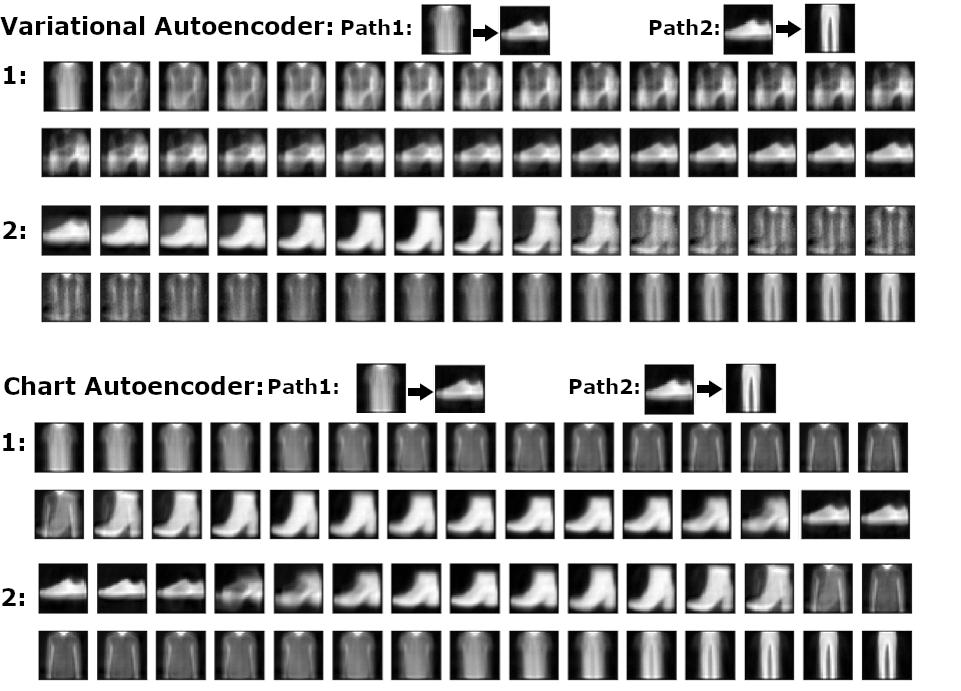}
  \vskip -0.1in
  \lrpicaption{Morphing on Fashion MNIST. Top: Morphing obtained by VAE. Bottom: Morphing obtained by the proposed CAE.}
  \label{fig:FMNIST2}
\end{figure}

\paragraph{Visualization of the Manifolds.}
To understand the manifold structure globally, we visualize each chart in Figures~\ref{fig:MNIST1}--\ref{fig:FMNIST1}. In these plots, we use t-SNE~\cite{maaten2008visualizing} to perform dimension reduction from the charted latent space to two dimensions. For MNIST, one sees that some digits are mostly covered by a single chart (e.g., purple 9) whereas others appear in multiple charts (e.g., navy blue 8). Similar observations are made for Fashion MNIST.

\begin{figure}[ht]
  \centering
  \includegraphics[width=.95\linewidth]{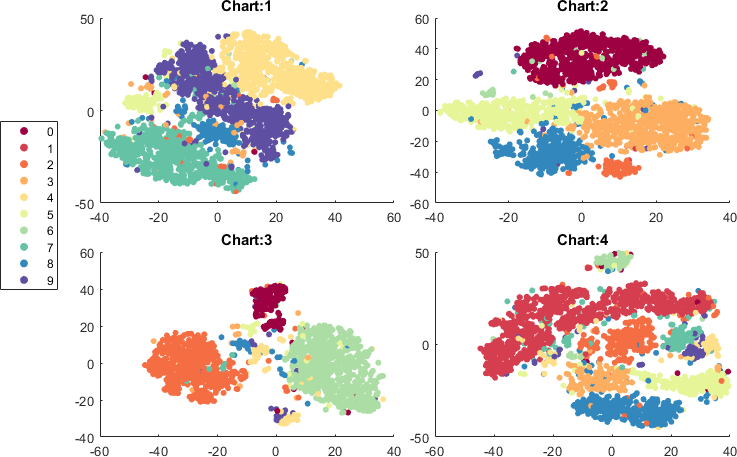}
  \vskip -0.1in
  \lrpicaption{TSNE Visualization of the MNIST manifold by charts. Each color corresponds to a different class from the training data.}
  \label{fig:MNIST1}
\end{figure}

\begin{figure}[ht]
  \centering
  \includegraphics[width=.95\linewidth]{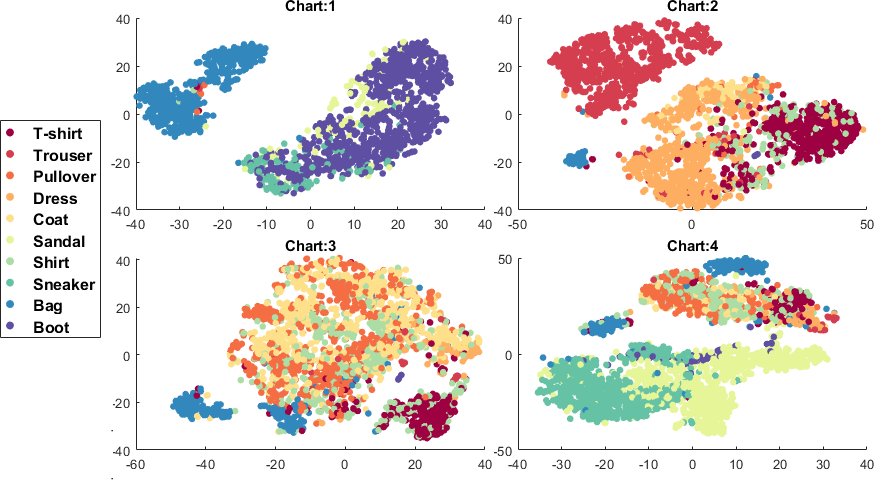}
  \vskip -0.1in
  \lrpicaption{Visualization of the Fashion MNIST manifold by charts.Each color corresponds to a different class from the training data.}
  \label{fig:FMNIST1}
\end{figure}

\subsection{
Model Evaluation}\label{sec:eval}

In addition to the qualitative evaluations so far, we quantitatively evaluate the performance of CAE by comparing it with plain auto-encoders and VAEs. Besides the usual reconstruction error, we define two complementary metrics to comprehensively evaluate models. Both require a uniform sampling in the latent space to make sense. The first one, named \emph{faithfulness}, is the constant one minus distance of a randomly generated sample from the training set. A larger value means closer to the data manifold and hence the model is more faithful to the manifold. The second metric, named \emph{coverage}, is the ratio between the number of distinct nearest training examples and the number of latent space samples. A high coverage is desired because otherwise some training examples (modes) are missed by the latent space. See Supplementary Material (Section~\ref{sec:eval.additional}) for the formal definitions of all three metrics.

We consider three data sets: sphere, MNIST \cite{lecun2010mnist}, and Fashion MNIST \cite{xiao2017fashion}. Because of space limitation, all results are reported in Section~\ref{sec:eval.additional} and here we show only Figure~\ref{fig:SpiderSphere}, which is typical. Each spider chart corresponds to one model class and the last one is an overlay of all. The four axes in each chart are (P) number of parameters, (R) reconstruction error, (F) faithfulness, and (C) Coverage. For all metrics R, F, and C, the farther away from the center, the better. On the other hand, the value of P increases radially and a larger P indicates higher model complexity (in terms of number of parameters).

\begin{figure}[ht]
  \centering
  \includegraphics[width=0.5\linewidth]{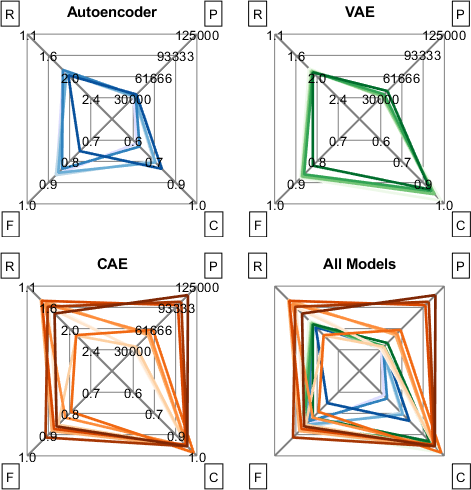}
  \lrpicaption{CAE Model comparison on $S^2$. P: number of parameters; R: reconstruction error; F: faithfulness; C: Coverage.}
  \label{fig:SpiderSphere}
\end{figure}

From the figure, one sees that at the same level model complexity, VAE outperforms plain auto-encoder, while CAE achieves the best results. Curiously, for VAEs and plain auto-encoders, the reconstruction error stays approximately the same regardless of the latent dimension; and only CAEs are able to reduce the reconstruction error through increasing the number of parameters.

\section{Conclusions and Future Work for Chart Based Auto-Encoding} \label{sec:PAEconc}

We have proposed and investigated the use of chart based parameterization to model manifold structured data, through introducing multi-chart latent spaces along with chart transition functions. The parameterization follows the mathematical definition of manifolds and allows one to significantly reduce the dimension of latent encoding. Numerically, we design geometric examples to analyze the behavior of the proposed model and illustrate its advantage over plain auto-encoders and VAEs. We also apply the model to real-life data sets (MNIST and fashion MNIST) to illustrate the manifold structures under-explored by existing auto-encoders.

The proposed chart based parameterization offers many opportunities for further analysis and applications. One interesting avenue is to study manifolds equipped with probability measures, which naturally introduce distributions in the latent space, more similar to VAEs. Another direction is to extend to non-auto-encoder type of generative models (e.g., GAN), which also incur distributional assumptions in the latent space.

\section{VAEs Do Not Generalize for Double Torus}\label{sec:vae.torus}

Figure~\ref{fig:vae} shows an experiment of VAEs with increasingly more parameters on data sampled from a double torus. The latent space dimension is set at two, the intrinsic dimension of the object. One sees that increasing the number of parameters in a VAE alone (without increasing the latent dimension) does not simultaneously produce good reconstruction and generalize. A latent space with a small dimension does not cover the entire manifold and a model with too many parameters overfits (the generated points may be far from the manifold).

\begin{figure}[ht]
  \centering
  \includegraphics[width=.95\linewidth]{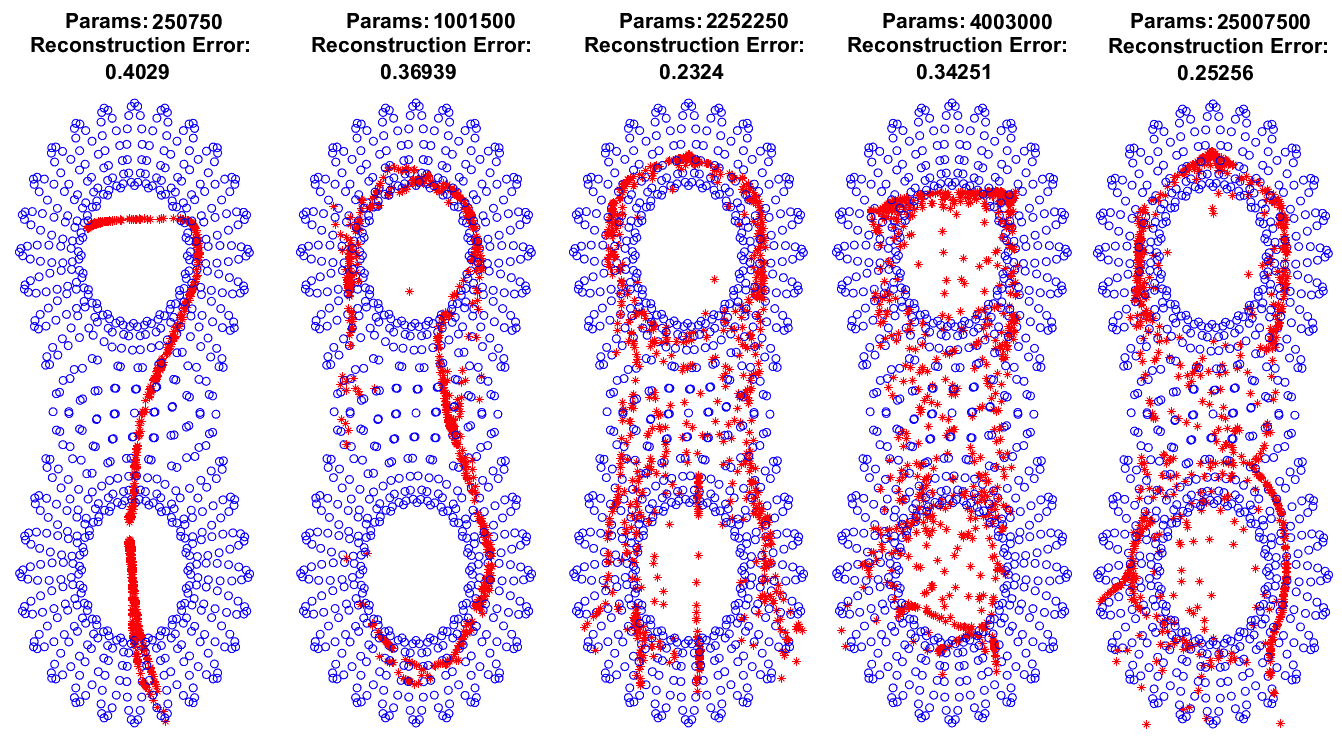}
  \lrpicaption{Increasingly overparametized VAEs with 2-dimensional flat latent space for data sampled from a double torus. Blue: training data. Red: generated data sampled from the latent space.}
  \label{fig:vae}
\end{figure}


\section{Chart Transition Functions}\label{sec:trans}

A key feature of the chart based parameterization in differential geometry is the construction of chart transition functions. As shown in Figure~\ref{fig:ChartManifold}, some points on the manifold may be parameterized by multiple charts. Let $\phi_{\alpha}$ and $\phi_{\beta}$ be two chart functions with chart overlap $M_{\alpha} \cap M_{\beta} \neq \emptyset$; then, the chart transition function $\phi_{\alpha \beta} = \phi_{\beta} \phi^{-1}_{\alpha}$.

In our model, the chart decoder $\textbf{D}_{\alpha}$ plays the role of $\phi_{\alpha}^{-1}$ and the composition $\textbf{E}_\beta \circ \textbf{E}$ plays the role of $\phi_{\beta}$. Hence, the chart transition function can be modeled by the composition:
\begin{equation}
  \phi_{\alpha \beta} :U_\alpha\cap U_\beta \rightarrow U_\beta\cap U_\alpha, \quad z_\alpha \mapsto  \textbf{E}_\beta \Big( \textbf{E} \big( \textbf{D}_\alpha (z_\alpha) \big) \Big).
\end{equation}
Note that if $x \in \mathcal{M}_{\alpha} \cap \mathcal{M}_{\beta}$, then to obtain a high-quality transition function we need:
\begin{itemize}
    \item $p_{\alpha}(x) \approx p_{\beta}(x)$ 
    \item  $ x \approx \textbf{D}_\alpha ( \textbf{E}_ \alpha( \textbf{E}(x) )) $
    \item  $ x \approx \textbf{D}_\beta ( \textbf{E}_\beta ( \textbf{E}(x) )) $.
\end{itemize}
Each of these conditions are naturally met if the loss function~\eqref{eqn:LossPP} is well minimized. To gauge the accuracy of such transition functions, one may re-encode the decoded data in a second pass:
\begin{equation}
    \begin{split}
  \mathcal{R}_{cycle}(x):= &
  \|x - \textbf{D}_\beta \circ \textbf{E}_\beta \circ \textbf{E} \circ \textbf{D}_\alpha \circ \textbf{E}_\alpha \circ \textbf{E} (x) \| \\
  & + \|x - \textbf{D}_\alpha \circ \textbf{E}_\alpha \circ \textbf{E} \circ \textbf{D}_\beta \circ \textbf{E}_\beta\circ \textbf{E} (x)\|.   
\end{split}
\end{equation}
The residual $\mathcal{R}_{cycle}(x)$ measures the error in chart transition and reconstruction.

\section{Chart Orientation}\label{sec:PCA}

We can extend pre-training to additionally orient all charts, whose centers are denoted by $c_{\alpha}$. To do so, we take a small sample of points $\mathcal{N}(c_\alpha)$ around the center and use principal component analysis (PCA) to define a $d$-dimensional embedding of this local neighborhood. Let the embeddings be $\hat{x}_\alpha(x) :=\frac{1}{C_\alpha} W_\alpha x + b_\alpha$ for all $x \in \mathcal{N}(c_\alpha)$, where $W_\alpha$ is the optimal orthogonal projection from $U_\alpha$ to $\RR^d$, $b_\alpha$ is used to shift $\hat{x}_\alpha(c_\alpha)$ to $[.5]^d$, and $C_i$ is chosen as a local scaling constant. Then, we can use this coordinate system to initialize the chart orientations by minimizing an additional regularization:
\begin{equation}\label{eqn:PCA}
  \mathcal{R}_{cords} =  \sum_{\alpha=1}^N\sum_{x \in \mathcal{N}(c_\alpha)} \langle \textbf{E}_\alpha \circ \textbf{E}(x) , \hat x_\alpha(x) \rangle.
\end{equation}

\section{Additional Results on Model Evaluation}\label{sec:eval.additional}

Here, we report the numerical results for all models, data sets, and metrics mentioned in Section~\ref{sec:eval} of the main text. See Figures~\ref{fig:SpiderMNIST}--\ref{fig:SpiderFMNIST} and Tables~\ref{tab:evaluation.S2}--\ref{tab:evaluation.FMNIST}.

The evaluation metrics are defined in the following.

\paragraph{Reconstruction Error}
Let $x$ be a data point in the test set $D_{test}$ and $y(x)$ be its reconstruction. Let there be $N$ test points. The reconstruction error is
\begin{equation}\label{eqn:recon}
\mathcal{E}_{recon}:= \frac{1}{N}\sum_{x \in D_{Test}} ||x-y||^2 .
\end{equation}

\paragraph{Faithfulness}
Let $\{z_i\}_{i=1}^\ell$ be a uniform sampling in the latent space and $\textbf{D}$ denote the decoder. Let $D_{test}$ be the training set. The faithfulness is
\begin{equation}\label{eqn:faithful}
\mathcal{E}_{faithful}=1 - \left( \frac{1}{\ell} \sum_{i=1}^\ell \min_{x \in D_{train} } \|x - \textbf{D} (z_i)\|^2\right) .
\end{equation}
We set $\ell=100$. The concept of faithfulness is complementary to the concept of novelty in deep generative models. Whereas novel samples are encouraged, this metric is concerned with how close the novel sample stays to the manifold. When the training set is sufficiently dense on the data manifold, newly generated data faraway from anything observed during training are unlikely to be realistic.

\paragraph{Coverage}
Let $\ell^*$ be the cardinality of the set
\begin{equation}
\{x^*~|~ x^* = \arg \min_{x \in D_{train} } \|x - \textbf{D} (z_i)\|^2 \}.
\end{equation}
Then, we define the coverage
\begin{equation}\label{eqn:coverage}
\mathcal{E}_{coverage} = \frac{\ell^*}{\ell}.
\end{equation}
A coverage score close to 1 indicates that the newly generated samples are well distributed on the manifold, whereas a score close to 0 indicates that the model may be experiencing mode collapse.

\begin{figure*}[ht]
  \centering
  \includegraphics[width=0.95\linewidth]{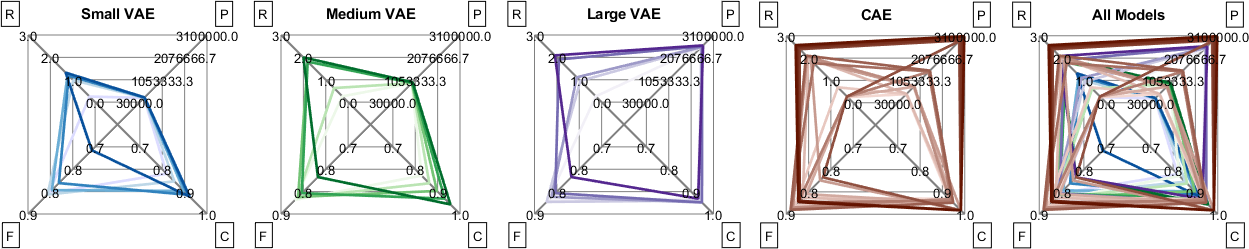}
  \lrpicaption{CAE Model comparison on MNIST. P: number of parameters; R: reconstruction error; F: faithfulness; C: Coverage.}
  \label{fig:SpiderMNIST}
\end{figure*}

\begin{figure*}[ht]
  \centering
  \includegraphics[width=0.95\linewidth]{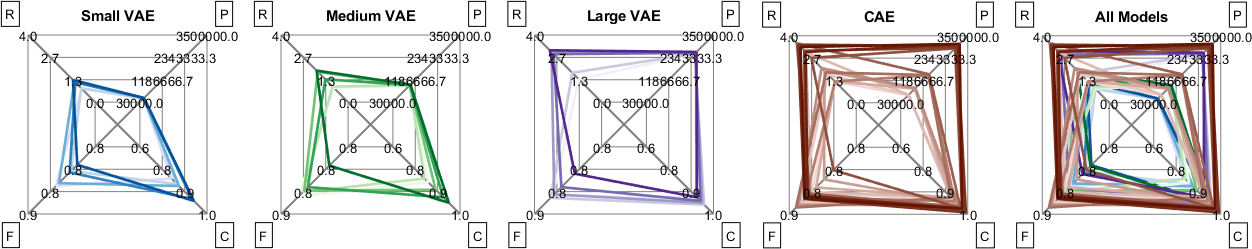}
  \lrpicaption{CAE Model comparison on Fashion MNIST. P: number of parameters; R: reconstruction error; F: faithfulness; C: Coverage.}
  \label{fig:SpiderFMNIST}
\end{figure*}

\begin{table*}[ht]
\begin{center}
\lrpicaption{Model comparison on $S^2$.}
\hspace{8pt}
\label{tab:evaluation.S2}
\resizebox{\textwidth}{!}{
\begin{tabular}{|c|c|c|c|c|c|c|}
\hline
\textbf{}                                                           & \textbf{\begin{tabular}[c]{@{}c@{}}\# of \\ Charts\end{tabular}}          & \textbf{\begin{tabular}[c]{@{}c@{}}Dim \\ of Charts\end{tabular}}         & \textbf{\begin{tabular}[c]{@{}c@{}}\# of \\  Param.\end{tabular}}                                               & \textbf{\begin{tabular}[c]{@{}c@{}}Recon.\\  Error\end{tabular}}                                                                                                                                                                   & \textbf{Faithfulness}                                                                                                                                                                                            & \textbf{Coverage}                                                                                                                                                                               \\ \hline
Auto-Encoder                                                        & \begin{tabular}[c]{@{}c@{}}1\\ 1\\ 1\\ 1\\ 1\end{tabular}                 & \begin{tabular}[c]{@{}c@{}}2\\ 4\\ 8\\ 16\\ 32\end{tabular}               & \begin{tabular}[c]{@{}c@{}}30850\\ 31250\\ 32050\\ 33650\\ 34950\end{tabular}                                   & \begin{tabular}[c]{@{}c@{}}0.0174 $\pm$ .0001\\ 0.0180 $\pm$ .0006\\ 0.0180 $\pm$ .0007\\ 0.0173 $\pm$ .0011\\ 0.0184 $\pm$ .0002\end{tabular}                                                                                     & \begin{tabular}[c]{@{}c@{}}0.838 $\pm$ .021\\ 0.842 $\pm$ .032\\ 0.829 $\pm$ .021\\ 0.808 $\pm$ .041\\ 0.710 $\pm$ .034\end{tabular}                                                                             & \begin{tabular}[c]{@{}c@{}}0.62 $\pm$ .01\\ 0.73 $\pm$ .01\\ 0.74 $\pm$ .01\\ 0.64 $\pm$ .02\\ 0.78 $\pm$ .01\end{tabular}                                                                      \\ \hline
\begin{tabular}[c]{@{}c@{}}Variational \\ Auto-Encoder\end{tabular} & \begin{tabular}[c]{@{}c@{}}1\\ 1\\ 1\\ 1\\ 1\end{tabular}                 & \begin{tabular}[c]{@{}c@{}}2\\ 4\\ 8\\ 16\\ 32\end{tabular}               & \begin{tabular}[c]{@{}c@{}}31052\\ 31654\\ 32858\\ 35266\\ 40082\end{tabular}                                   & \begin{tabular}[c]{@{}c@{}}0.0178 $\pm$  .0003\\ 0.0174 $\pm$  .0001\\ 0.0176 $\pm$  .0003\\ 0.0178 $\pm$  .0002\\ 0.0178 $\pm$  .0004\end{tabular}                                                                                & \begin{tabular}[c]{@{}c@{}}0.858 $\pm$ .021\\ 0.852 $\pm$ .016\\ 0.849 $\pm$ .021\\ 0.838 $\pm$ .023\\ 0.790 $\pm$ .011\end{tabular}                                                                             & \begin{tabular}[c]{@{}c@{}}0.98 $\pm$ .02\\ 0.93 $\pm$ .01\\ 0.94 $\pm$ .01\\ 0.92 $\pm$ .02\\ 0.91 $\pm$ .01\end{tabular}                                                                      \\ \hline
\begin{tabular}[c]{@{}c@{}}CAE\end{tabular}       & \begin{tabular}[c]{@{}c@{}}2\\ 2\\ 2\\ 4\\ 4\\ 4\\ 8\\ 8\\ 8\end{tabular} & \begin{tabular}[c]{@{}c@{}}2\\ 4\\ 8\\ 2\\ 4\\ 8\\ 2\\ 4\\ 8\end{tabular} & \begin{tabular}[c]{@{}c@{}}30533\\ 31891\\ 34607\\ 51991\\ 54557\\ 59689\\ 96707\\ 101689\\ 111653\end{tabular} & \begin{tabular}[c]{@{}c@{}}0.0130 $\pm$ .0002\\ 0.0144 $\pm$ .0002\\ 0.0201 $\pm$ .0007\\ 0.0128 $\pm$ .0001\\ 0.0136 $\pm$ .0002\\ 0.0200 $\pm$ .0003\\ 0.0130 $\pm$ .0001\\ 0.0142 $\pm$ .0001\\ 0.0156 $\pm$ .0004\end{tabular} & \begin{tabular}[c]{@{}c@{}}0.848 $\pm$ .021\\ 0.832 $\pm$ .016\\ 0.829 $\pm$ .061\\ 0.864 $\pm$ .021\\ 0.790 $\pm$ .011\\ 0.738 $\pm$ .021\\ 0.853 $\pm$ .016\\ 0.899 $\pm$ .011\\ 0.835 $\pm$ .021\end{tabular} & \begin{tabular}[c]{@{}c@{}}0.94 $\pm$ .01\\ 0.94 $\pm$ .01\\ 0.94 $\pm$ .01\\ 0.94 $\pm$ .01\\ 0.98 $\pm$ .01\\ 0.94 $\pm$ .01\\ 0.94 $\pm$ .01\\ 0.94 $\pm$ .01 \\ 0.94 $\pm$ .01\end{tabular} \\ \hline
\end{tabular}
}
\end{center}
\end{table*}

\begin{table*}[ht]
\lrpicaption{Model comparison on MNIST.}
\begin{center}
\resizebox{\textwidth}{!}{
\begin{tabular}{|c|c|c|c|c|c|c|}
\hline
\textbf{\begin{tabular}[c]{@{}c@{}}Model\\ (Latent Space)\end{tabular}}     & \textbf{\begin{tabular}[c]{@{}c@{}}\# of \\ Charts\end{tabular}}                                      & \textbf{\begin{tabular}[c]{@{}c@{}}Dim \\ of Charts\end{tabular}}                                   & \textbf{\begin{tabular}[c]{@{}c@{}}\# of \\  Param.\end{tabular}}                                                                                                           & \textbf{\begin{tabular}[c]{@{}c@{}}Recon.\\  Error\end{tabular}}                                                                                                                                                                                                                                                              & \textbf{Faithfulness}                                                                                                                                                                                                                                                                                           & \textbf{Coverage}                                                                                                                                                                                                                                                                   \\ \hline
\begin{tabular}[c]{@{}c@{}}Small\\ Variational \\ Auto-Encoder\end{tabular} & \begin{tabular}[c]{@{}c@{}}1\\ 1\\ 1\\ 1\\ 1\end{tabular}                                             & \begin{tabular}[c]{@{}c@{}}4\\ 8\\ 16\\ 32\\ 64\end{tabular}                                        & \begin{tabular}[c]{@{}c@{}}221568\\ 228632\\ 231040\\ 235856\\ 245488\end{tabular}                                                                                          & \begin{tabular}[c]{@{}c@{}}0.0675 $\pm$ .000\\ 0.0602 $\pm$ .001\\ 0.0577 $\pm$ .003\\ 0.0582 $\pm$ .001\\ 0.0568 $\pm$ .001\end{tabular}                                                                                                                                                                                     & \begin{tabular}[c]{@{}c@{}}0.838 $\pm$ .021\\ 0.842 $\pm$ .032\\ 0.829 $\pm$ .021\\ 0.806 $\pm$ .041\\ 0.711 $\pm$ .034\end{tabular}                                                                                                                                                                            & \begin{tabular}[c]{@{}c@{}}0.79 $\pm$ .01\\ 0.83 $\pm$ .01\\ 0.88 $\pm$ .01\\ 0.90 $\pm$ .02\\ 0.91 $\pm$ .01\end{tabular}                                                                                                                                                          \\ \hline
\begin{tabular}[c]{@{}c@{}}Medium\\ VAE\end{tabular}                        & \begin{tabular}[c]{@{}c@{}}1\\ 1\\ 1\\ 1\\ 1\end{tabular}                                             & \begin{tabular}[c]{@{}c@{}}4\\ 8\\ 16\\ 32\\ 64\end{tabular}                                        & \multicolumn{1}{l|}{\begin{tabular}[c]{@{}l@{}}893028\\ 896032\\ 902040\\ 914056\\ 938088\end{tabular}}                                                                     & \multicolumn{1}{l|}{\begin{tabular}[c]{@{}l@{}}0.0674 $\pm$ .001\\ 0.0637 $\pm$ .002\\ 0.0519 $\pm$ .001\\ 0.0519 $\pm$ .002\\ 0.0500 $\pm$ .003\end{tabular}}                                                                                                                                                                & \multicolumn{1}{l|}{\begin{tabular}[c]{@{}l@{}}0.858 $\pm$ .021\\ 0.852 $\pm$ .016\\ 0.849 $\pm$ .021\\ 0.838 $\pm$ .023\\ 0.790 $\pm$ .011\end{tabular}}                                                                                                                                                       & \multicolumn{1}{l|}{\begin{tabular}[c]{@{}l@{}}0.80 $\pm$ .01\\ 0.84 $\pm$ .02\\ 0.88 $\pm$ .02\\ 0.92 $\pm$ .01\\ 0.95 $\pm$ .01\end{tabular}}                                                                                                                                     \\ \hline
\begin{tabular}[c]{@{}c@{}}Large\\ VAE\end{tabular}                         & \begin{tabular}[c]{@{}c@{}}1\\ 1\\ 1\\ 1\\ 1\end{tabular}                                             & \begin{tabular}[c]{@{}c@{}}4\\ 8\\ 16\\ 32\\ 64\end{tabular}                                        & \multicolumn{1}{l|}{\begin{tabular}[c]{@{}l@{}}2535028\\ 2541032\\ 2553040\\ 2577056\\ 2625088\end{tabular}}                                                                & \multicolumn{1}{l|}{\begin{tabular}[c]{@{}l@{}}0.0674 $\pm$ .001\\ 0.0605 $\pm$ .000\\ 0.0589 $\pm$ .000\\ 0.0509 $\pm$ .000\\ 0.0491 $\pm$ .007\end{tabular}}                                                                                                                                                                & \multicolumn{1}{l|}{\begin{tabular}[c]{@{}l@{}}0.860 $\pm$ .008\\ 0.864 $\pm$ .011\\ 0.849 $\pm$ .016\\ 0.838 $\pm$ .017\\ 0.893 $\pm$ .011\end{tabular}}                                                                                                                                                       & \multicolumn{1}{l|}{\begin{tabular}[c]{@{}l@{}}0.92 $\pm$ .00\\ 0.93 $\pm$ .02\\ 0.94 $\pm$ .01\\ 0.94 $\pm$ .02\\ 0.92 $\pm$ .01\end{tabular}}                                                                                                                                     \\ \hline
CAE                                                                & \begin{tabular}[c]{@{}c@{}}4\\ 4\\ 4\\ 8\\ 8\\ 8\\ 16\\ 16\\ 16\\ 32\\ 32\\ 32\\ 32\\ 32\end{tabular} & \begin{tabular}[c]{@{}c@{}}4\\ 8\\ 16\\ 4\\ 8\\ 16\\ 4\\ 8\\ 16\\ 4\\ 8\\ 16\\ 32\\ 64\end{tabular} & \begin{tabular}[c]{@{}c@{}}419807\\ 424939\\ 435203\\ 759139\\ 769103\\ 789031\\ 1445003\\ 1464631\\ 1503887\\ 2845531\\ 2884487\\ 2962399\\ 3072399\\ 3080638\end{tabular} & \begin{tabular}[c]{@{}c@{}}0.0675 $\pm$ .000\\ 0.0631 $\pm$ .006\\ 0.0499 $\pm$ .002\\ 0.0672 $\pm$ .003\\ 0.0511 $\pm$ .002\\ 0.0493 $\pm$ .003\\ 0.0673 $\pm$ .001\\ 0.0523 $\pm$ .002\\ 0.0489 $\pm$ .001\\ 0.0672 $\pm$ .002\\ 0.0460 $\pm$ .001\\ 0.0453 $\pm$ .003\\ 0.0447 $\pm$ .001\\ 0.0431 $\pm$ .002\end{tabular} & \begin{tabular}[c]{@{}c@{}}0.850 $\pm$ .008\\ 0.874 $\pm$ .006\\ 0.851 $\pm$ .016\\ 0.840 $\pm$ .012\\ 0.799 $\pm$ .013\\ 0.880 $\pm$ .021\\ 0.864 $\pm$ .014\\ 0.886 $\pm$ .016\\ 0.867 $\pm$ .015\\ 0.790 $\pm$ .011\\ 0.760 $\pm$ .008\\ 0.864 $\pm$ .008\\ 0.860 $\pm$ .008\\ 0.863 $\pm$ .009\end{tabular} & \begin{tabular}[c]{@{}c@{}}0.92 $\pm$ .01\\ 0.93 $\pm$ .01\\ 0.94 $\pm$ .01\\ 0.94 $\pm$ .02\\ 0.98 $\pm$ .01\\ 0.92 $\pm$ .01\\ 0.93 $\pm$ .01\\ 0.94 $\pm$ .01\\ 0.94 $\pm$ .02\\ 0.98 $\pm$ .01\\ 0.98 $\pm$ .01\\ 0.98 $\pm$ .02\\ 0.98 $\pm$ .01\\ 0.98 $\pm$ .02\end{tabular} \\ \hline
\end{tabular}
}
\label{tab:evaluation.MNIST}
\end{center}
\end{table*}

\begin{table*}[ht]
\lrpicaption{Model comparison on Fashion MNIST.}
\begin{center}
\label{tab:evaluation.FMNIST}
\resizebox{\textwidth}{!}{
\begin{tabular}{|c|c|c|c|c|c|c|}
\hline
\textbf{\begin{tabular}[c]{@{}c@{}}Model\\ (Latent Space)\end{tabular}}     & \textbf{\begin{tabular}[c]{@{}c@{}}\# of \\ Charts\end{tabular}}                                      & \textbf{\begin{tabular}[c]{@{}c@{}}Dim \\ of Charts\end{tabular}}                                   & \textbf{\begin{tabular}[c]{@{}c@{}}\# of \\  Param.\end{tabular}}                                                                                                           & \textbf{\begin{tabular}[c]{@{}c@{}}Recon.\\  Error\end{tabular}}                                                                                                                                                                                                                                                              & \textbf{Faithfulness}                                                                                                                                                                                                                                                                                           & \textbf{Coverage}                                                                                                                                                                                                                                                                   \\ \hline
\begin{tabular}[c]{@{}c@{}}Small\\ Variational \\ Auto-Encoder\end{tabular} & \begin{tabular}[c]{@{}c@{}}1\\ 1\\ 1\\ 1\\ 1\end{tabular}                                             & \begin{tabular}[c]{@{}c@{}}4\\ 8\\ 16\\ 32\\ 64\end{tabular}                                        & \begin{tabular}[c]{@{}c@{}}221568\\ 228632\\ 231040\\ 235856\\ 245488\end{tabular}                                                                                          & \begin{tabular}[c]{@{}c@{}}0.0619 $\pm$ .003\\ 0.0617 $\pm$ .001\\ 0.0577 $\pm$ .002\\ 0.0582 $\pm$ .001\\ 0.0568 $\pm$ .001\end{tabular}                                                                                                                                                                                     & \begin{tabular}[c]{@{}c@{}}0.837 $\pm$ .021\\ 0.839 $\pm$ .032\\ 0.831 $\pm$ .021\\ 0.807 $\pm$ .041\\ 0.788 $\pm$ .034\end{tabular}                                                                                                                                                                            & \begin{tabular}[c]{@{}c@{}}0.81 $\pm$ .01\\ 0.84 $\pm$ .01\\ 0.85 $\pm$ .01\\ 0.92 $\pm$ .02\\ 0.93 $\pm$ .01\end{tabular}                                                                                                                                                          \\ \hline
\begin{tabular}[c]{@{}c@{}}Medium\\ VAE\end{tabular}                        & \begin{tabular}[c]{@{}c@{}}1\\ 1\\ 1\\ 1\\ 1\end{tabular}                                             & \begin{tabular}[c]{@{}c@{}}4\\ 8\\ 16\\ 32\\ 64\end{tabular}                                        & \multicolumn{1}{l|}{\begin{tabular}[c]{@{}l@{}}893028\\ 896032\\ 902040\\ 914056\\ 938088\end{tabular}}                                                                     & \multicolumn{1}{l|}{\begin{tabular}[c]{@{}l@{}}0.0614 $\pm$ .002\\ 0.0607 $\pm$ .003\\ 0.0519 $\pm$ .001\\ 0.0564 $\pm$ .003\\ 0.0512 $\pm$ .002\end{tabular}}                                                                                                                                                                & \multicolumn{1}{l|}{\begin{tabular}[c]{@{}l@{}}0.831 $\pm$ .021\\ 0.851 $\pm$ .016\\ 0.848 $\pm$ .021\\ 0.840 $\pm$ .023\\ 0.702 $\pm$ .011\end{tabular}}                                                                                                                                                       & \multicolumn{1}{l|}{\begin{tabular}[c]{@{}l@{}}0.83 $\pm$ .01\\ 0.81 $\pm$ .02\\ 0.87 $\pm$ .02\\ 0.91 $\pm$ .01\\ 0.94 $\pm$ .01\end{tabular}}                                                                                                                                     \\ \hline
\begin{tabular}[c]{@{}c@{}}Large\\ VAE\end{tabular}                         & \begin{tabular}[c]{@{}c@{}}1\\ 1\\ 1\\ 1\\ 1\end{tabular}                                             & \begin{tabular}[c]{@{}c@{}}4\\ 8\\ 16\\ 32\\ 64\end{tabular}                                        & \multicolumn{1}{l|}{\begin{tabular}[c]{@{}l@{}}2535028\\ 2541032\\ 2553040\\ 2577056\\ 2625088\end{tabular}}                                                                & \multicolumn{1}{l|}{\begin{tabular}[c]{@{}l@{}}0.0564 $\pm$ .001\\ 0.0525 $\pm$ .002\\ 0.0401 $\pm$ .001\\ 0.0414 $\pm$ .001\\ 0.0391 $\pm$ .002\end{tabular}}                                                                                                                                                                & \multicolumn{1}{l|}{\begin{tabular}[c]{@{}l@{}}0.859 $\pm$ .008\\ 0.864 $\pm$ .011\\ 0.856 $\pm$ .016\\ 0.840 $\pm$ .017\\ 0.810 $\pm$ .011\end{tabular}}                                                                                                                                                       & \multicolumn{1}{l|}{\begin{tabular}[c]{@{}l@{}}0.91 $\pm$ .00\\ 0.95 $\pm$ .02\\ 0.94 $\pm$ .01\\ 0.93 $\pm$ .02\\ 0.91 $\pm$ .01\end{tabular}}                                                                                                                                     \\ \hline
CAE                                                              & \begin{tabular}[c]{@{}c@{}}4\\ 4\\ 4\\ 8\\ 8\\ 8\\ 16\\ 16\\ 16\\ 32\\ 32\\ 32\\ 32\\ 32\end{tabular} & \begin{tabular}[c]{@{}c@{}}4\\ 8\\ 16\\ 4\\ 8\\ 16\\ 4\\ 8\\ 16\\ 4\\ 8\\ 16\\ 32\\ 64\end{tabular} & \begin{tabular}[c]{@{}c@{}}419807\\ 424939\\ 435203\\ 759139\\ 769103\\ 789031\\ 1445003\\ 1464631\\ 1503887\\ 2845531\\ 2884487\\ 2962399\\ 3072399\\ 3080638\end{tabular} & \begin{tabular}[c]{@{}c@{}}0.0575 $\pm$ .003\\ 0.0531 $\pm$ .002\\ 0.0399 $\pm$ .002\\ 0.0572 $\pm$ .002\\ 0.0511 $\pm$ .003\\ 0.0493 $\pm$ .002\\ 0.0573 $\pm$ .002\\ 0.0519 $\pm$ .001\\ 0.0345 $\pm$ .002\\ 0.0432 $\pm$ .001\\ 0.0390 $\pm$ .003\\ 0.0353 $\pm$ .002\\ 0.0397 $\pm$ .002\\ 0.0367 $\pm$ .003\end{tabular} & \begin{tabular}[c]{@{}c@{}}0.850 $\pm$ .009\\ 0.874 $\pm$ .007\\ 0.851 $\pm$ .015\\ 0.840 $\pm$ .011\\ 0.822 $\pm$ .012\\ 0.879 $\pm$ .019\\ 0.863 $\pm$ .018\\ 0.885 $\pm$ .016\\ 0.854 $\pm$ .018\\ 0.801 $\pm$ .014\\ 0.864 $\pm$ .009\\ 0.866 $\pm$ .017\\ 0.859 $\pm$ .011\\ 0.865 $\pm$ .010\end{tabular} & \begin{tabular}[c]{@{}c@{}}0.93 $\pm$ .01\\ 0.92 $\pm$ .01\\ 0.91 $\pm$ .01\\ 0.91 $\pm$ .02\\ 0.94 $\pm$ .01\\ 0.95 $\pm$ .01\\ 0.93 $\pm$ .01\\ 0.91 $\pm$ .01\\ 0.96 $\pm$ .02\\ 0.97 $\pm$ .01\\ 0.98 $\pm$ .01\\ 0.99 $\pm$ .02\\ 0.97 $\pm$ .01\\ 0.97 $\pm$ .02\end{tabular} \\ \hline
\end{tabular}
}
\end{center}
\end{table*}


 
\chapter{SIMPLICAL APPROXIMATION OF DATA MANIFOLDS} \label{Chap:Approx}

Most provable NN approximation papers \cite{csaji2001approximation, gyorfi2006distribution, shaham2018provable, chen2019efficient} for low-dimensional manifolds deal with approximating some function $f: \M \rightarrow \RR$ supported on or near some smooth $d$-dimensional manifold isometrically embedded in $\RR^D$ where $d \ll D$. This setup models tasks such as recognition, classification or segmentation of data. However, when dealing with generative models such as auto-encoders and GANs, it is more interesting to see how well a model can actually represent the manifold $\M$ given some training data $X = \{x\}_{i=1}^n$ sampled form $\M$. In this chapter we construct a deep relu-neural network that behaves like a simplectic approximation (high dimensional mesh) and show that such a network can also represent important topological and geometric information in the so-called latent space.. Recently, \cite{khrulkov2019universality} showed a version of universal approximation for data distributed on compact manifolds, but do not provide any bounds for the size of the networks or their approximation quality. In this chapter we construct a deep relu-neural network that behaves like a simplectic approximation (high dimensional mesh) and show that such a network can also represent important topological and geometric information in the so-called latent space. In this chapter, we rigorously address the topology and geometry approximation behaviors in auto-encoders. We show that topology preservation in auto-encoders is a necessary condition to approximate data manifold $\epsilon$-closely. Moreover, we study a universal manifold approximation theorem based on multi-chart parameterization and provide estimations of training data size and network size.

\blfootnote{Portions of this chapter have been submitted as S. C. SCHONSHECK, J. CHEN, AND R. LAI (2020), \textit{Chart Auto-Encoders for Manifold Structured Data}, NeuRIPS 2020.}

\section{Faithful Representation}

Mathematically, we denote an auto-encoder of $\M\subset\real^m$ by a 3-tuple $(\Z; \E,\D)$. Here, $\Z = \E(\M)$ represents the latent space; $\E$, representing the encoder, continuously maps $\M$ to $\Z$; and $\D$, representing the decoder, continuously generates a data point $\D(z)\in\real^m$ from a latent variable $z\in\Z$. A plain auto-encoder refers to an auto-encoder whose latent space is a simply connected Euclidean domain. 

\paragraph{Faithful representation} To quantitatively measure an auto-encoder, we introduce the following concept to characterize it. 

\begin{definition}[Faithful Representation]
An auto-encoder $(\Z; \E,\D)$ is called a faithful representation of $\M$ if $x = \D\circ\E(x), \forall x\in\M$. An auto-encoder is called an $\epsilon$-faithful representation of $\M$ if $\displaystyle  \sup_{x\in\M} \|x - \D\circ\E(x)\| \leq \epsilon$. 
\end{definition}
To characterize manifolds and quality of auto-encoders in terms of their intrinsic geometry, we reiterate the concept \emph{reach} of a manifold~\cite{federer1959curvature}. The reach of a manifold is essentially the size of the maximum unique tubular neighborhood around the manifold. More formally, given a $d$-dimensional compact data manifold $\M\subset \real^m$, let $\displaystyle \mathcal{G} = \Big \{y\in \real^m~|~\exists p\neq q\in\M \text{~satisfying~} \|y - p\| = \|y- q\| = \inf_{x\in\M}\|x - y\| \Big\} $. The reach of $\M$ is defined as $\displaystyle \tau (\M)= \inf_{x\in\M,y\in\mathcal{G}}\|x-y\|$. With this in mind we can state our first important theorem.

\begin{theorem}
\label{thm:faithfulrep} Let $\M$ be a $d$-dimensional compact manifold. 
If an auto-encoder $(\Z; \E,\D)$ of $\M$ is an  $\epsilon$-faithful representation with $\epsilon < \tau(\M)$, then $\Z$ and $\D(\Z)$ must be homeomorphic to $\M$. Particularly, a $d$-dimensional compact manifold with non-contractible topology can not be $\epsilon$-faithfully represented by a plain auto-encoder with a  latent space $\Z$ being a $d$-dimensional simply connected domain in $\real^d$.
\end{theorem}

This theorem provides a necessary condition of the latent space topology for a faithful representation. It implies that a data manifold with complex topology can not be $\epsilon$-faithfully represented by an auto-encoder with a simply connected latent space like $\real^d$ used in plain auto-encoders. For example, a plain auto-encoder with a single 2 dimensional latent space cannot $\epsilon$-faithfully represent a sphere. 


\section{Main Theoretical Results}

To address the issue of topology violation in plain auto-encoders,  we propose a multi-chart model based on the definition of manifolds. We discuss the main results for approximating data manifolds using multi-chart auto-encoders with training data size and network size estimation. These results motivate the proposed CAE architecture as a generalization of vanilla auto-encoder \ref{Chap:CAE}. We defer all detailed proof of all statements to the section \ref{sec:proofs}.

A training data $X =\{x_i\}_{i=1}^n\subset\M$ is called $\delta$-dense in $\M$ if $\text{dist}(X,p) = \min_{x\in X}\|x - p\| <\delta, \forall p\in\M$. We write $B^d_r$ a $d$-dimensional radius $r$ ball with $\displaystyle vol(B^d_r)  = \pi^{d/2}r^d/\Gamma(1+ d/2)$.

\begin{theorem}[Universal Manifold Approximation Theorem] Consider a d-dimensional compact data manifold $\M\subset \mathbb{R}^m$ with reach $\tau$ and $\displaystyle C = vol(\M)/vol(B^d_1)$. Let $X=\{x\}_{i=1}^n$ be a training data set drawn uniformly randomly on $\M$. For any $0< \epsilon <\tau/2$, if the cardinality of the training set $X$ satisfies
\begin{equation}
\label{eqn:SampleNum}
n > \beta_1 \Big(\log (\beta_2) + \log (1/\nu) \Big) \approx O(-d \epsilon^{-d}\log\epsilon)
\end{equation}
where $ \displaystyle \beta_1 = C~ \Big(\frac{\epsilon}{4}\Big)^{-d}\Big(1 - (\frac{\epsilon}{8\tau})^2\Big)^{-d/2} $ and $\displaystyle \beta_2 = C~\Big(\frac{\epsilon}{8}\Big)^{-d}\Big(1 - (\frac{\epsilon}{16\tau})^2\Big)^{-d/2} $, then
 based on the training data set $X$, there exists a CAE $(\Z,\E,\D)$ with $L>d$ charts $\epsilon$-faithfully representing $\M$  with probability $1 - \nu$. 
 In other words, we have 
\[  \sup_{x\in\M} \|x - \D\circ\E(x)\| \leq \epsilon. \]
Moreover, the encoder $\E$ and the decoder $\D$  has at most $ O(Lmd \epsilon^{-d - d^2/2}(-\log^{1+d/2}\epsilon))$ parameters and $\displaystyle   O(-d^2\log_2\epsilon/2)$ layers.
\label{thm:UMA}
\end{theorem}

This main result characterizes the approximation behavior of the CAE topologically and geometrically.  From theorem \ref{thm:faithfulrep}, we have that the latent space $\Z$ in the above network is homeomorphic to the data manifold and that the generating manifold $\D(\Z)$ preserves the topology of $\M$. Moreover, this theorem  provides estimations of requiring training data size and a network size to geometrically approximate the data manifold $\epsilon$-closely. 

\subsection{Sketch of Proof}
We first apply a result from Niyogi-Smale-Weinberger~\cite{niyogi2008finding} to obtain a estimation of number of training set $X$ on $\M$ satisfying that $X$ is $\delta$-dense on $\M$. Then, we use a constructive proof to show that there is a network satisfying the required accuracy and network parameters estimation. 

We use a constructive proof to show that there is a network satisfying the required accuracy and network parameters estimation. 
We begin by diving the manifold $\M$ into $L$ charts satisfying $\M = \bigcup_\ell \M_\ell$. We parameterize each chart $\M_\ell$ on a d-dimensional tangent space $\Z_\ell$ using the $\log$ map. Then,  based on a simplicial structure induced from the image of the train data set on the tangent space, we obtain a simplicial complex $\S_\ell$ whose vertices are provided by the training data on $\M_\ell$. After that, we construct a neural network to represent the piecewise linear map between the latent space and $\S_\ell$. This gives an essential ingredient to construct an encoder $\E_\ell$ and a decoder $\D_\ell$. Furthermore, we also argument that the difference between $\M_\ell$ and its simplicial approximation $\S_\ell$ is bounded above by $\epsilon$. More precisely, this local chart approximation can be summarized as: 


\begin{theorem}[Local chart approximation]Consider a geodesic neighborhood $\M_r(p) = \{x\in\M~|~ d(p,x) < r \}$ around $p\in\M$. For any $0 < \epsilon  < \tau(\M)$, if $X= \{x_i\}_{i=1}^n$ is a $\epsilon$-dense sample drawn uniformly randomly on $\M_r(p)$, then there exists an auto-encoder $(\Z,\E,\D)$ which is $\epsilon$-faithful representation of $\M_r(p)$. In other words, we have 
\begin{equation}  \sup_{x\in\M_r(p)} \|x - \D\circ\E(x)\| \leq \epsilon \end{equation}
Moreover, the encoder $\E$ and the decoder $\D$  has at most $O(mdn^{1+d/2})$ parameters and $O(\frac{d}{2} \log_2(n))$ layers.
\label{thm:localchart}
\end{theorem}
After the construction for each local chart, the global theorem can be obtained by patching together results of each local chart construction. This leads to the desired the CAE. 

\section{Proofs}
\label{sec:proofs}

\paragraph{Proof of Theorem \ref{thm:faithfulrep}}We prove that $\M$ is homeomorphic to $\Z$ by showing that $E$ is a homeomorphism. First, $\E$ is onto by the definition. Second, Assume that there are $x_1\neq x_2\in\M$ such that $\E(x_1) = \E(x_2) =z$, then $\|\D\circ\E(x_1) - x_1\| \leq \epsilon < \tau(\M)$ and $\|\D\circ\E(x_2) - x_2\| \leq \epsilon < \tau(\M)$ as the auto-encoder is $\epsilon$-faithful representation of $\M$. This contradicts with the definition of the reach $\tau(\M)$. Thus, $\E$ is a one-to-one map. Third, since $\E$ is bijective, we can have $\E^{-1}$. Note that $\E$ is a continuous map from a compact space $\M$ to a Hausdorff space $\Z$. Any closed subset $C\subset \M$ is compact, thus $\E(C)$ is compact which is also closed in the Hausdorff space $Z$. Thus, $\E$ is a closed map which maps a closed set in $\M$  to a closed set in $\Z$. By passage to complements, this implies pre-images of any open set under $\E^{-1}$ will be also open. Thus, $\E^{-1}$ is continuous.  Similarly, $\D$ is also one-to-one, otherwise, there exist $z_1\neq z_2$ satisfying $\D(z_1) = \D(z_2)$. Since $\E$ is a homeomorphism. We can find $x_1\neq x_2$ such that $\E(x_1) = z_1$ and $\E(x_2) = z_2$. From the definition of $\epsilon$-faithful representation. We have  $\|\D\circ\E(x_1) - x_1\| \leq \epsilon < \tau(\M)$ and $\|\D\circ\E(x_2) - x_2\| \leq \epsilon < \tau(\M)$ which is contradict with the definition of $\tau(\M)$. Thus, $\D$ is a homeomorphism from $\Z$ to $\D(\Z)$ based on the same argument as before. Last, if $\M$ is not contractbile, it will not homeomorphic to a simply connected domain. This concludes the proof. 
\qed

\begin{definition}[Simplicial complex]
\label{def:simplex}
A d-simplex $S$ is a d-dimensional convex hull provided by convex combinations of  $d+1$ affinely independent vectors $\{v_i\}_{i=0}^d \subset \real^m$. In other words, $\displaystyle S = \left\{ \sum_{i=0}^d \xi_i v_i ~|~ \xi_i \geq 0, \sum_{i=0}^d \xi_i = 1 \right \}$. If we write $V = (v_1 - v_0,\cdots,v_d - v_0)$, then $V$ is invertible and $S = \left\{v_0 + V\beta ~|~ \beta \in\real^m, \beta \in\Delta \right\}$ where $\Delta = \left\{\beta\in\real^d~|~ \beta \geq 0, \vec{1}^\top \beta \leq 1 \right \}$ is a template simplex in $\real^d$. The convex hull of any subset of $\{v_i\}_{i=0}^d$ is called a {\it face} of $S$. A simplicial complex $\displaystyle \mathcal{S} = \bigcup_\alpha S_\alpha$ is composed with a set of simplices $\{S_\alpha\}$ satisfying: 1) Every face of a simplex from $\S$ is also in $\S$; 2) The non-empty intersection of any two simplices $\displaystyle S _{1},S _{2}\in \S$ is a face of both $S_1$ and $S_2$. For any vertex $v \in \S$, we further write $\N^1(v) = \{a~|~ v\in S_\alpha\}$ and $\displaystyle \S^1(v) = \bigcup_{\alpha\in\N^1(v)} S_i $ the first ring neighborhood of $v$.
\end{definition}

\begin{theorem}
\label{thm:simplexfun}
Given a d-dimensional simplicial complex $\mathcal{S}= \bigcup_\alpha \S_\alpha $ with $n$ vertices $\{v_\ell\}_{\ell=1}^n$ where each  $S_\alpha$ is a d-dimensional simplex. Then, for any given piecewise linear function $f:\S \rightarrow \mathbb{R}$ satisfying $f$ linear on each simplex, there is a ReLU network representing $f$. Moreover, this neural network has $n(K(d+1) + 4(2K-1)) + n $ paremeraters and $\log2 (K) + 2$ layers, where $K = \max_i |\N(v_i)|$ which is bound above by the number of total $d$-simplices in $\S$.
\end{theorem}
\begin{proof}
We first show a hat function on $\S$ can be represented as a neural network. Given a vertex $v\in\{v_\ell\}$, let $\displaystyle \S^1(v) = \bigcup_{i\in\N^1(v)} S_i $ the first ring neighborhood of $v$. Let $\Delta = \left\{\beta\in\real^d~|~ \beta \geq 0, \vec{1}^\top \beta \leq 1 \right \}$ be a template simplex in $\real^d$ and write $S_i = \left\{ v + V_i ~\beta ~|~ \beta\in\Delta \right\}$ where $V_i\in \real^{d\times d}$ is determined by the vertices of $S_i$ and invertible. Let's write $F_i = \left\{ v + V_i ~\beta ~|~ \beta \geq 0 , |\beta| = 1 \right\} $ and $\displaystyle \bigcup_{i\in\N^1(v)} F_i$ forms the boundary of the first ring  $\S^1(v)$. We consider the following one-to-one correspondence between a point $x\in\real^m$ and its barycentric coordinates $\beta$ on the simplex $S_i$. 
\begin{equation}
    T_i:\real^d \rightarrow \real^d,  \qquad x \mapsto  \beta_i = T_i(x) =  W_i x + b_i ,  \qquad  \forall i\in\N^1(v)
\end{equation}
where $W_i = V_i^{-1}, b_i = -V_i^{-1} v$.  Meanwhile, $\beta_i$ provides a convenient way to check if $x\in S_i$, namely, $x\in S_i \Leftrightarrow \beta_i = T_i(x) \in\Delta$. We define the following function $\eta_{v}:\S\rightarrow \real $:
\begin{equation}
    \eta_{v}(x) = \max \left\{ \min_{i\in\N^1(v)} \left\{ 1 - \vec{1}^\top (W_i x + b_i) \right\}, 0 \right \}
\end{equation}
We claim that $\eta_{v}$ is a pyramid (hat) function supported on $\S(v)$, namely, $\eta_{v}$ is a piecewise linear function satisfying:
\begin{equation}
\label{eqn:hatfun1}
    \eta_{v}(x) = \left\{ \begin{array}{cc}
        1 & \text{if} \quad x = v \\
        1  - \vec{1}^\top (W_i x + b_i), & \quad\quad  \text{if} \quad x \in S_i \quad \text{for some}~ i\in\N^1(v) \\
        0 &   \text{if}  \quad x\in \S -  \bigcup_{i\in\N^1(v)} S_i
    \end{array}\right.
\end{equation}
First, it is easy to see that $T_i(v) = 0, i = 1,\cdots,K$ which yields $\eta_{v}(v) = 1$. Second, assume $x\in S_i$, then $\beta_i = T_i(x) \in \Delta$ and $1 -  
\vec{1}^\top \beta_i \geq 0$. Consider the barycentric coordinates $\beta_j = T_j(x)$ of $x$ on $S_j, j\neq i$. For those components of $\beta_j$ along the intersection edges of $S_i$ and $S_j$, they are exactly the same as the corresponding components of $\beta_i$. For those components of $\beta_j$ along the non-intersection edges, they are negative. Thus, we can $\vec{1}^\top \beta_i \geq \vec{1}^\top \beta_j, \forall j$. This implies $0\leq 1 -  
\vec{1}^\top \beta_i \leq 1 - \vec{1}^\top \beta_j, \forall j$. Therefore, $\eta_{v}(x) = 1 - \vec{1}^\top(W_i x + b_i)$. In addition, it is straightforward to check $\eta_{v}(x) = 0, \forall x\in F_i$. Third, if $x\in  \left\{ v + V_i ~\beta ~|~ \beta \geq 0, \vec{1}^\top\beta> 1  \right\}$, we have $1 -  \vec{1}^\top \beta_i <0$, thus $\eta_{v}(x) = \max\{1 -  \vec{1}^\top \beta_i, 0\} = 0$. 
It is easy to see that 
\begin{equation}
\label{eqn:min}
\displaystyle \min\{ a, b\} = \frac{1}{2} (\mathrm{ReLu}(a+b) - \mathrm{ReLu}(a-b) - \mathrm{ReLu}(-a+b)  - \mathrm{ ReLu}(-a-b))
\end{equation}
which means that $\min \{a, b\}$ can be represented as a 2-layer ReLu network. Based on equations \eqref{eqn:hatfun1} and \eqref{eqn:min}, it straightforward to show that $\eta_{v}$ can be represented as a DNN with at most $\log_2(|\N^1(v)|) + 1$ layers and at most $ (|\N^1(v)|(d +1)+ 4(2(|\N^1(v)|-1)$ parameters according to Lemma D.3 in~\cite{arora2016understanding}. Note that we can write $\displaystyle f(x) = \sum_{\ell} f(v_\ell) \phi_{v_\ell}(x)$, therefore, $f$ can be written as a DNN with at most $\log_2(K) + 2$ layers and at most $ n(K(d+1) + 4(2K-1)) + n $ parameters, where $K = \max_i |\N^1(v_i)|$ which is bound above by the number of total $d$-simplices in $\S$. This concludes the proof.
\end{proof}

We remark that our construction is different from the construction used in ~\cite{arora2016understanding}, where number of parameters is not straightforward to estimate since it relies on a hinging hyperplane theorem in~\cite{wang2005generalization}.


\paragraph{Proof of Theorem \ref{thm:localchart}}
We begin with constructing a neural network on a given chart $\M_r(p) =  \{x\in\M~|~ d(p,x) \leq \gamma\}$. Let $\T_{p,r}\M = \{v\in\T_p\M~|~ \|v\| \leq r\}$. Since $\M$ is compact, then the exponential map $\exp_p (v):\T_{p,r}\M \rightarrow \M_r(p), v\mapsto \gamma_v(1) $ is one-to-one and onto where $\gamma_v(t)$ is a geodesic curve satisfying $\gamma_v(0) = p, \dot{\gamma}_v(0) = v$. We write the inverse of $\exp_v$ as the logarithmic map $\log_p(x): \M_r \rightarrow \T_{p,r}\M $.  In the rest of the proof, we will construct an encoder $\E$ to approximate $\log_p$ and a decoder $\D$ to approximate $\exp_p$ based on the training set $X$. We define $\{z_i\}_{i=1}^n = \{\log_{p} (x_i)\}_{i=1}^n \in \Z$ as the corresponding latent variables of $X$. Note that $\{z_i\}_{i=1}^n$ are sampled on a bounded domain $\T_{p,r}\M$; thus there exists a simplicial complex for $\{z_i\}_{i=1}^n$ through a Delanuay triangulation $\S = \bigcup_{\alpha=1}^T S_\alpha$ with $T = O (n^{\lceil d/2\rceil })$~\cite{bern1995dihedral}. Here, each $S_\alpha$ is a $d$-dimensional simplex whose vertices are $d+1$ points from $\{z_i\}_{i=1}^n$. From the one-to-one correspondence between $\{z_i\}_{i=1}^n$ and $\{x_i\}_{i=1}^n$, we can have a $d$-simplex $\bar{S}_\alpha$ by replacing vertices in $S_\alpha$ as the corresponding $x_i\in X$. This provides a simplicial complex $\bar{\S} = \bigcup_{\alpha=1}^T \bar{S}_\alpha$. Note that each vertex of $\bar{\S}$ is on $\M_r$; therefore, $\bar{\S}$ provides a simplicial complex approximation of $\M_r(p)$. We define $\Z = \S = \bigcup_{\alpha=1}^T S_\alpha$ which is essentially a $d$-dimensional ball with radius $r$. It is also  straightforward to define a simplicial map
\begin{equation}
 F: \Z = \S\rightarrow \bar{\S}\subset\real^m, \qquad F(z) = \sum_{i=0}^d\xi_i x_{\alpha_i} \quad \text{for} \quad z = \sum_{i=0}^d\xi_i z_{\alpha_i}\in S_\alpha .
\end{equation} 
Here, $F$ maps $z_i$ to $x_i$ and piecewise linearly spend the rest of the map. According to Theorem~\ref{thm:simplexfun}, each component of $F$ can be represented as a neural network. Therefore, $F$ can be represented by a neural network $\D$ with at most $\displaystyle mn(n^{\lceil d/2\rceil }(d+1) + 4(2n^{\lceil d/2\rceil }-1)) + mn = O(mdn^{1+d/2})$ parameraters and $\displaystyle \lceil d/2\rceil \log_2 (n) + 2 = O(\frac{d}{2} \log_2(n))$ layers.\footnote{We remark that this estimation is not sharp since we overestimate $\max_i |\N^1(z_i)|$ using the total number of simplices. We conjecture that this number $\max_i |\N^1(z_i)|$ should constantly depend only on $d$. This is true for Delaunay triangulation to points distributed according to a Poisson process in $\real^d$~\cite{dwyer1991higher}. If this conjecture is true, then the number of parameters has order $O(mdn)$ which will improve the parameter size as $O(\epsilon^{-d})$. }

Next, we construct the decoder $\D:\M_r(p) \rightarrow \Z$. We first construct a projection from $\M_r(p)$ to its simplicial approxmation $\bar{\S}$. We write $\{x_{\alpha_0}, \cdots, x_{\alpha_d}\} \subset X$ are $d+1$ vertices in a simplex $\bar{S}_\alpha$. For convenience, we write $V_\alpha = \{x_{\alpha_0} + X_\alpha \beta~|~ \beta\in\real^d\}$ and each $\bar{S}_\alpha = \{x_{\alpha_0} + X_\alpha \beta~|~ \beta\in\Delta\}$ where $\Delta$ is the template $d$-simplex used in Definition~\ref{def:simplex}. Note that $X_\alpha = (x_{\alpha_1}-x_{\alpha_0}, \cdots, x_{\alpha_d}-x_{\alpha_0})\in\real^{m\times d}$ is full rank matrix.
We define the projection operator:
\begin{equation}
\text{Proj}_\alpha : \M_r(p) \rightarrow V_\alpha , \qquad x \mapsto \text{Proj}_\alpha(x) = X_\alpha X_\alpha^{\dagger} (x - x_{\alpha_0}) + x_{\alpha_0}
\end{equation}
where  $X_\alpha^{\dagger} = (X_\alpha^\top X_\alpha)^{-1} X_\alpha^\top$ is the Moore-Penrose pseudo-inverse of $X_\alpha$. It is clear to see that $X_\alpha^{\dagger} (x - x_{\alpha_0})$ provides coordinates of $\text{Proj}_{\alpha}(x)$ in the simplex $\bar{S}_\alpha$. Similar as the construction used in the proof of Theorem~\ref{thm:simplexfun}, for each $x_i \in X$ surrounded by $\{\bar{S}_i\}_{i\in\N^1(x_i)}$, we construct the function
\begin{equation}
    \eta_{x_i}(x) =  \chi(\|x - x_i\|^2) \max \left\{ \min_{ \alpha \in\N^1(x_i)} \left\{ 1 - \vec{1}^\top X_\alpha^{\dagger} (x - x_{\alpha_0})  \right\}, 0 \right \}
\end{equation}
where 
$\chi$ is an indicator function with $\mu =  \tau/10$.
\begin{equation} 
\displaystyle \chi(t) = \left\{\begin{array}{cc} 1, & \text{if} \quad  0\leq t\leq \delta^2 + \mu \\
1 + \frac{1}{\mu} (\delta^2 + \mu - t), & \text{if} \quad  \delta^2 + \mu  \leq t \leq \delta^2 + 2\mu \\
 0, & \text{if} \quad  t \geq \delta^2 +  2\mu 
\end{array} \right. 
\end{equation}
Since the indicator function $\chi$ restrict $x$ in the first ring of $x_i$, using similar argument as before, one can also show that 
\begin{equation}
\label{eqn:hatfun2}
    \eta_{x_i}(x) = \left\{ \begin{array}{cc}
        1 & \text{if} \quad x = x_i \\
        1 - \vec{1}^\top X_\alpha^{\dagger} (x - x_{\alpha_0}) , & \quad \text{if} \quad x\in \bar{S}_\alpha \quad \text{for some}~ \alpha\in\N^1(x_i) \\
        0 &   \text{otherwise}
    \end{array}\right.
\end{equation}
It is straightforward to check that $\chi(t) =  \frac{1}{\mu}  \text{ReLu}(-t + \delta^2 + 2\mu) - \frac{1}{\mu}  \text{ReLu}(-t + \delta^2 + \mu)$. Therefore, $\eta_{x_i}$ can be represented as the neural network with at most $\log_2(|\N^1(v)|) + 1$ layers and at most $ (|\N^1(v)|(d +1)+ 4(2(|\N^1(v)|-1)$ parameters. We remark that this neural network is not a feed forward ReLu network as we request the network compute multiplication between features for $\|x - x_i\|^2$ and multiplication of $\chi$. We define the encoder $\E(x) = F^{-1}\circ\text{Proj}_\alpha(x)$. Here $\bar{S}_\alpha$ is chosen as the closest simplex to $x$ and $F^{-1}(x) = \sum_{i=0}^d\xi_i z_{\alpha_i} \quad \text{for} \quad x = \sum_{i=0}^d\xi_i x_{\alpha_i}\in \bar{S}_\alpha$. Similar as approximation of $F$, we can use a neural network to represent $\E$ with at most $ O(mdn^{1+d/2})$ parameters and $\displaystyle \lceil d/2\rceil \log_2 (n) + 2 = O(\frac{d}{2} \log_2(n))$ layers.

Now, we estimate the difference between $x$ and $\D\circ\E(x)$. From the above construction of $\E$ and $\D$, we have $\D\circ\E(x) = \text{Proj}_\alpha(x)$ where $\bar{S}_\alpha$ is chosen as the closest simplex to $x$.    Since $X$, the vertices of $\bar{\S}$, is a $\epsilon/2$-dense sample, this implies $\text{diam}(\bar{S}_\alpha) = \max \Big\{ \|x -y\|~|~ x,y\in \bar{S}_\alpha \Big\} \leq \epsilon$. Since the reach of $\M$ is $\tau$. The worst scenario becomes to compute the approximation error between a segment connecting two $\epsilon$-away points on a radius $\tau$ circle to itself. This error is $\tau - \sqrt{\tau^2 - (\epsilon/2)^2} \leq \epsilon$ since $\epsilon < \tau/2$
This concludes the proof.
\qed

\paragraph{Proof of Theorem \ref{thm:UMA}}
We first apply proposition 3.2 from Niyogi-Smale-Weinberger~\cite{niyogi2008finding} to obtain an estimation of number of training set $X$ on $\M$ satisfying that $X$ is $\epsilon/2$-dense on $\M$ with probability $1-\nu$. We reiterate this proposition here:
\begin{prop}{\textbf{Niyogi-Smale-Weinberger}}\cite{niyogi2008finding}
Let $M$ be a d-dimensional compact manifold with the reach $\tau$. Let $ X = \{x_i\}_{i=1}^n$ be a set of $n$ points drawn in i.i.d. fashion according to the uniform probability measure on $\M$. Then with probability greater than $1-\nu$, we have that $X$ is $\epsilon/2$-dense ($\epsilon < \tau/2$) in $\M$ provided:
\begin{equation}\label{eqn:npts}
    n > \beta_1 \Big(\log (\beta_2) + \log (\dfrac{1}{\nu}) \Big)
\end{equation}
where $ \beta_1 = \dfrac{vol(\M)}{\cos^d\Big(\arcsin(\dfrac{\epsilon }{8 \tau} )vol(B^d_{\epsilon/4})\Big)}$ and $\beta_2 = \dfrac{vol(\M)}{\cos^d\Big(\arcsin (\dfrac{\epsilon }{16 \tau} )vol(B^d_{\epsilon/8})\Big)}$. Here $vol(B^d_\delta)$ denotes the volume of the standard $d$-dimensional ball of radius $\delta$. 
\end{prop}
Note that $\displaystyle vol(B^d_\delta) = \frac{\pi^{d/2}\delta^d}{\Gamma(1+ d/2)}$ and $\cos(\arcsin(\delta)) = \sqrt{1 - \delta^2}$. Plugging them in the above proposition yields $ \displaystyle \beta_1 = C~ \Big(\frac{\epsilon}{4}\Big)^{-d}\Big(1 - (\frac{\epsilon}{8\tau})^2\Big)^{-d/2} $ and $\displaystyle \beta_2 = C~\Big(\frac{\epsilon}{8}\Big)^{-d}\Big(1 - (\frac{\epsilon}{16\tau})^2\Big)^{-d/2}  $. It is clear that $n = O\Big(-d\epsilon^{-d}\log\epsilon\Big)$. 

Since $\M$ is a compact manifold, we cover $\M$ using $L$ geodesic ball as we used in theorem \ref{thm:localchart}, i.e. $\M = \bigcup_{\ell=1}^L \M_r(p_\ell)$. One can control the radius such that $L \geq d$. We write restriction of training data set $X$ on each of $\M_\ell$ as $X_\ell = \M_r(p_\ell)\cap X$. Since $X$ is uniformly sampled on $\M$, thus $|X_\ell | = O(n/L)$. Based on the local theorem \ref{thm:localchart}, each of $\M_r(p_\ell)$ has an $\epsilon$-faithful representation $(\Z_\ell, \E_\ell, \D_\ell)$ where each $\Z_\ell$ is a radius $r$ standard ball in $\real^d$, both $\E_\ell$ and $\D_\ell$ have  $O(md (n/L)^{1+d/2}) = O(md \epsilon^{-d - d^2/2}(-\log^{1+d/2}\epsilon))$ parameters and $\displaystyle  O(\frac{d}{2}\log_2(n/L)) = O(-d^2\log_2\epsilon/2)$ layers.

To construct a latent space for $\M$, we consider a disjoint union $\bar{\Z} = \bigsqcup_\ell \Z_\ell$ and glue $\Z_\ell$ through an equivalence relation. Given $z_{\ell_1}\in\Z_{\ell_1}$ and $ z_{\ell_2} \in \Z_{\ell_2}$, we define $z_{\ell_1}\sim z_{\ell_2}$ if $\E_{\ell_1}^{-1} (z_{\ell_1}) = \E_{\ell_2}^{-1} (z_{\ell_2})$, then the latent space is defined as $\Z = \bar{\Z}/\sim$. This construction guarantees that $\Z$ is homeomorhphic to $\M$ which is compatible with the result from Theorem~\ref{thm:faithfulrep}. According the construction of $\D_\ell$, it is clear to see that if $z_{\ell_1}\sim z_{\ell_2}$, then $\D_{\ell_1}(z_{\ell_1}) = \D_{\ell_2}(z_{\ell_2})$. Therefore, the collection of encoders and decoders are well-defined. Since each of $(\Z_\ell, \E_\ell, \D_\ell)$ is $\epsilon$-faithful representation, therefore $(\Z, \{\E_\ell\}, \{\D_\ell\})$ is also $\epsilon$-faithful representation. Overall,  encoders and decoders have $O(Lmd \epsilon^{-d - d^2/2}(-\log^{1+d/2}\epsilon))$ parameters and $\displaystyle   O(-d^2\log_2\epsilon/2)$ layers. This concludes the proof.

\qed


\section{Wrap Up of Simplectic Relu-Nets}

The work in this chapter is mostly conceptual.  We theoretically prove that multi-chart is necessary for preserving the data manifold topology and approximating it $\eps$-closely. These architectures are explicitly contrived and do not look like something we would most likely use in practice, but since they are much more sparsely connected than standard fully connected and/or convolutional architectures, they serve as a bound for the power of these more common models. This proves a universal approximation theorem on the representation capability of CAE and provides estimations of training data size and network size. Moreover, these results give us a theoretical understanding of the advantages of multi-chart encoding.

 
\chapter{CONCLUSION}\label{Chap:Conclusion}

Finally, we conclude this work with some ideas for future work and final thoughts. 

\section{Future Work}

Throughout this work, our inspiration has come from combining geometric intuition with multi-scale, or at least multi-modal approaches to signal processing. By fusing these fields, each is enhanced: geometry tells us how to construct a good basis (or other representation of a signal), and a good basis tells us how to work with geometry. Once we have established efficient representations of spaces and signals, their analysis becomes much easier. We have not focused on this analysis aspect in this thesis, but much of our ongoing and future work will.  

The variational model presented in Chapter \ref{Chap:LBBP} can easily be extended to graph cases, by replacing the mass and stiffness matrices with the graph Laplacian and integration matrices. The conformal deformation introduced is similar to the concept of graph attention \cite{vestner2017efficient}. However, state-of-the-art approaches rely on solving correspondence models with neural networks rather than directly solving optimization problems, which can be time-consuming. 

Both our work on Parallel Transportation Convolution (Chapter \ref{Chap:PTC}) and Charted Auto-Encoders (Chapter \ref{Chap:CAE}) can be used to improve generative models for 3D shapes. Using them to construct molecules and other geometric graphs is a very interesting path for possible future work. The ability of the chart auto-encoder to preserve topological invariants such as closed paths makes it a very exciting tool for representing data with cycles in it, such as human gait data. 

\section{Final Thoughts}

In this thesis, we have studied many aspects of computational geometry and geometric learning. In Chapter \ref{Chap:LBBP} we presented a variational model for non-isometric shape matching. Inspired by state-of-the-art machine learning techniques we developed parallel transport convolution in Chapter \ref{Chap:PTC} and used it to preform a variety of signal processing tasks, including shape matching. In Chapter \ref{Chap:CAE} we developed a novel model for representing manifold disturbed data in a way that preserves geometrically meaningful information, establishing a corresponds between latent and embedded representations. Finally, in Chapter \ref{Chap:Approx} we showed the existence of local representations which are provably better than ones that rely on global parametereizations. 

Some of this work has already begun to affect the field. Analysis that uses a basis pursuit and functional maps has recently been proposed in \cite{azencot2019shape}. Parallel Transportation Convolution has been extended to the point cloud case in \cite{jin2019nptc} and used alongside a specially designed set of filter banks to create equivalent neural networks on surfaces in \cite{cohen2019gauge}. Recently submitted work by another member of our research group uses PTC, along with several other innovations, to separate intrinsic and extrinsic information from 3d shapes in an unsupervised manner \cite{tatro2020unsupervised}. We hope that this work continues to inspire and influence other advances in the fields of shape analysis, geometric deep learning, and computational mathematics as a whole.

 
%
\specialhead{REFERENCES}

\bibliographystyle{siamplain.bst}
\bibliography{refrences4}

 
\appendix    
\addcontentsline{toc}{chapter}{APPENDICES...}             

\chapter{EFFICIENT COMPUTATION OF PTC LAYER}
Since the limitation of spare matrix product implementation in TensorFlow and PyTorch, we use the following method to implement the proposed convolution. More specifically, we consider a mesh with $n$ points, a signal with $q$ channels $F = (F_1,\cdots,F_q)\in\RR^{n\times q}$ and $p$ filters  each of which has $q$ input channels denoted $\textbf{K} = \{K_{11},\cdots,K_{1p},\cdots,K_{q1},\cdots,K_{qp}\}$. We would like to compute convolution $F\star \textbf{K} =  \sum_{i=1}^{q} F_i \star K_{ij} \in\RR^{n\times p}$. Given a mesh with the mass matrix $M$, we write $I_i$ as the index set of the neighborhood of the $i$ point and denote $W_i\in \RR^{|I_i| \times k}$ the parallel transportation operation to the $i$-th point. 
The following method provides a fast, memory efficient implementation of PTC convolution in TensorFlow and PyTorch. 

We write $Z_i = F_i ^T M \in \RR^{n\times 1} , i = 1,\cdots,q $  and let $L = \sum_{i}|I_i|$. We define $\textbf{Z}_i$ as a $L\times L$ sparse matrix whose support at the $k$-th row is provided by $I_k$ with value $Z_i(I_k)$, formally we write:
\begin{equation}
textbf{Z}_i = \begin{pmatrix}
Z_i(I_1) \\
Z_i(I_2) \\ 
\vdots   \\  
Z_i(I_n) \\
\end{pmatrix},~
\textbf{Z}= 
\begin{pmatrix}
\textbf{Z}_1   & &   &  &   \\ 
& \textbf{Z}_2 & & \bigzero &  \\ 
& &\ddots  & &  \\ 
& \bigzero & & \ddots    &    \\ 
&  &  & &   \textbf{Z}_q
\end{pmatrix}
\end{equation}
In addition, we define:
\begin{equation}
\textbf{W} = 
\begin{pmatrix}
W_1 \\ 
W_2 \\ 
\vdots \\ 
\vdots \\
W_n 
\end{pmatrix}, 
\bar{\textbf{W}} = 
\begin{pmatrix}
\textbf{W}K_{11}   & \cdots &  \textbf{W}K_{1p} \\ 
\textbf{W}K_{21}      & \cdots &  \textbf{W}K_{2p} \\ 
\vdots          &                 \ddots    &   \vdots \\ 
 \textbf{W}K_{q1} & \cdots &  \textbf{W}K_{qp}
\end{pmatrix}
\end{equation}
where $\bar{\textbf{W}} = reshape(\textbf{WK},[Lq,p])$. Finally, the PTC can be computed as 
\begin{equation}
(F \star \textbf{K})= \left(\sum_{axis = 3} reshape(\textbf{Z}\bar{\textbf{W}},[p,n,q])\right)^T
\end{equation}

Using the above sparse matrix operations, the computation complexity of the proposed PTC is the same scale as the standard convolution in Euclidean domains.


\chapter{MODEL DETAILS FOR CAE EXPERIMENTS}\label{app:Networks}

This section provides the details of the neural network architectures used in the numerical experiments. We denote by $FC_m$ a fully connected layer with $m$ output neurons; by $Conv_{i,j,k.l}$ a convolution layer with filters of size $(i,j)$, input dimension $k$, and output dimension $l$; by $d$ the dimension of the latent space, $n$ the dimension of the ambient space and $N$ the number of charts. See~\eqref{eqn:model.ae},\eqref{eqn:model.small.vae}, \eqref{eqn:model.medium.vae}, \eqref{eqn:model.large.vae}, \eqref{eqn:model.cae1}, and~\eqref{eqn:model.cae2} for the architectures.

\begin{figure*}[ht]
\textbf{Auto-Encoders}:
\begin{equation}\label{eqn:model.ae}
  \begin{split}
    \text{Encoder}&: x \rightarrow FC_{250} \rightarrow  FC_{250} \rightarrow  FC_{250} \rightarrow FC_{d} \rightarrow  z \\
    \text{Decoder}&: z  \rightarrow FC_{250} \rightarrow  FC_{250} \rightarrow  FC_{250} \rightarrow FC_{n} \rightarrow  y
  \end{split}
\end{equation}

\textbf{Small Variational Auto-Encoders}:
\begin{equation}\label{eqn:model.small.vae}
  \begin{split}
    \text{Encoder}&: x \rightarrow FC_{50} \rightarrow  FC_{50} \rightarrow FC_{2d} \rightarrow  \mu, \sigma \\
    \text{Decoder}&: z \in \mathcal{N}(\mu,\sigma) \rightarrow FC_{50} \rightarrow  FC_{50} \rightarrow FC_{n} \rightarrow  y
  \end{split}
\end{equation}

\textbf{Medium Variational Auto-Encoders}:
\begin{equation}\label{eqn:model.medium.vae}
  \begin{split}
    \text{Encoder}&: x \rightarrow FC_{100} \rightarrow  FC_{100} \rightarrow  FC_{100} \rightarrow FC_{2d} \rightarrow  \mu, \sigma \\
    \text{Decoder}&: z \in \mathcal{N}(\mu,\sigma) \rightarrow FC_{100} \rightarrow  FC_{100} \rightarrow  FC_{100} \rightarrow FC_{n} \rightarrow  y
  \end{split}
\end{equation}

\textbf{Large Variational Auto-Encoders}:
\begin{equation}\label{eqn:model.large.vae}
  \begin{split}
    \text{Encoder}&: x \rightarrow FC_{250} \rightarrow  FC_{250} \rightarrow  FC_{250} \rightarrow FC_{2d} \rightarrow  \mu, \sigma \\
    \text{Decoder}&: z \in \mathcal{N}(\mu,\sigma) \rightarrow FC_{250} \rightarrow  FC_{250} \rightarrow  FC_{250} \rightarrow FC_{n} \rightarrow  y
  \end{split}
\end{equation}

\textbf{CAE}
\begin{equation}\label{eqn:model.cae1}
  \begin{split}
    \text{Initial Encoder}&: x \rightarrow FC_{150} \rightarrow  FC_{150} \rightarrow  FC_{150} \rightarrow z \\
    \text {Chart Encoder} &:z \rightarrow FC_{150} \rightarrow FC_{d} \rightarrow z_{\alpha} \\
    \text{Chart Decoder}&:z_{\alpha} \rightarrow FC_{150} \rightarrow  FC_{150} \rightarrow  FC_{150} \rightarrow FC_{n} \rightarrow  y_{\alpha} \\
    \text{Chart Prediction}&: x \rightarrow FC_{150} \rightarrow  FC_{N} \rightarrow softmax \rightarrow p
  \end{split}
\end{equation}

\textbf{Conv CAE}
\begin{equation}\label{eqn:model.cae2}
  \begin{split}
    \text{Initial Encoder}&: x \rightarrow FC_{150} \rightarrow  FC_{150} \rightarrow  FC_{625} \rightarrow z  \\
    \text {Chart Encoder} &:z\rightarrow Conv_{3,3,1,8} \rightarrow  Conv_{3,3,8,8} \rightarrow  Conv_{3,3,8,16} \rightarrow z \\
    \text{Chart Decoder}&:z_{\alpha}\rightarrow Conv_{3,3,16,8} \rightarrow  Conv_{3,3,8,8} \rightarrow  Conv_{3,3,8,1} \rightarrow FC_{n} \rightarrow  y_{\alpha} \\
    \text{Chart Prediction}&: z \rightarrow FC_{250} \rightarrow  FC_{10} \rightarrow softmax \rightarrow p
  \end{split}
\end{equation}
\end{figure*} 
 
\end{document}